\documentclass[twoside,11pt]{article}

\usepackage[abbrvbib,preprint]{jmlr2e}

\usepackage{xspace}
\usepackage{xargs}
\usepackage{ifthen}
\usepackage{etoolbox}
\usepackage{enumitem}
\usepackage{csquotes}
\usepackage[table]{xcolor}
\usepackage{amsfonts}
\usepackage{mathrsfs}
\usepackage{dsfont}
\usepackage{amssymb}
\usepackage{thmtools}
\usepackage[fixamsmath,disallowspaces]{mathtools}
\mathtoolsset{showonlyrefs}
\mathtoolsset{centercolon=true}
\usepackage{bm} 
\makeatletter
\g@addto@macro\bfseries{\boldmath}
\makeatother
\allowdisplaybreaks[4]

\hypersetup{hidelinks} 

\newenvironment{highlight}{\begin{quote}\itshape}{\end{quote}}
\newenvironment{rmklist}{\begin{enumerate}[label={(\arabic*)},itemindent=2em,leftmargin=0em]}{\end{enumerate}}
\newenvironment{thmproperties}{\begin{enumerate}[label={(\roman*)}]}{\end{enumerate}}
\newenvironment{deflist}{\begin{enumerate}[label={(\alph*)}]}{\end{enumerate}}
\newenvironment{romanlist}{\begin{enumerate}[label={(\roman*)}]}{\end{enumerate}}
\newtheorem{problem}{Problem}
\newcommand{\qedsymbol}{$\blacksquare$}
\providecommand{\qed}{\hfill\qedsymbol}

\newcommand{\ifargdef}[3][{}]{\ifthenelse{\equal{#2}{}}{#1}{#3}}

\newcommand{\opleft}[1]{\mathopen{}\left#1}
\newcommand{\opright}[1]{\right#1\mathclose{}}
\newcommandx{\braces}[4]{%
\ifstrequal{#3}{normal}{#1#4#2}{%
\ifstrequal{#3}{auto}{\left#1#4\right#2}{%
\ifstrequal{#3}{opauto}{\opleft#1#4\opright#2}{%
#3#1#4#3#2}}}%
}

\newcommand{\N}{\mathbb{N}} 
\newcommand{\R}{\mathbb{R}} 

\renewcommand{\iff}{\Leftrightarrow} 
\newcommand{\suchthat}{\mid} 

\newcommand{\cardinality}[1]{\abs{#1}} 
\newcommand{\union}{\cup} 
\newcommand{\bigunion}{\bigcup} 
\newcommand{\intersec}{\cap} 
\newcommand{\bigintersec}{\bigcap} 
\newcommand{\distsq}[2]{\operatorname{dist}^2(#1, #2)}
\newcommand{\diam}{\operatorname{diam}}
\newcommand{\radius}{\operatorname{rad}}

\newcommandx{\intvcl}[3][1=normal]{\braces{[}{]}{#1}{#2, #3}} 
\newcommandx{\intvop}[3][1=normal]{\braces{(}{)}{#1}{#2, #3}} 
\newcommandx{\intvclop}[3][1=normal]{\braces{[}{)}{#1}{#2, #3}} 
\newcommandx{\intvopcl}[3][1=normal]{\braces{(}{]}{#1}{#2, #3}} 

\newcommandx{\abs}[2][1=normal]{\braces{\lvert}{\rvert}{#1}{#2}} 
\newcommandx{\ceil}[2][1=normal]{\braces{\lceil}{\rceil}{#1}{#2}} 
\newcommandx{\floor}[2][1=normal]{\braces{\lfloor}{\rfloor}{#1}{#2}} 
\newcommandx{\round}[2][1=normal]{\braces{[}{]}{#1}{#2}} 
\newcommandx{\der}[1]{D^{#1}} 
\newcommandx{\gradient}{\nabla} 
\newcommandx{\partder}[4][1={},4={}]{\frac{\partial^{#4} #2}{\partial #3^{#4}}\ifargdef{#1}{\Big|_{#1}}} 
\newcommandx{\integ}[4][1={},2={}]{\int_{#1}^{#2} #3 \, #4} 
\newcommandx{\asympffaster}[2][1=normal]{o\braces{(}{)}{#1}{#2}} 
\newcommandx{\asympfaster}[2][1=normal]{O\braces{(}{)}{#1}{#2}} 
\newcommandx{\asympeq}[2][1=normal]{\Theta\braces{(}{)}{#1}{#2}} 
\newcommandx{\asympsslower}[2][1=normal]{\omega\braces{(}{)}{#1}{#2}} 
\newcommandx{\asympslower}[2][1=normal]{\Omega\braces{(}{)}{#1}{#2}} 

\newcommandx{\norm}[2][1=normal]{\braces{\|}{\|}{#1}{#2}} 
\renewcommandx{\sp}[3][1=normal]{\braces{\langle}{\rangle}{#1}{#2, #3}} 
\newcommandx{\End}[2][2={}]{\mathcal{L}\opleft( #1 \ifargdef{#2}{, #2} \opright)} 
\newcommand{\T}{\mathsf{T}} 
\renewcommand{\vec}[1]{\boldsymbol{#1}} 
\newcommandx{\opnorm}[2][1=normal]{\norm[#1]{#2}_{\operatorname{op}}} 
\newcommandx{\ball}[2][1={},2={}]{\mathbb{B}_{#1}^{#2}} 
\renewcommand{\S}{\mathbb{S}} 
\newcommand{\I}[1]{\vec{I}_{#1}}
\newcommand{\vnull}{\vec{0}}

\newcommandx{\measure}[2][1=normal]{\operatorname{vol}\braces{(}{)}{#1}{#2}} 
\newcommandx{\Leb}[3][1={},3=normal]{L^{#2}\ifargdef{#1}{\braces{(}{)}{#3}{#1}}{}} 
\newcommandx{\Lebnorm}[4][1=normal,3={2},4={}]{\norm[#1]{#2}_{#3}} 
\renewcommandx{\l}[3][1={},3=normal]{\ell^{#2}\ifargdef{#1}{\braces{(}{)}{#3}{#1}}} 
\newcommandx{\lnorm}[4][1=normal,3={2},4={}]{\norm[#1]{#2}_{#3}} 
\newcommandx{\Smooth}[4][1={},3={},4=normal]{C_{#3}^{#2}\ifargdef{#1}{\braces{(}{)}{#4}{#1}}} 
\newcommandx{\Schwartz}[2][1={},2=normal]{\mathscr{S}\ifargdef{#1}{\braces{(}{)}{#2}{#1}}} 
\newcommandx{\Schwartzpoly}[2][1=normal]{\braces{\langle}{\rangle}{#1}{\abs[#1]{#2}} } 
\newcommandx{\Tempdistr}[2][1={},2=normal]{\mathscr{S}'\ifargdef{#1}{\braces{(}{)}{#2}{#1}}} 
\newcommandx{\distrinp}[3][1=normal]{\braces{\langle}{\rangle}{#1}{#2, #3}} 
\newcommandx{\ft}[3][1=default,2=auto]{
\ifstrequal{#1}{default}{\widehat{#3}}{
\ifstrequal{#1}{long}{{\braces{(}{)}{#2}{#3}}^{\wedge}}{}}} 
\newcommandx{\ift}[3][1=default,2=auto]{
\ifstrequal{#1}{default}{\check{#3}}{
\ifstrequal{#1}{long}{{\braces{(}{)}{#2}{#3}}^{\vee}}{}}} 

\newcommand{\convhull}[1]{\operatorname{conv}(#1)} 
\newcommand{\cone}[1]{\operatorname{cone}(#1)} 
\newcommand{\meanwidth}[2][{}]{w_{#1}(#2)} 
\newcommand{\effdim}[2][{}]{w_{#1}^2(#2)} 
\newcommand{\covnumber}[2]{\mathcal{N}(#1, #2)} 

\newcommandx{\prob}[2][1={},2=normal]{\mathbb{P}\ifargdef{#1}{\braces{(}{)}{#2}{#1}}}
\newcommandx{\mean}[2][1={},2=normal]{\mathbb{E}\ifargdef{#1}{\braces{[}{]}{#2}{#1}}}
\newcommandx{\var}[2][1={},2=normal]{\mathbb{V}\ifargdef{#1}{\braces{[}{]}{#2}{#1}}}

\newcommand{\distributed}{\sim}
\newcommand{\probind}[1]{\mathds{1}_{#1}} 
\newcommandx{\Unif}[2][1=normal]{\mathsf{U}\braces{(}{)}{#1}{#2}} 
\newcommandx{\Normdistr}[3][1=normal]{\mathsf{N}\braces{(}{)}{#1}{#2, #3}} 
\newcommandx{\normsubg}[2][1=normal]{\norm[#1]{#2}_{\psi_2}} 

\newcommand{\layersec}[1]{\hat{#1}} 
\newcommand{\relu}{\operatorname{ReLU}}

\newcommand{\NNlayer}[1][]{\Phi^{#1}} 
\newcommand{\NNlayersec}[1][]{\layersec{\Phi}^{#1}}
\newcommand{\NN}{F} 

\newcommand{\w}{\vec{w}} 
\newcommand{\W}{\vec{W}} 
\newcommand{\bias}{b} 
\newcommand{\Bias}{\vec{b}} 
\newcommand{\biasp}{\lambda} 

\newcommand{\x}{\vec{x}} 
\newcommand{\set}{\mathcal{X}} 
\newcommand{\setcen}{\mathcal{C}} 
\newcommand{\Rad}{R} 
\newcommand{\rad}{r} 
\newcommand{\cent}{\vec{c}} 
\newcommand{\mwloc}{w}

\newcommand{\mdist}{\delta} 
\newcommand{\marg}{\mu} 
\newcommand{\narrow}{\varepsilon} 
\newcommand{\narrowmdist}{\gamma}
\newcommand{\hypp}[2][]{H[#2\ifargdef{#1}{, #1}]} 
\newcommand{\hyppsep}{t} 
\newcommand{\sepdir}{\vec{u}} 
\newcommand{\sepdircoor}{u} 
\newcommand{\sepbias}{\tau} 

\newcommand{\vgen}{\vec{z}} 
\newcommand{\setind}{I} 

\newcommand{\probsuccess}{\eta} 
\newcommand{\probsuccessexp}{u} 
\newcommand{\event}{\mathsf{A}} 
\newcommand{\eventalt}{\mathsf{B}} 
\newcommand{\g}{\vec{g}} 
\providecommand{\G}{} 
\renewcommand{\G}{\vec{G}}
\newcommand{\normgauss}{\vec{v}}

\DeclareMathOperator{\distance}{\operatorname{dist}}
\newcommand{\dimredprecision}{\kappa}
\newcommand{\dimredmatrix}{\vec{A}}

\jmlrheading{}{}{}{}{}{}{Sjoerd Dirksen, Martin Genzel, Laurent Jacques, and Alexander Stollenwerk}

\firstpageno{1}

\usepackage{lastpage}
\jmlrheading{23}{2022}{1-\pageref{LastPage}}{9/21; Revised
	10/22}{11/22}{21-1079}{Sjoerd Dirksen, Martin Genzel, Laurent Jacques, and Alexander Stollenwerk}
\ShortHeadings{The Separation Capacity of Random Neural Networks}{Dirksen, Genzel, Jacques, and Stollenwerk}

\begin{document}
	
	\title{The Separation Capacity of Random Neural Networks}
	
	\author{\name Sjoerd Dirksen \email s.dirksen@uu.nl \\
		\addr Mathematical Institute\\
		Utrecht University\\
		3584 CD Utrecht, Netherlands
		\AND
		\name Martin Genzel \email martingenzel@gmail.com \\
		\addr Mathematical Institute\\
		Utrecht University\\
		3584 CD Utrecht, Netherlands
		\AND
		\name Laurent Jacques \email laurent.jacques@uclouvain.be \\
		\addr ISPGroup, INMA, ICTEAM Institute\\
		Université Catholique de Louvain\\
		1348 Louvain-la-Neuve, Belgium
		\AND
		\name Alexander Stollenwerk \email alexander.stollenwerk@uclouvain.be \\
		\addr ISPGroup, INMA, ICTEAM Institute\\
		Université Catholique de Louvain\\
		1348 Louvain-la-Neuve, Belgium}
	
	\editor{Joan Bruna}
	
	\maketitle
	
	\begin{abstract}
		Neural networks with random weights appear in a variety of machine learning applications, most prominently as the initialization of many deep learning algorithms and as a computationally cheap alternative to fully learned neural networks. In the present article, we enhance the theoretical understanding of random neural networks by addressing the following data separation problem: under what conditions can a random neural network make two classes $\set^-, \set^+ \subset \R^d$ (with positive distance) linearly separable?
		We show that a sufficiently large two-layer ReLU-network with standard Gaussian weights and uniformly distributed biases can solve this problem with high probability.
		Crucially, the number of required neurons is explicitly linked to geometric properties of the underlying sets $\set^-, \set^+$ and their mutual arrangement.
		This instance-specific viewpoint allows us to overcome the usual curse of dimensionality (exponential width of the layers) in non-pathological situations where the data carries low-complexity structure. 
		We quantify the relevant structure of the data in terms of a novel notion of mutual complexity (based on a localized version of Gaussian mean width), which leads to sound and informative separation guarantees.
		We connect our result with related lines of work on approximation, memorization, and generalization.
	\end{abstract}
	
	\begin{keywords}
		Random neural networks, classification, hyperplane separation, high-dimensional geometry, Gaussian mean width
	\end{keywords}

\section{Introduction}
\label{sec:intro}

\begin{figure}
	\centering
	\includegraphics[width=0.6\linewidth]{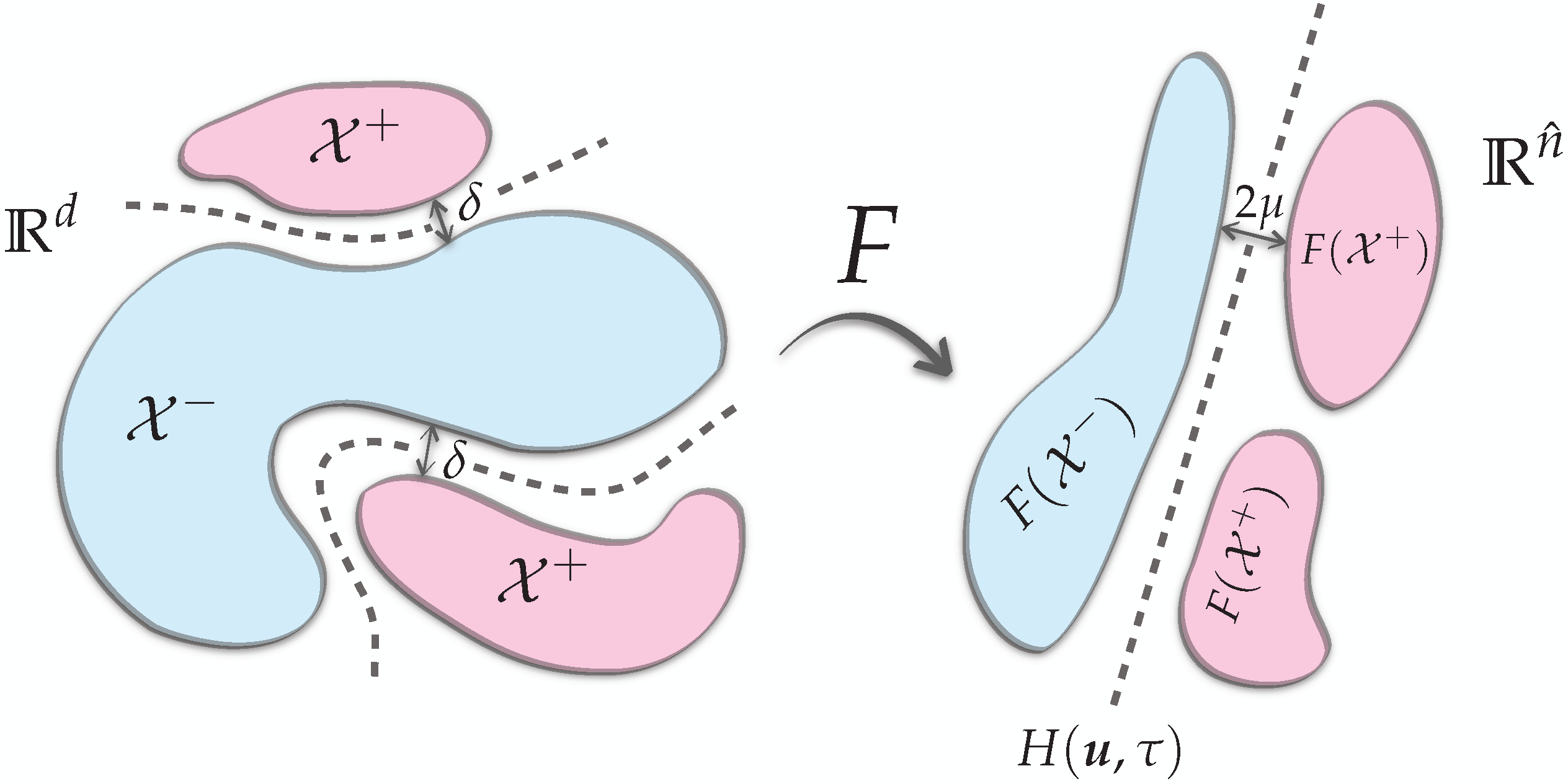}
	\caption{\textbf{Illustration of Problem~\ref{prob:rnn-sep}.} Can a random NN $\NN \colon \R^d\to \R^{\layersec{n}}$ \enquote{disentangle} the two sets $\set^-, \set^+\subset \R^d$ such that they become linearly separable in the feature space $\R^{\layersec{n}}$ with a positive margin~$\marg$? Except for being $\mdist$-separated and bounded, $\set^-$ and $\set^+$ may have an arbitrary decision boundary, possibly with multiple connected components. }
	\label{fig:intro:problem}
\end{figure}

Despite the unprecedented success of neural networks (NNs) in countless applications \citep{lbh15,sch15,gbc16}, a rigorous understanding of their operating principles is still in its infancy. The present work is devoted to a mathematical study of \emph{random NNs}, i.e., feedforward NNs whose weight parameters are drawn from a generic probability distribution. Random NNs play an important role in machine learning in at least three different ways. First, it is standard to initialize the training of a (deep) NN by random weights and it is well known that this initialization is a key contributor to the exceptional performance of NNs \citep{hzrs15,gbc16,ab19}. Second, it has been empirically observed that architecture search can be effectively carried out with random NNs, in the sense that the hierarchy in performance of fully trained architectures closely matches the hierarchy of the architectures with random weights \citep{sax+11}. Finally, random NNs have been extensively investigated as a cheap computational alternative to fully trained NNs: it has been demonstrated empirically that pre-processing with a random NN and applying a simple classification method already gives surprisingly good results \citep{hzs06,rr08,zbhrv16}. For these reasons, it is of substantial interest to gain a deeper theoretical understanding of the properties of random NNs.

In this work, we shed new light on the capabilities of random NNs as a pre-processor by addressing the following fundamental problem on class separability:
\begin{problem}\label{prob:rnn-sep}
	Consider two bounded, possibly infinite sets $\set^-, \set^+\subset \R^d$ that are \emph{$\mdist$-separated}, i.e.,
	\begin{equation}\label{eq:rnn-sep:delta-sep}
		\lnorm{\x^+ - \x^-} \geq \mdist \qquad \text{for all $\x^+\in \set^+$ and $\x^-\in \set^-$.}
	\end{equation}
	Let $\NN \colon \R^d\to \R^{\layersec{n}}$ represent a (multi-layer) feedforward NN with \emph{random weights}, where the architecture of~$\NN$ may depend on $\set^-$ and $\set^+$.
	\begin{highlight}
		Under what conditions does $\NN$ make the classes $\set^-$ and $\set^+$ \emph{linearly} separable with high probability?
		Is there a lower bound for the induced separation margin?
	\end{highlight}
\end{problem}
Formally, we will identify conditions that ensure the existence (with high probability) of a hyperplane $\hypp[\sepbias]{\sepdir} \coloneqq \{ \vgen \in \R^{\layersec{n}} \suchthat \sp{\sepdir}{\vgen} + \sepbias = 0 \}$ with $\lnorm{\sepdir} = 1$ and $\sepbias \in \R$ that separates $\NN(\set^-)$ and $\NN(\set^+)$ with a certain \emph{margin} $\marg > 0$, i.e.,
\begin{align}
	\sp{\sepdir}{\NN(\x^-)}+\sepbias &\leq - \marg \qquad \text{for all $\x^-\in \set^-$}, \\
	\sp{\sepdir}{\NN(\x^+)}+\sepbias & \geq +\marg \qquad \text{for all $\x^+\in \set^+$.}\label{eq:rnn-sep:hyper-sep}
\end{align}
Problem~\ref{prob:rnn-sep} thus states a purely geometric question on the \emph{separation capacity} of random NNs, see Figure~\ref{fig:intro:problem} for an illustration.
However, it is useful to bear in mind that the ability to render two \enquote{intertwined} sets linearly separable also has immediate consequences for associated learning tasks.
To see this, let $(\x, y)$ be drawn from an arbitrary data distribution on $\R^d \times \{\pm1\}$ satisfying
\begin{equation}
	\prob(\x \in \set^+ \suchthat y= +1) = 1 = \prob(\x \in \set^- \suchthat y=-1),
\end{equation}
i.e., the binary label $Y$ is consistent with the classes $\set^-$ and $\set^+$.
Conditioned on the high-probability event of Problem~\ref{prob:rnn-sep}, the transformed pair $(\NN(\x), y)$ then fulfills a \emph{hard-margin condition}:
\begin{equation}\label{eq:intro:hard-margin-cond}
	\prob\big(y \cdot (\sp{\sepdir}{\NN(\x)}+\sepbias) \geq \marg \big) = 1,
\end{equation}
where $\hypp[\sepbias]{\sepdir}$ denotes the separating hyperplane in \eqref{eq:rnn-sep:hyper-sep}.
Given i.i.d.~training samples of $(\x, y)$, this enables us to learn the unknown output parameters $(\sepdir, \sepbias)$ by standard classification methods, such as support vector machines (SVMs) \citep{sc08}.
In particular, one can achieve provable control over the generalization error in terms of the margin size~$\marg$, e.g., see~\citet[Thm.~15.4]{sb14}.
Of similar relevance is the width~$\layersec{n}$ of the output-layer of $\NN$, as it determines the ambient dimension of the feature space and therefore the computational complexity of the classification method.
For these reasons, we seek to solve Problem~\ref{prob:rnn-sep} with reasonable bounds for both $\marg$ and $\layersec{n}$.

In principle, the aforementioned idea of using random NNs as a pre-processing step for well-understood (linear) classifiers is not new \citep{hzs06,rr08,zbhrv16}.
But despite conceptual overlaps, the analytical approach of the present article is different from most existing works, see also Section~\ref{sec:intro:literature} for a short literature overview. Although Problem~\ref{prob:rnn-sep} includes a large family of classification tasks --- namely all pairs of $\mdist$-separated sets --- we are primarily interested in an \emph{instance-specific} analysis:
our main results quantify the dependence of the key parameters~$\marg$ and~$\layersec{n}$ as functions of the underlying classes $\set^-$ and $\set^+$ and their \enquote{interaction}. The resulting instance-specific bounds allow us to avoid pessimistic (worst-case) bounds caused by pathological pairs of sets. In our analysis, we make no explicit assumptions about the data in Problem~\ref{prob:rnn-sep}, such as a handcrafted generative model or sampling from a generic distribution.
Instead, we will introduce complexity measures that quantify the geometric complexities of~$\set^-$ and~$\set^+$ as well as their mutual entanglement. This perspective appears more natural to us in the context of data-driven methods.

Let us now specify the class of random NNs for which we will explore Problem~\ref{prob:rnn-sep}.
Throughout, the function $\NN \colon \R^d\to \R^{\layersec{n}}$ will be composed of (hidden) layers of the following form.
\begin{definition}\label{def:intro:random-layer}
	We call $\NNlayer \colon \R^{n_{\text{in}}}\to \R^{n_{\text{out}}}$ a \emph{random ReLU-layer} with maximal bias $\biasp \geq 0$ if
	\begin{equation}\label{eq:intro:random-layer}
		\NNlayer(\x) = \sqrt{\tfrac{2}{n_{\text{out}}}} \cdot \relu(\W \x + \Bias), \qquad \x \in \R^{n_{\text{in}}},
	\end{equation}
	where the weight matrix $\W\in \R^{n_{\text{out}}\times n_{\text{in}}}$ has standard Gaussian entries, the bias vector $\Bias$ is uniformly distributed on $[-\biasp, \biasp]^{n_{\text{out}}}$, independently of $\W$, and the element-wise activation function is the rectified linear unit (ReLU), i.e., $\relu(s) \coloneqq \max\{0, s\}$ for $s \in \R$.
\end{definition}
It will turn out that already two random ReLU-layers are sufficient for our solution to Problem~\ref{prob:rnn-sep}, although in principle deeper architectures are also possible.
We consider the ReLU mainly because of its popularity, but our analysis can be adapted for other common activation functions, e.g., the thresholding activation.
Let us note that the random weights and normalization in Definition~\ref{def:intro:random-layer} do not exactly correspond to a standard initialization in deep learning. The closest is the popular \emph{He initialization} \citep{hzrs15}, which would be obtained by replacing $n_{\text{out}}$ by $n_{\text{in}}$ and taking $\Bias=\mathbf{0}$ in~\eqref{eq:intro:random-layer}. Our non-standard choice of the bias is due to a hyperplane tessellation argument used in the proof of our main result. The proof sketch in Section~\ref{sec:intro:proof} will provide an intuitive explanation for this choice; in particular, see Figure~\ref{fig:intro:tessellation}(b).

Instead of directly formulating our main result, Theorem~\ref{thm:main}, we will first present several readily accessible special cases of increasing generality. Afterwards, we will highlight our proof strategy in Section~\ref{sec:intro:proof}.

\subsection{A Gentle Start: Finite Sets and Memorization}
\label{sec:intro:memo}

Our first result below gives an answer to Problem~\ref{prob:rnn-sep} in the situation of \emph{finite} point sets.
In the following, $\ball[2][d] \coloneqq \{ \x \in \R^d \suchthat \lnorm{\x} \leq 1 \}$ denotes the Euclidean unit ball; see also Section~\ref{sec:intro:notation} for a summary of common notation used in this article.
\begin{theorem}[Finite sets]\label{thm:memorization}
	There exist absolute constants $c, C > 0$ such that the following holds. 
	
	Let $\set^-, \set^+\subset\ball[2][d]$ be $\mdist$-separated sets with $N^+ \coloneqq \cardinality{\set^+} < \infty$ and $N^- \coloneqq \cardinality{\set^-} < \infty$. 
	Suppose that $\biasp \gtrsim  \sqrt{\log(e\biasp/\mdist)}$. Let $\NNlayer \colon \R^d\to \R^{n}$ and $\NNlayersec \colon \R^{n}\to \R^{\layersec{n}}$ be two (independent) random ReLU-layers with maximal biases $\biasp, \layersec{\biasp} \geq 0$, respectively, such that
	\begin{align}
		n&\gtrsim \big(\tfrac{\biasp}{\mdist}\big)^{8} \cdot \log (2N^-N^+/\probsuccess)
		\label{eq:memorization:first_layer}
	\end{align}
	and
	\begin{equation}
		\layersec{\biasp}\gtrsim  \big(\tfrac{\biasp}{\mdist}\big)^{4} \cdot \big(\alpha+\biasp \big),
		\qquad \layersec{n}\gtrsim \tfrac{\layersec{\biasp}}{\biasp} \cdot \theta \cdot \log(N^-/\probsuccess), 
		\label{eq:memorization:second_layer}
	\end{equation}
	where $\alpha=\sqrt{\log N^+}$ and
	\begin{equation}
		\label{eqn:exponentialTerm}
		\theta = \exp\Big(C \cdot \big(\alpha^2+\biasp^2\big) \cdot \biasp^6\cdot \mdist^{-8} \cdot \log(\biasp/\mdist)\Big).
	\end{equation}
	Then, given the two-layer random NN $\NN \colon \R^d\to \R^{\layersec{n}}, \ \x \mapsto \NNlayersec(\NNlayer(\x))$, with probability at least $1-\probsuccess$, the sets $\NN(\set^-), \NN(\set^+) \subset \layersec{\biasp} \ball[2][\layersec{n}]$ are linearly separable with margin $c\biasp^2/(\layersec{\biasp}\theta)$.
\end{theorem}
In this result, the best choices for $\biasp$ and $\layersec{\biasp}$ are the minimal settings that satisfy the stated bounds. The governing condition in Theorem~\ref{thm:memorization} (and in all following results, Theorems~\ref{thm:eucl_balls}, \ref{thm:uniform_covering}, and~\ref{thm:main} below) is condition \eqref{eq:memorization:second_layer} on the width $\layersec{n}$ of the second layer. It features the term~$\theta$ that scales exponentially in terms of the logarithm of the number of points, so that $\layersec{n}$ scales as $N_+^{\text{poly}(\biasp,1/\mdist)}$, in contrast to the logarithmic scaling of $n$ in \eqref{eq:memorization:first_layer}. To gauge whether this condition is necessary, let us connect Theorem~\ref{thm:memorization} to the \emph{memorization capacity} of random NNs. The ability of memorizing large data sets (including their noisy components) is a well-known phenomenon in deep learning research and considered as an important piece of the still unsolved generalization puzzle~\citep{zbhrv16,zbhrv21}. Theorem~\ref{thm:memorization} applies to any (\mbox{$\mdist$-separated}) completely unstructured data set --- imagine a point cloud with arbitrary binary labels. Remarkably, one can therefore memorize the labels of any such (finite) set with high probability by efficiently computing a separating hyperplane of~$\NN(\set^-)$ and~$\NN(\set^+)$, e.g., using a hard-margin SVM.\footnote{This observation is especially interesting when no obvious learning rule is available. On the other hand, if the data carries more structure (e.g., $\set^-$ and~$\set^+$ are already linearly separable), there certainly exist more effective approaches than randomized transforms.}

Although this shows that random NNs can be powerful memorizers in practice, existing results in the literature indicate that perfect memorization is already possible when the number of neurons scales linearly in $(N^- + N^+)$ up to logarithmic factors, e.g., see~\citet{ysj19,ver20,bn20}. Hence, we expect that the dependence on $\mdist$ and $\biasp$ within the exponential term $\theta$ in \protect\eqref{eq:memorization:second_layer} may be improved.

\subsection{Separation of Euclidean Balls}
\label{sec:intro:eucl_balls}

Although Theorem~\ref{thm:memorization} provides a margin bound, its actual size was not relevant to the network's memorization capacity.
The situation is different for \emph{infinite} \emph{classes}, on which we will focus from now on.
Problem~\ref{prob:rnn-sep} is then connected to a binary classification task through the hard-margin condition \eqref{eq:intro:hard-margin-cond}, and the margin size determines the generalization performance \citep[Thm.~15.4]{sb14}.
Our next result may be seen as a natural extension of Theorem~\ref{thm:memorization}, replacing discrete data points by a finite collection of Euclidean balls; see Figure~\ref{fig:intro:eucl_balls} for an illustration of this model.
\begin{theorem}[Euclidean balls]\label{thm:eucl_balls}
	There exist absolute constants $c, C > 0$ such that the following holds. 
	
	Let $\set^-, \set^+\subset \ball[2][d]$ be $\mdist$-separated sets that can be written as the union of finitely many Euclidean balls of radius $\rad \geq 0$, i.e.,
	\begin{equation}
		\textstyle\set^- = \bigunion_{l\in [N^-]}\ball[2][d](\cent^-_l,\rad),\quad \set^+ = \bigunion_{j\in [N^+]}\ball[2][d](\cent^+_j,\rad). 
	\end{equation}
	Suppose that $\biasp \gtrsim \sqrt{\log(e\biasp/\mdist)}$ and $\rad \lesssim \mdist^2/\biasp$.
	We assume that $\NNlayer \colon \R^d\to \R^{n}$ and $\NNlayersec \colon \R^{n}\to\nobreak \R^{\layersec{n}}$ are two (independent) random ReLU-layers with maximal biases $\biasp, \layersec{\biasp} \geq 0$, respectively, such that
	\begin{align}\label{eq:eucl_balls:first_layer}
		n&\gtrsim (1+\biasp^6 \mdist^{-8} r^2) \cdot d + \big(\tfrac{\biasp}{\mdist}\big)^8 \cdot \log(2 N^-N^+/\probsuccess)
	\end{align}
	and \eqref{eq:memorization:second_layer} holds with $\alpha = \rad\sqrt{d}+\sqrt{\log N^+}$ and $\theta$ as in \eqref{eqn:exponentialTerm}.
	Then, given the two-layer random NN $\NN \colon \R^d\to \R^{\layersec{n}}, \ \x \mapsto \NNlayersec(\NNlayer(\x))$, with probability at least $1-\probsuccess$, the sets $\NN(\set^-), \NN(\set^+) \subset \layersec{\biasp} \ball[2][\layersec{n}]$ are linearly separable with margin $c\biasp^2/(\layersec{\biasp}\theta)$.
\end{theorem}
\begin{figure}[t]
	\centering
	\includegraphics[width=0.35\linewidth]{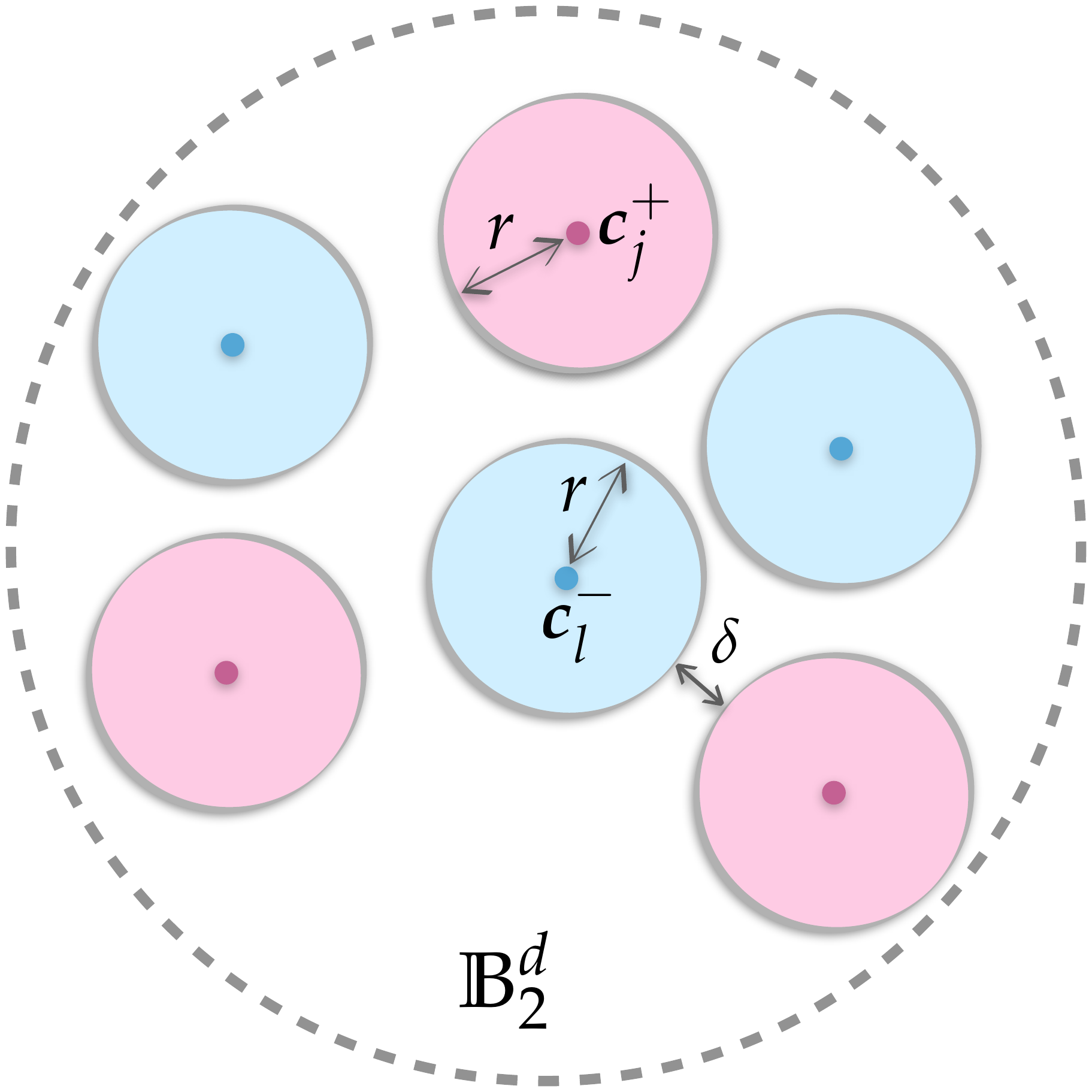}
	\caption{\textbf{Illustration of the Euclidean-ball model in Theorem~\ref{thm:eucl_balls}.} The sets $\set^-, \set^+ \subset \ball[2][d]$ consist of Euclidean balls of radius $\rad$, where the center points are denoted by $\cent^-_l$ and $\cent^+_j$, respectively. Note that the $\mdist$-separation only concerns balls of different classes, while arbitrary intersections are allowed within each class.}
	\label{fig:intro:eucl_balls}
\end{figure}
It is worth noting that in the limit case $\rad = 0$, the above statement is essentially consistent with Theorem~\ref{thm:memorization}.
On the other hand, Theorem~\ref{thm:eucl_balls} reveals the price of dealing with full-dimensional sets instead of points: the bound on the output dimension $\hat{n}$ scales exponentially in terms of~$\rad^2 d$.
Thus, to avoid the curse of dimensionality, the radius needs to satisfy $\rad \lesssim 1 / \sqrt{d}$.
Under this assumption, Theorem~\ref{thm:eucl_balls} certifies that random NNs can efficiently separate (unstructured) collections of Euclidean balls.

\subsection{Towards a General Separation Guarantee}
\label{sec:intro:uniform_covering}

So far, we have only considered specific examples of data sets.
Our next theorem concerns Problem~\ref{prob:rnn-sep} for arbitrary $\mdist$-separated classes.
For a formal statement, we need to introduce two important geometric parameters. The \emph{covering number} of a bounded subset $\set \subset \R^d$ at scale $\rad > 0$ is given by
\begin{equation}\label{eq:intro:covering_number}
	\covnumber{\set}{\rad} \coloneqq \min\Big\{ N \in \N \suchthat \exists \cent_1, \dots, \cent_N \in \R^d \colon \set \subset \textstyle\bigunion_{j\in [N]}\ball[2][d](\cent_j,\rad) \Big\},
\end{equation}
i.e., the smallest number of Euclidean balls of radius $\rad$ required to cover $\set$.
Moreover, the \emph{(Gaussian) mean width} of $\set$ is defined as
\begin{equation}\label{eq:intro:mean_width}
	\meanwidth{\set} \coloneqq \mean_{\g}\Big[ \sup_{\x \in \set} \sp{\g}{\x} \Big],
\end{equation}
where $\g \distributed \Normdistr{\vnull}{\I{d}}$ denotes a standard Gaussian random vector.
Both $\covnumber{\set}{\rad}$ and $\meanwidth{\set}$ are natural complexity measures, which are well-established in high-dimensional geometry, statistics, and signal processing, e.g., see~\citet{gm04,crpw12,tal14,ver18}.
\begin{theorem}[General sets]\label{thm:uniform_covering}
	There exist absolute constants $c, C > 0$ such that the following holds. 
	
	Let $\set^-, \set^+\subset \ball[2][d]$ be $\mdist$-separated sets and suppose that $\biasp \gtrsim  \sqrt{\log(e\biasp/\mdist)}$.
	Moreover, let $N^-\coloneqq \covnumber{\set^-}{c\mdist^2/\biasp}$ and $N^+\coloneqq \covnumber{\set^+}{c\mdist^2/\biasp}$.
	We assume that $\NNlayer \colon \R^d\to \R^{n}$ and $\NNlayersec \colon \R^{n}\to\nobreak \R^{\layersec{n}}$ are two (independent) random ReLU-layers with maximal biases $\biasp, \layersec{\biasp} \geq 0$, respectively, such that\footnote{Note that the first line in \eqref{eq:uniform_covering:first_layer} is always satisfied if $n \gtrsim d$, since we have that $\effdim{\S^{d-1}} \asymp d$. See also Section~\ref{sec:intro:notation} for a precise definition of $\cone{\cdot}$.}
	\begin{align}
		n&\gtrsim  \effdim{\cone{\set^--\set^-}\intersec\S^{d-1}} + \effdim{\cone{\set^+-\set^+}\intersec\S^{d-1}}, \\
		n&\gtrsim \big(\tfrac{\biasp}{\mdist}\big)^{8} \cdot \Big( \biasp^{-2}\big(\effdim{\set^-}+\effdim{\set^+}\big) + \log (2N^-N^+/\probsuccess) \Big)
		\label{eq:uniform_covering:first_layer}
	\end{align}
	and \eqref{eq:memorization:second_layer} holds with $\alpha=\meanwidth{\set^-}+\meanwidth{\set^+}$ and $\theta$ as in \eqref{eqn:exponentialTerm}. Then, given the two-layer random NN $\NN \colon \R^d\to \R^{\layersec{n}}, \ \x \mapsto \NNlayersec(\NNlayer(\x))$, with probability at least $1-\probsuccess$, the sets $\NN(\set^-), \NN(\set^+) \subset \layersec{\biasp} \ball[2][\layersec{n}]$ are linearly separable with margin $c\biasp^2/(\layersec{\biasp}\theta)$.
\end{theorem}
Compared to Theorem~\ref{thm:eucl_balls}, the above guarantee yields a much stronger statement due to the use of the mean width as a complexity measure.
To see this, let $\set^+ \subset \bigunion_{j\in [N^+]}\ball[2][d](\cent^+_j,\rad)$ be any covering of $\set^+$ at scale $\rad = c \biasp^{-1}\mdist^2$ and consider the following upper bound (see Lemma~\ref{lem:Gaussian_width_union}):
\begin{equation}\label{eq:intro:mean_width_localized}
	\meanwidth{\set^+} \lesssim  w^+ + \sqrt{\log N^+},
\end{equation}
where $w^+=\max_{j\in[N^+]}\meanwidth{\set^+ \intersec \ball[2][d](\cent^+_j,\rad)}$. An analogous bound holds for $\set^-$. While using the worst-case estimate $\mwloc^+ \lesssim \rad \sqrt{d}$ would lead to a similar bottleneck as in Theorem~\ref{thm:eucl_balls}, the \emph{localized} mean width parameter $\mwloc^+$ can be substantially smaller for structured data sets. Typical examples are data residing on a low-dimensional manifold or contained in the convex hull of finitely many points, see Remark~\ref{rmk:meanwidth} below for some concrete examples.
On the other hand, the covering number $N^+$ reflects the \emph{global} size of $\set^+$ in \eqref{eq:intro:mean_width_localized}.

For these reasons, Theorem~\ref{thm:uniform_covering} takes an important step towards a general solution to Problem~\ref{prob:rnn-sep}, which meets our overall goal of \emph{instance-specific} bounds for the separation margin and layer widths.
Having said this, the geometric parameters in this result only capture the individual complexities of $\set^-$ and $\set^+$, but remain silent about their mutual arrangement.
For instance, one would expect that two sets become easier to separate if their \enquote{centers of mass} are farther apart, even though the minimal distance $\mdist$ is small; see Figure~\ref{fig:main:mutual_covering} in Section~\ref{sec:main} for an illustration.
In such scenarios, a non-uniform covering strategy for $\set^-$ and~$\set^+$ is preferable, in the sense that data points far away from the decision boundary should be covered by fewer but larger balls.
The most general outcome of this work, Theorem~\ref{thm:main}, makes this intuition precise by employing a novel notion of \emph{mutual complexity} (see Definition~\ref{def:mutual_covering} and~\ref{def:mutual_complexity}).
We refer to Section~\ref{sec:main} for an in-depth discussion and further refinements due to Theorem~\ref{thm:main}.
Finally, we emphasize that all previously presented results follow from Theorem~\ref{thm:main} as special cases, see Section~\ref{sec:special_cases} for detailed proofs.

We close this part with a few examples of concrete bounds on the mean width to highlight its usefulness as a complexity measure:
	\begin{remark}[Controlling the mean width]
		\begin{rmklist}
			\item 
			\emph{Worst-case bound.} Since $\set^-, \set^+\subset\nobreak \ball[2][d]$, the mean width parameter in Theorem~\ref{thm:uniform_covering} satisfies the trivial bound
			\begin{equation}
				\alpha = \meanwidth{\set^-} + \meanwidth{\set^+} \lesssim \meanwidth{\ball[2][d]} \asymp \sqrt{d}.
			\end{equation}
			Thus, an exponential width of the second layer in terms of~$d$ allows us to solve Problem~\ref{prob:rnn-sep} for arbitrary $\mdist$-separated sets, regardless of their specific shape.
			\item
			\emph{Low-dimensional subspaces.}
			As highlighted above, already much smaller networks can achieve separation if the underlying classes carry more structure.
			A typical example of low-complexity structure is a situation where $\set^-$ and $\set^+$ reside in a union of low-dimensional subspaces, say $\set^-, \set^+\subset \bigunion_{j\in [N]} L_j \intersec \ball[2][d]$ with $\dim L_j \ll d$. Then,
			\begin{equation}
				\alpha = \meanwidth{\set^-} + \meanwidth{\set^+} \lesssim \max_{j\in[N]}\sqrt{\dim L_j}+ \sqrt{\log N} \ll \sqrt{d},
			\end{equation}
			assuming that $N$ is not exponentially large (see Lemma~\ref{lem:Gaussian_width_union}).
			\item
			\emph{Point clouds and their convex hulls.}
			Another important example of a low-complexity set is the convex hull of finitely many points.
			Indeed, assuming $\set^-, \set^+\subset \convhull{\x_1, \dots, \x_N} \subset\nobreak \ball[2][d]$, the mean width only scales logarithmically in $N$:
			\begin{equation}
				\alpha = \meanwidth{\set^-} + \meanwidth{\set^+} \leq 2 \cdot \meanwidth{\convhull{\x_1, \dots, \x_N}} = 2 \cdot \meanwidth{\{\x_1, \dots, \x_N\}} \lesssim \sqrt{\log N},
			\end{equation}
			where we have used a basic bound on the mean width (e.g., see \citealp[Ex.~1.3.8]{ver15}) and its invariance under taking the convex hull.
			Note that this bound particularly extends the situation of finitely many data points from Theorem~\ref{thm:memorization} to infinite data sets.
		\end{rmklist}\label{rmk:meanwidth}
	\end{remark}

\subsection{Proof Strategy}
\label{sec:intro:proof}

To keep our exposition as simple as possible, we will describe our proof strategy in the prototypical situation of Euclidean balls from Theorem~\ref{thm:eucl_balls}, see also Figure~\ref{fig:intro:eucl_balls}.
Recall that $\setcen^- \coloneqq \{\cent^-_1, \dots, \cent^-_{N^-}\}$ and $\setcen^+ \coloneqq \{\cent^+_1, \dots, \cent^+_{N^+}\}$ denote the center points of~$\set^-$ and~$\set^+$, respectively. For convenience, we also set $\set_l^- \coloneqq \ball[2][d](\cent^-_l,\rad)$ and $\set_j^+ \coloneqq \ball[2][d](\cent^+_j,\rad)$ so that $\set^- = \bigunion_{l\in [N^-]}\set_l^-$ and $\set^+ = \bigunion_{j\in [N^+]}\set_j^+$.

Our data separation approach consists of a \emph{two-step procedure}, which essentially corresponds to the composition of the random layers $\NNlayer \colon \R^d\to \R^{n}$ and $\NNlayersec \colon \R^{n}\to \R^{\layersec{n}}$.
Although $\NNlayer$ and $\NNlayersec$ both follow exactly the same random design (see Definition~\ref{def:intro:random-layer}), we will see that their purposes are different:
while the first one already establishes a desirable geometrical configuration under mild conditions, a major challenge is to show that one can actually take advantage of it by applying a second (wider) random layer.
The central finding of our geometric analysis of Problem~\ref{prob:rnn-sep} is a subtle interplay between the separation capacity of random NNs and their stability properties.
In particular, we demonstrate that the linearization of complicated data is possible on a global scale, without too much disturbing its local geometry, e.g., Euclidean point distances.

To understand the effect of the first layer $\NNlayer$, it is useful to take a coordinate-wise perspective:
\begin{equation}
	[\NNlayer(\x)]_i = \sqrt{\tfrac{2}{n}}\relu(\sp{\w_i}{\x}+\bias_i), \quad i = 1, \dots, n, \quad \x \in \R^d,
\end{equation}
where $\w_i \in \R^d$ is the $i$-th row of the weight matrix $\W\in \R^{n\times d}$ and $b_i \in [-\biasp, \biasp]$ the corresponding bias. Thus, $[\NNlayer(\x)]_i$ indicates on which side of the (unnormalized) random hyperplane $\hypp[\bias_i]{\w_{i}}$ a given point $\x \in \R^d$ lies.
Our first major proof step (elaborated in Theorem~\ref{thm:lin_sep_random_ReLU/Thres_NN_finite_1st_layer}) shows that as long as $\biasp \gtrsim 1$, the following holds with high probability:
for every pair of center points $(\cent^-_l, \cent^+_j) \in \setcen^- \times \setcen^+$, there are coordinates $\setind_{l,j} \subset [n]$ with $\cardinality{\setind_{l,j}} \gtrsim \mdist \biasp^{-1} n$ such that $\hypp[\bias_i]{\w_{i}}$ separates $\cent^-_l$ from $\cent^+_j$ for all $i \in \setind_{l,j}$, in fact
\begin{equation}\label{eq:intro:proof:sep_layer1}
	[\NNlayer(\cent^-_l)]_i = 0 \quad \text{and} \quad [\NNlayer(\cent^+_j)]_i \gtrsim \tfrac{\mdist}{\sqrt{n}},
\end{equation}
see Figure~\ref{fig:intro:tessellation} for an illustration.
The key insight to show \eqref{eq:intro:proof:sep_layer1} is that the probability of a single random hyperplane separating a fixed pair of \mbox{$\mdist$-separated} points is of order $\asympslower{\mdist \biasp^{-1}}$, see Theorem~\ref{thm:prob_single_hyperplane_separates_two_points}.
Combining this with a Chernoff bound over all hyperplanes associated with $\NNlayer$ then leads to a high-probability event of the above type.
A remarkable fact about this argument is that the required layer width~$n$ is very moderate, scaling only logarithmically with the number of center points (see \eqref{eq:eucl_balls:first_layer}).

\begin{figure}[t]
	\centering
	\includegraphics[width=0.85\linewidth]{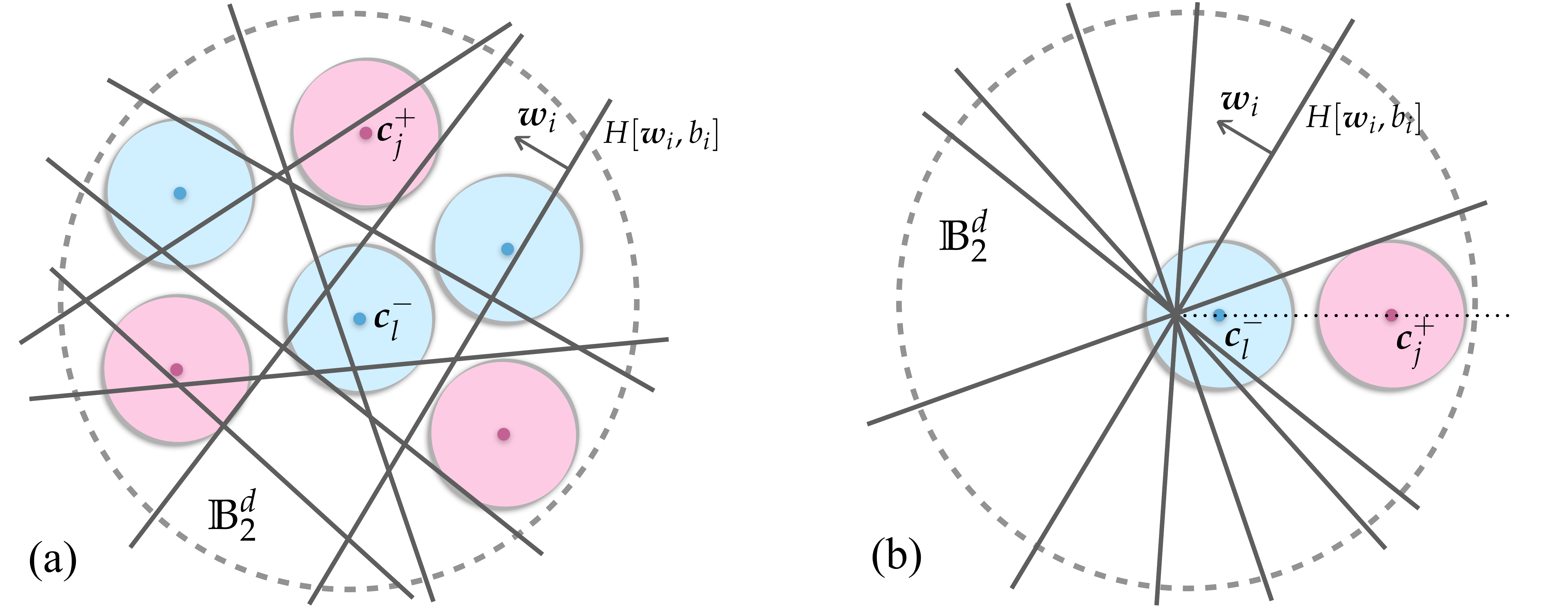}
	\caption{\textbf{Random hyperplanes in the input domain.} (a) Each coordinate of the first layer $\NNlayer \colon \R^d\to \R^{n}$ can be associated with a random hyperplane $\hypp[\bias_i]{\w_{i}}$, where $\w_{i} \distributed \Normdistr{\vnull}{\I{d}}$ and $\bias_i$ is uniformly distributed on $[-\biasp, \biasp]$. Already a relatively few of such hyperplanes are enough to separate every pair of center points $(\cent^-_l, \cent^+_j) \in \setcen^- \times \setcen^+$ at least once. For this, a sufficiently large bias parameter ($\biasp \gtrsim 1$) is vital, as it ensures a \emph{uniform tessellation} of the input domain $\ball[2][d]$; otherwise, the probability of separating points close to the boundary of $\ball[2][d]$ would become too low. Subfigure (b) illustrates what could go wrong for $\biasp = 0$: if the center points reside on a ray starting at the origin, a separation by hyperplanes without offsets becomes impossible.}
	\label{fig:intro:tessellation}
\end{figure}

A crucial part of \eqref{eq:intro:proof:sep_layer1} is that the corresponding coordinates in $\NNlayer(\cent^-_l)$ are actually vanishing, due to the \emph{non-linear} activation.
	Based on this, we can explicitly construct a (normalized) hyperplane $\hypp[\sepbias_l]{\sepdir_l}$ for each $l \in\nobreak [N^-]$ that separates $\NNlayer(\cent^-_l)$ and $\NNlayer(\setcen^+)$ with margin $\tilde{\marg} \asymp \mdist^2 \biasp^{-1}$ (see Theorem~\ref{thm:lin_sep_random_ReLU/Thres_NN_finite_1st_layer}).
	The resulting arrangement of the transformed sets $\NNlayer(\setcen^-)$ and $\NNlayer(\setcen^+)$ resembles a big \enquote{planet} which is orbited by small \enquote{satellites} and is illustrated in Figure~\ref{fig:intro:first_layer}.\footnote{Note that such a geometric arrangement would not be achievable without some kind of non-linearity in~$\NNlayer$. For example, imagine a series of points on a straight line where the class label $\pm1$ alternates with each point. This arrangement cannot be transformed into the situation in Figure~\ref{fig:intro:first_layer} by an affine map, which maps lines to lines.}

To conclude with the first layer, we need to ensure that the simplification achieved by~$\NNlayer$ does not only apply to the center points but to the entire data set.
Indeed, leveraging the \emph{geometry-preserving} properties of random ReLU-layers (i.e., $\NNlayer$ preserves $\l{2}$-distances between nearby points; see Theorem~\ref{thm:distance_preservation_ReLU}), it follows that $\NNlayer(\set_l^-)$ and $\NNlayer(\set^+)$ are also linearly separable for every $l \in [N^-]$ (still with margin $\tilde{\marg} \asymp \mdist^2 \biasp^{-1}$, see Theorem~\ref{thm:lin_sep_random_ReLU_NN_continuous_1st_layer}). Hence, the geometric picture of Figure~\ref{fig:intro:first_layer} remains true when replacing~$\setcen^-$ and~$\setcen^+$ by~$\set^-$ and~$\set^+$, respectively.\footnote{Using such distance preservation properties of $\NNlayer$ is very different from a direct approach, according to which a random hyperplane $\hypp[\bias_i]{\w_{i}}$ would separate pairs of balls $(\set^-_l, \set^+_j)$. In fact, the latter event is much less likely than the separation of single points by a random hyperplane (this will become clear with Theorem~\ref{thm:hyperplane_sep:general} below). Instead, distance preservation is an integral property of the random layer, which exploits information from \emph{all} coordinates of $\NNlayer$.}
As for \eqref{eq:intro:proof:sep_layer1}, this reasoning requires the bias $\Bias$ to be large enough (i.e., $\biasp \gtrsim\nobreak \sqrt{\log(\biasp/\mdist)}$) and exploits that it is \emph{uniformly} distributed. 

\begin{figure}[t]
	\centering
	\includegraphics[width=0.6\linewidth]{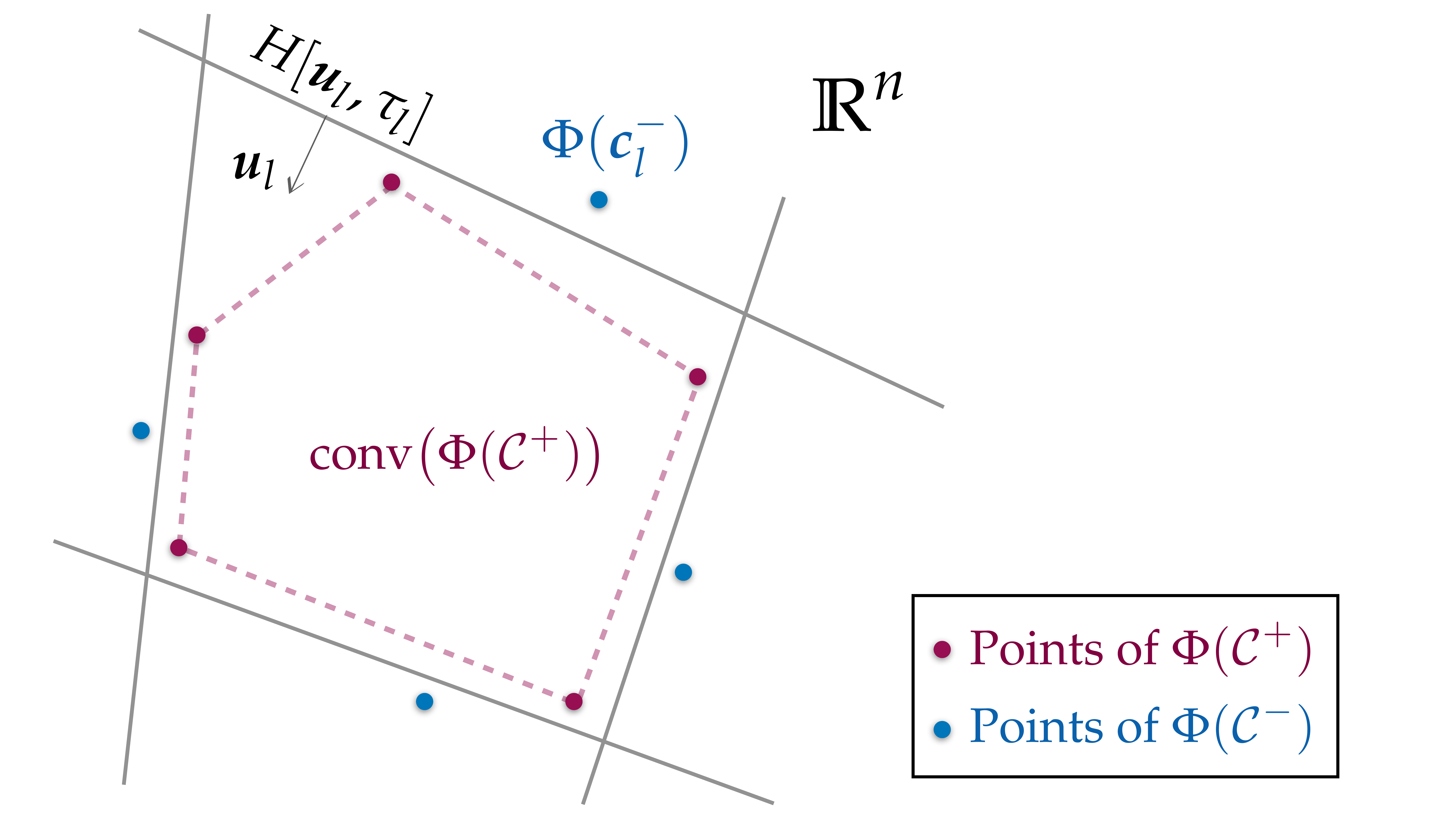}
	\caption{\textbf{The geometric effect of the first random ReLU-layer.} For each $l \in [N^-]$, the sets $\NNlayer(\cent^-_l)$ and $\NNlayer(\setcen^+)$ are linearly separable, or equivalently, it holds that $\NNlayer(\setcen^-) \intersec \convhull{\NNlayer(\setcen^+)} = \emptyset$, where $\convhull{\cdot}$ is the convex hull operator. One can picture $\convhull{\NNlayer(\setcen^+)}$ as a big \enquote{planet} which is orbited by small \enquote{satellites} namely the transformed center points in $\NNlayer(\setcen^-)$. For symmetry reasons, an analogous statement holds with high probability if the roles of $\setcen^-$ and $\setcen^+$ are interchanged.}
	\label{fig:intro:first_layer}
\end{figure}

Let us now turn to the second layer $\layersec{\NNlayer} \colon \R^{n}\to \R^{\layersec{n}}$, which builds directly on the geometric situation after applying $\NNlayer$.
It is again helpful to treat each coordinate individually:
\begin{equation}
	[\layersec{\NNlayer}(\x)]_i = \sqrt{\tfrac{2}{\layersec{n}}}\relu(\sp{\layersec{\w}_i}{\x}+\layersec{\bias}_i), \quad i = 1, \dots, \layersec{n}, \quad \x \in \R^n,
\end{equation}
where $\layersec{\w}_i \distributed \Normdistr{\vnull}{\I{n}}$ and $\layersec{\bias}_i$ is uniformly distributed on $[-\layersec{\biasp}, \layersec{\biasp}]$.
Our main goal is to show that for every $l \in [N^-]$ there exist sufficiently many coordinates $i \in [\layersec{n}]$ such that
\begin{equation}\label{eq:intro:proof:sep_layer2}
	[\layersec{\NNlayer}(\NNlayer(\set_l^-))]_i \geq \tfrac{t}{\sqrt{\layersec{n}}} \quad \text{and} \quad [\layersec{\NNlayer}(\NNlayer(\set^+))]_i = 0,
\end{equation}
where $t > 0$ depends on the complexity of $\set^-$ and $\set^+$; note that the vanishing coordinates are associated with data from~$\set^+$ instead of~$\set^-$.
Our basic strategy to establish \eqref{eq:intro:proof:sep_layer2} is similar to \eqref{eq:intro:proof:sep_layer1}, but there is a major difference: we now have to deal with the probability that a random hyperplane separates the sets $\layersec\set_l^- \coloneqq \NNlayer(\set_l^-)$ and $\layersec\set^+ \coloneqq \NNlayer(\set^+)$.
The outcome of the first layer implies the existence of a separator, e.g., $\hypp[\sepbias_l]{\sepdir_l}$ (see also Figure~\ref{fig:intro:first_layer}), but this does not mean that it is likely to be found by a single random draw.\footnote{To be clear about this point, $\hypp[\sepbias_l]{\sepdir_l}$ does explicitly depend on the unknown sets $\set^-$ and $\set^+$. Hence, in contrast to random hyperplanes, it cannot be used for practical purposes offhand.}
Certainly, the probability of successful separation may not only depend on the distance between $\layersec\set_l^-$ and $\layersec\set^+$, but also on their complexity and mutual arrangement.
The following result makes this concern precise and forms a key component of our analysis; its proof can be found in Section~\ref{sec:proof:hyperplane_sep}.
The notion of $(\narrow,\narrowmdist)$-linear separability used below is formally introduced in Definition~\ref{def:separable_refined}; for now, it is useful to think of a refinement of linear separability that captures how much a separating hyperplane for two sets can be perturbed, such that it is still separates the sets.
\begin{theorem}\label{thm:hyperplane_sep:general}
	There exists an absolute constant $C>0$ such that the following holds.
	
	For $\narrow\in [0,1]$, $\narrowmdist>0$, and $\Rad\geq 1$, let $\set^-, \set^+\subset \Rad \ball[2][d]$ be two $(\narrow,\narrowmdist)$-linearly separable sets\footnote{Note that the notation for $\set^-$ and $\set^+$ is generic here. In the proof of our main result, we will apply Theorem~\ref{thm:hyperplane_sep:general} with $\set^- \coloneqq \layersec\set^+$ and $\set^+ \coloneqq \layersec\set_l^-$.} and put $\marg \coloneqq (1-\narrow)\narrowmdist$.
	Let $\g \distributed \Normdistr{\vnull}{\I{d}}$ be a standard Gaussian vector and let $\sepbias$ be uniformly distributed on $[-\biasp,\biasp]$ for some $\biasp > 0$. For any $\hyppsep \gtrsim  \meanwidth{\set^+-\set^-}+\Rad$ with $\biasp\gtrsim \Rad\hyppsep\marg^{-1}$, the hyperplane $\hypp[\sepbias]{\g}$ $\hyppsep$-separates $\set^-$ from $\set^+$ with probability at least 
	\begin{equation}\label{eq:hyperplane_sep:general:prob}
		\tfrac{\hyppsep}{\biasp} \cdot \exp\big({-C} \cdot \hyppsep^2\marg^{-2} \cdot\log(4(1-\narrow)^{-1})\big).
	\end{equation}
\end{theorem}
While the factor $\hyppsep / \biasp$ may become small due to the condition $\biasp\gtrsim \Rad\hyppsep\marg^{-1}$, the dominating term in \eqref{eq:hyperplane_sep:general:prob} is the exponential one. In fact, the central element of Theorem~\ref{thm:hyperplane_sep:general} is the mean width $\meanwidth{\set^+-\set^-}$, which dictates the severity of the exponential decay in \eqref{eq:hyperplane_sep:general:prob}.

An appropriate combination of Theorem~\ref{thm:hyperplane_sep:general} with a Chernoff bound will allow us to derive a statement of the form \eqref{eq:intro:proof:sep_layer2}.
With this at hand, it is then relatively straightforward to show that $\NN(\set^-) = \layersec{\NNlayer}(\NNlayer(\set^-))$ and $\NN(\set^+) = \layersec{\NNlayer}(\NNlayer(\set^+))$ are indeed linearly separable (see~Corollary~\ref{coro:lin_sep}).
Noteworthy is that the resulting margin and the number of required neurons~$\layersec{n}$ both inherit the exponential scaling from \eqref{eq:hyperplane_sep:general:prob}, which is reflected in all presented separation guarantees.
This observation particularly explains why the mean width appears as a natural measure of complexity for the data sets.

To the best of our knowledge, Theorem~\ref{thm:hyperplane_sep:general} is a new result and could be of independent interest: it concerns the fundamental question of when pairs of sets are likely to be separated by a random hyperplane and when not.
Perhaps not very surprisingly, the probability of success might scale poorly in the worst case, which is an inevitable consequence of the concentration of measure phenomenon.
But for highly structured (low-dimensional) sets, the situation can be much more benign;
finite point sets as considered in our analysis of the first layer are a good example.

\begin{remark}
	Using the definition of $(\narrow,\narrowmdist)$-linear separability (see Definition~\ref{def:separable_refined}), one can show that $\meanwidth{\set^+-\set^-} \lesssim \Rad \sqrt{\narrow d}$ holds in the setup of Theorem~\ref{thm:hyperplane_sep:general}.
	This general upper bound indicates that in the worst case, the probability of separation may decrease exponentially with the ambient dimension~$d$ (unless $\narrow \lesssim \tfrac{1}{d}$).
\end{remark}

\subsection{Related Literature}
\label{sec:intro:literature}

The two-step separation procedure underlying our proofs (see Section~\ref{sec:intro:proof}) is inspired by a construction of \citet{abb15}.
Their main result verifies that any two disjoint sets $\set^+, \set^- \subset \R^d$ can be made linearly separable by a \emph{deterministic} two-layer NN.
However, \citet{abb15} show a pure existence statement and their method is not feasible from an algorithmic perspective, since the selected weight parameters explicitly depend on the sets to be separated; furthermore, no informative bounds for the number of required neurons are provided.
By using random weights and suitable notions of complexity (namely mutual covering), we are able to derive much more practical separation guarantees, which eliminate the aforementioned shortcomings.
This achievement entails novel mathematical ingredients, most notably the separation capacity of random hyperplanes (see Theorem~\ref{thm:hyperplane_sep:general}) and uniform distance preservation by random ReLU-layers (see Theorem~\ref{thm:distance_preservation_ReLU}).

Below we will survey some works from the rich literature on random NNs that have notable conceptual similarities to our work. We are not aware of a comparable result that addresses the separation capacity of random NNs.

\paragraph{Approximation theory.} 
A very active line of research investigates to what extent random NNs are \emph{universal approximators} (e.g., see~\citealp{apvz14,sgt18,ys19,nnss20,hssv21} and the references therein). Specifically, one considers a class of real-valued functions on a domain in~$\R^d$ (e.g., continuous or Lipschitz functions), an approximation metric (typically the $L^2$- or $L^{\infty}$-norm), and a shallow NN consisting of a ReLU-layer with random weights followed by a linear layer with arbitrary weights (that may depend on the function to be approximated). The aforementioned works quantify which size of the random layer guarantees that the NN can reach a pre-specified approximation error for every function in the given class. These results feature an exponential bottleneck, in the sense that the width of the random layer needs to scale exponentially in terms of the data dimension $d$ to ensure accurate approximation (see also \citet{nnss20} for a refinement if the domain is a lower-dimensional smooth manifold).

In principle, one could try to approach Problem~\ref{prob:rnn-sep} by applying such an approximation result to a function $f:\R^d\to \R$ that takes values $+1$ and $-1$ on $\mathcal{X}^+$ and $\mathcal{X}^-$, respectively. However, existing approximation guarantees cannot certify a \emph{zero} approximation error on the two sets. In addition, even if an existing result would apply, it would lead to a more pessimistic statement that involves a random ReLU-layer whose width scales exponentially in terms of $d$, rather than the more refined, instance-specific complexity measures considered here.

\paragraph{Learning with random features.} 
The concept of random features was introduced by \citet{rr07} as a cheap computational alternative to kernel methods. The idea is to construct a random feature map such that inner products between random data features approximate kernel evaluations of the original data, provided that the feature dimension is high enough. A prime example are random Fourier features, which are designed to approximate the Gaussian kernel \citep{rr07}. Instead of a computationally expensive kernel method (e.g., kernel SVM), one can use a linear method (e.g., SVM) on the random features. A previous line of research has analyzed the generalization error of such methods, thereby quantifying the feature dimension that guarantees a performance on par with the associated kernel method (e.g., see~\citealp{rr08,rr17,bac17,sgt18,ltos19} and~\citealp{lhcs20} for a survey). Several of these results particularly apply when the feature map is a random ReLU-layer. Although these works indicate that the data is transformed in a beneficial way for learning, they do not have a direct connection to Problem~\ref{prob:rnn-sep}. Perhaps the closest connection can be found in \citealp{CaG19RF}, where it is shown that if the random ReLU feature function class from \citealp{rr08} can separate a finite set of data on the sphere, then a sufficiently wide random ReLU layer (without bias) can make the same data linearly separable with high probability, see \citealp[Asm.~4.10 and Lem.~B.2]{CaG19RF}. It is, however, unclear how to extend this statement to infinite datasets and how this separability assumption relates to the Euclidean separability assumption in Problem~\ref{prob:rnn-sep}.

\paragraph{Neural tangent kernels and mean field regime.} 
An intriguing finding of deep learning theory is that training randomly initialized NNs via gradient descent in the infinite-width limit is equivalent to kernel gradient descent with a specific type of kernel, called the neural tangent kernel (NTK); see \citealp{jgh18}.
The behaviour in the infinite-width limit has partially motivated a line of work on the analysis of (stochastic) gradient descent for training NNs in the overparametrized regime, starting from a random initialization, e.g., see \citealp{aro19,os19,yy18,du+19,als19,ZoG19,CaG19} and the references therein. These works have roughly shown that (S)GD can achieve an arbitrarily small training (and sometimes even generalization) error if the NN is wide enough and, moreover, the (S)GD iterates remain close to the initialization. The required width of the NN is implicitly or explicitly linked to the NTK. Most closely connected to our work are \citealp{NCS19,jt20,CCZG19}, which explicitly link the required width to the separation capacity of the infinite-width NTK-feature map at initialization.
As part of the analysis it is shown that if the training data satisfies a separability condition in the reproducing kernel Hilbert space induced by the infinite-width NTK, 
then the NTK-feature map associated with the finite-width random NN at initialization makes the training data linearly separable with high probability, e.g., see \citealp[Asm.~2.1, Lem. 2.3 \& Sec.~5]{jt20}. These results bear resemblance with Problem~\ref{prob:rnn-sep}, but there are several important differences. While we are primarily interested in the separation of \emph{infinite} data sets, these works focus on finite-sample scenarios. It is not clear how the latter could be extended accordingly. 
Moreover, note that the NTK-feature map associated with a finite-width random NN is not a random NN itself.\footnote{If $\NN_{\vec{\theta}} : \R^d \to \R$ is a NN with (randomly initialized) weights $\vec{\theta}$, then the associated NTK-feature map is $\x \mapsto \frac{\partial \NN_{\vec{\theta}}(\x)}{\partial \vec{\theta}}$.} Therefore, the aforementioned results do not address Problem~\ref{prob:rnn-sep} as such and the bounds on the network width needed to achieve linear separability are not directly comparable to ours.

Let us mention for completeness that the connection with the NTK arises due to our choice of scaling in the ReLU layers. A different scaling leads to the mean field regime \citep{mmn18}. The key insight of \citealp{mmn18} is that in the infinite-width limit, the gradient flow is captured by a specific non-linear partial differential equation (PDE).
Due to non-asymptotic bounds on the accuracy of this measure-valued PDE model, new convergence results for (S)GD can be derived. The connection between the kernel and mean field regimes is explained in detail in \citealp[Sec.~4 and App.~H]{mmm19}.

\paragraph{Random embeddings.} 
A key component of our analysis is the capability of random ReLU-layers to preserve Euclidean distances with high probability (see Theorem~\ref{thm:distance_preservation_ReLU}).
This finding is related to results on non-linear random embeddings, which play a major role in the field of quantized compressed sensing (e.g., see~\citealp{jlbb13,pv14,or15,cxj17,dm18,dir19,xj18,dms22a,dms22b}).
In particular, our choice of the bias vector (see Definition~\ref{def:intro:random-layer}) is inspired by \emph{dithering}, a technique that has already proven useful in various signal reconstruction problems \citep{jc17,dm18b,dm18,xj18,jmps21}.
A remarkable new (and somewhat counterintuitive) insight of the present work is that for appropriate non-linearities like the ReLU-activation, desirable distance preservation properties and data separation can be achieved simultaneously.

Theorem~\ref{thm:distance_preservation_ReLU} is new in its own right and improves on a previous result by \citet{gsb18}, see also \citet{gsb20}. It is also closely related to a work of \citet{ab19}, who have investigated the capability of a random ReLU-layer as in Definition~\ref{def:intro:random-layer} (but with bias $\Bias=\vnull$) to preserve Euclidean norms.

\paragraph{Rare eclipse problem.}
Finally, we point out an interesting connection between the separation capacity of random hyperplanes (see Theorem~\ref{thm:hyperplane_sep:general}) and the rare eclipse problem studied by \citet{bmr17,cxj17}. In both cases, the goal is to use a random transform $T:\R^d\to \R^k$ to map two linearly separable sets $\set^+, \set^-\subset \R^d$ into a lower dimensional space $\R^k$ such that the following holds with a certain probability $p$:
\begin{equation}\label{eq:rareeclipse}
	T(\set^+)\intersec T(\set^-)=\emptyset.
\end{equation}
More specifically, the rare eclipse problem asks how small $k$ can become such that \eqref{eq:rareeclipse} holds with probability at least $p=1-\eta$, where $\eta>0$ is fixed but can be arbitrarily small. Using Gordon's Escape Through a Mesh Theorem \citep{gor88}, \citet{bmr17} have shown that if~$\set^+$ and~$\set^-$ are disjoint, closed, and convex sets, then $k\gtrsim \effdim{\cone{\set^+-\set^-}\intersec \S^{d-1}}+\log(\eta^{-1})$ ensures \eqref{eq:rareeclipse} with probability at least $1-\eta$, where $T\in \R^{k\times d}$ is a standard Gaussian random matrix. 

In contrast, Theorem~\ref{thm:hyperplane_sep:general} considers a map of the form $T(\x)=\sp{\g}{\x}+\sepbias$, where $\g$ is a standard Gaussian random vector and $\sepbias\in [-\biasp, \biasp]$ uniformly distributed for $\biasp>0$ large enough.
If $\set^+, \set^-\subset \Rad\ball[2][d]$ are $(\narrow, \narrowmdist)$-linearly separable (with some minimal distance), then \eqref{eq:rareeclipse} holds with probability at least $\Rad\biasp^{-1}\exp(-C\effdim{\set^+-\set^-})$, where $C>0$ only depends on $\narrow$,~$\narrowmdist$, and~$\Rad$.
Hence, Theorem~\ref{thm:hyperplane_sep:general} guarantees disjoint sets even for a single coordinate ($k=1$), however at the expense of a worse probability of success.
Remarkably, the Gaussian mean width and the difference set $\set^+-\set^-$ play a key role both in the rare eclipse problem and Theorem~\ref{thm:hyperplane_sep:general}.

\subsection{Overview and Notation}
\label{sec:intro:notation}

The rest of the article is organized as follows: In Section~\ref{sec:main}, we present our main result, Theorem~\ref{thm:main}, based on the notion of mutual complexity (see Definition~\ref{def:mutual_covering} and~\ref{def:mutual_complexity}).
The next two sections are then devoted to our main mathematical tools, namely separation by random hyperplanes (Section~\ref{sec:separation}) and distance preservation (Section~\ref{sec:distance}).
Finally, the proof of Theorem~\ref{thm:main} is given in Section~\ref{sec:proof:main}, followed by a derivation of its variants (Theorem~\ref{thm:memorization},~\ref{thm:eucl_balls}, and~\ref{thm:uniform_covering}) in Section~\ref{sec:special_cases}.

Before proceeding, let us fix some standard notations and conventions that are commonly used in this paper.
The letters $c$ and $C$ denote absolute (positive) constants, whose values may change from line to line. We speak of an \emph{absolute constant} if its value does not depend on any other involved parameter.
If an inequality holds up to an absolute constant $C$, we usually write $A \lesssim B$ instead of $A \leq C \cdot B$. The notation $A \asymp B$ is a shortcut for $A \lesssim B \lesssim A$.

For $d \in \N$, we set $[d] \coloneqq \{1, \dots, d\}$. The \emph{cardinality} of an index set $\setind \subset [d]$ is denoted by~$\cardinality{\setind}$. Vectors and matrices are denoted by lower- and uppercase boldface letters, respectively.
The \mbox{$i$-th} entry of a vector $\vgen \in \R^d$ is denoted by $[\vgen]_i$, or simply by $z_i$ if there is no danger of confusion.
We write $\I{d} \in \R^{d \times d}$ and $\vnull \in \R^d$ for the \emph{identity matrix} and the \emph{zero vector} in $\R^d$, respectively.
For $1 \leq q \leq \infty$, we denote the \emph{$\l{q}$-norm} on $\R^d$ by $\lnorm{\cdot}[q]$ and the associated closed \emph{unit ball} by~$\ball[q][d]$.
The \emph{Euclidean unit sphere} is given by $\S^{d-1} \coloneqq \{ \vgen \in \R^d \suchthat \lnorm{\vgen}[2] = 1 \}$, and we also set $\S_+^{d-1} \coloneqq \S^{d-1} \intersec [0, \infty)^d$.

Let $\set, \set' \subset \R^d$ and $\vgen \in \R^d$. The \emph{linear cone} generated by~$\set$ is denoted by $\cone{\set} \coloneqq \{v \tilde{\vgen} \suchthat \tilde{\vgen} \in \set, v \geq 0 \}$.
The \emph{Minkowski difference} between $\set$ and $\set'$ is defined by $\set - \set' \coloneqq \{ \vgen_1 -\nobreak \vgen_2 \suchthat \vgen_1 \in \set, \vgen_2 \in \set' \}$, and we use the shortcut $\set - \vgen \coloneqq \set - \{\vgen\}$.
The \emph{distance} between~$\vgen$ and~$\set$ is $\distance(\vgen, \set) \coloneqq \inf_{\tilde{\vgen} \in \set} \lnorm{\vgen - \tilde{\vgen}}$. Moreover, the \emph{diameter} and \emph{radius} of $\set$ are denoted by $\diam(\set) \coloneqq \sup_{\vgen_1, \vgen_2 \in \set} \lnorm{\vgen_1 - \vgen_2}$ and $\radius(\set) \coloneqq \sup_{\tilde{\vgen} \in \set} \lnorm{\tilde{\vgen}}$, respectively.

The \emph{$L^q$-norm} of a real-valued random variable $g$ is given by $\norm{g}_{L^q} \coloneqq (\mean[\abs{g}^q])^{1/q}$. We call $g$ \emph{sub-Gaussian} if $\normsubg{g} \coloneqq \inf\big\{v > 0 \suchthat \mean[\exp(\abs{g}^2 / v^2)] \leq 2 \big\} < \infty$; see \citet[Chap.~2~\&~3]{ver18} for more details on sub-Gaussian random variables and their properties.
Finally, we write $\g \distributed \Normdistr{\vnull}{\I{d}}$ if $\g$ is a \emph{standard Gaussian random vector} in $\R^d$.

The \emph{ceiling} and \emph{floor function} of $z \in \R$ are denoted by $\ceil{z}$ and $\floor{z}$, respectively.

\section{Main Separation Result and Mutual Complexity}
\label{sec:main}

This section presents our most general solution to Problem~\ref{prob:rnn-sep}, containing all guarantees from the introduction (Theorem~\ref{thm:memorization},~\ref{thm:eucl_balls}, and~\ref{thm:uniform_covering}) as special cases.
To formulate the main result, Theorem~\ref{thm:main}, we require two important definitions formalizing the idea of \emph{mutual complexity} between two sets.
The first one can be seen as a refinement of the uniform covering introduced in \eqref{eq:intro:covering_number}:	
\begin{definition}[Mutual covering]\label{def:mutual_covering}
	Let $\set^+, \set^- \subset \R^d$ and $\biasp > 0$. 
	
	We call $\setcen^+ \coloneqq \{ \cent_1^+, \dots, \cent_{N^+}^+ \} \subset \R^d$ and $\setcen^- \coloneqq \{ \cent_1^-, \dots, \cent_{N^-}^- \} \subset \R^d$ a \emph{$\biasp$-mutual covering} for $\set^+$ and $\set^-$ if there exist $\rad_1^+, \dots, \rad_{N^+}^+ \geq 0$ and $\rad_1^-, \dots, \rad_{N^-}^- \geq 0$ such that
	\begin{thmproperties}
		\item\label{def:mutual_covering:cover}
		the sets $\set_j^+ \coloneqq \set^+ \intersec\ball[2][d](\cent_j^+,\rad_j^+)$ for $j \in [N^+]$, and $\set_l^- \coloneqq \set^- \intersec\ball[2][d](\cent_l^-,\rad_l^-)$ for $l \in [N^-]$, cover $\set^+$ and $\set^-$, respectively;
		\item\label{def:mutual_covering:radius}
		$\rad_j^+ \leq \biasp^{-1} \distsq{\cent_j^+}{\setcen^-}$ for all $j \in [N^+]$, and $\rad_l^- \leq \biasp^{-1}\distsq{\cent_l^-}{\setcen^+}$ for all $l \in [N^-]$.
	\end{thmproperties}
	Furthermore, the sets $\set_1^+, \dots, \set_{N^+}^+ \subset \set^+$ and $\set_1^-, \dots, \set_{N^-}^- \subset \set^-$ are referred to as the \emph{components} of the covering.
\end{definition}
\begin{figure}[t]
	\centering
	\includegraphics[width=0.6\linewidth]{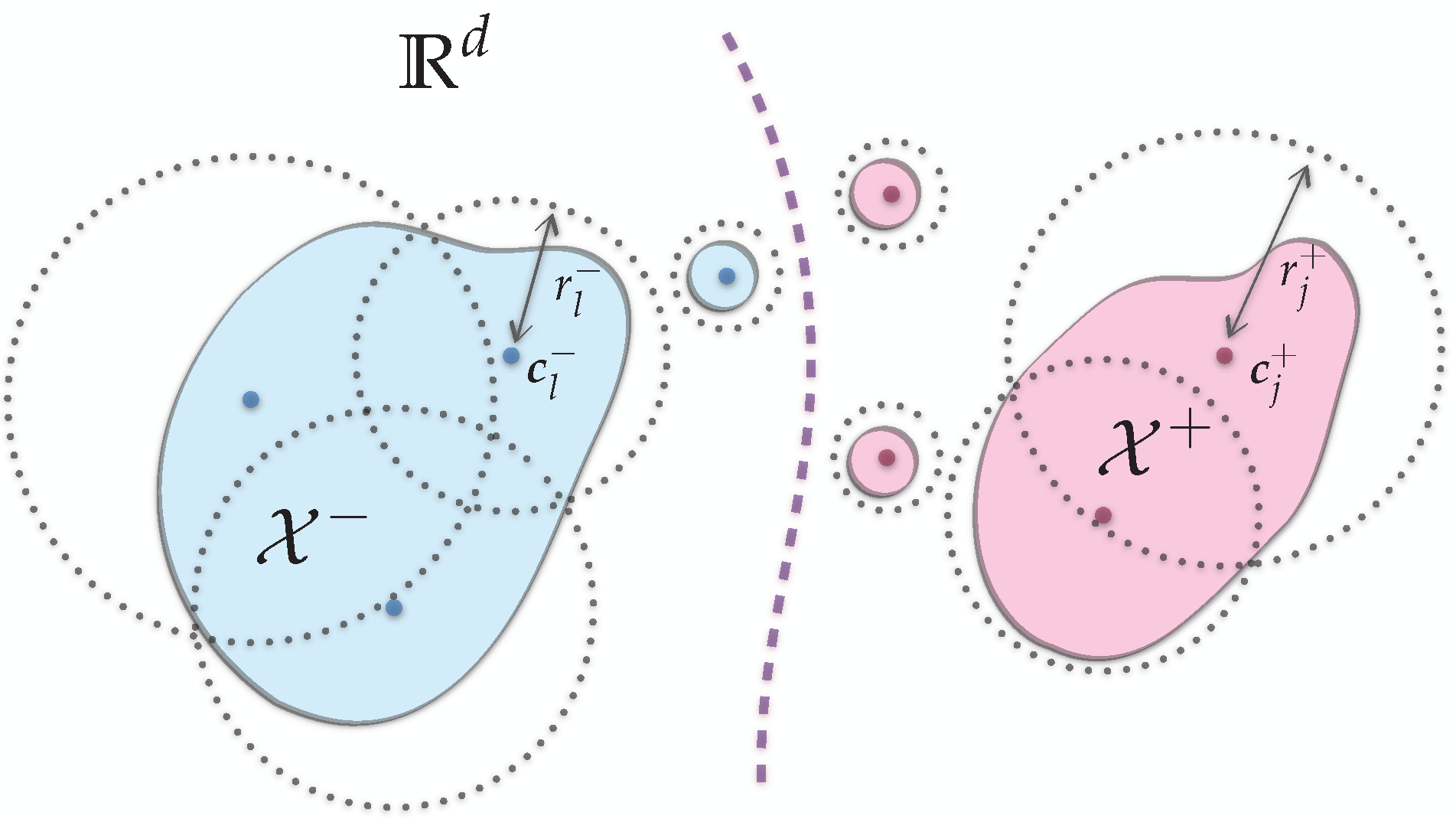}
	\caption{\textbf{Mutual covering.} This figure illustrates the geometric idea underlying Definition~\ref{def:mutual_covering}: those parts of $\set^-$ and $\set^+$ further away from the decision boundary may be covered by larger, and therefore fewer, Euclidean balls.}
	\label{fig:main:mutual_covering}
\end{figure}
Although the notion of $\biasp$-mutual covering involves some technicalities, it is conceptually simple:
We allow $\set^-$ and $\set^+$ to be covered by Euclidean balls of any radius, as long as the balls corresponding to different classes do not get too close in the sense of condition~\ref{def:mutual_covering:radius}.
This constraint is also consistent with the setting of Theorem~\ref{thm:uniform_covering}, which is obtained by choosing $\rad_j^+ = \rad_l^- = c\mdist^2/\biasp$.
However, Definition~\ref{def:mutual_covering} is much more flexible and accounts for the mutual arrangement of the classes.
For example, those parts of $\set^-$ that are far away from the decision boundary may be covered by a few large balls, while smaller radii are only needed for data closer to $\set^+$; see Figure~\ref{fig:main:mutual_covering} for an illustration.
In general, this strategy leads to more efficient coverings and motivates the following geometric complexity parameters:
\begin{definition}[Mutual complexity]\label{def:mutual_complexity}
	Let $\set^+, \set^- \subset \R^d$ and $\mdist, \biasp > 0$.
	
	We say that~$\set^+$ and~$\set^-$ have \emph{$(\Rad, \mdist, \biasp)$-mutual complexity} $(N^+, N^-, \mwloc^+, \mwloc^-)$ if there exists a $\biasp$-mutual covering $\setcen^+ =\nobreak \{ \cent_1^+, \dots, \cent_{N^+}^+ \}$ and $\setcen^- = \{ \cent_1^-, \dots, \cent_{N^-}^- \}$ for $\set^+$ and~$\set^-$ such that
	\begin{thmproperties}
		\item
		$\displaystyle\max_{j \in [N^+]} \meanwidth{\set_j^+} \leq \mwloc^+$ and $\displaystyle\max_{l \in [N^-]} \meanwidth{\set_l^-} \leq \mwloc^-$;
		\item
		$\setcen^+, \setcen^-\subset \Rad\ball[2][d]$ are $\mdist$-separated.
	\end{thmproperties}
\end{definition}
It is useful to keep in mind that the covering numbers~$N^+$ and~$N^-$ reflect the \emph{global} size of~$\set^+$ and~$\set^-$, respectively, while $\mwloc^+$ and $\mwloc^-$ should be viewed as \emph{local} complexity measures (cf.~\eqref{eq:intro:mean_width_localized}).
In contrast, the parameters $(\Rad, \mdist, \biasp)$ are not instance-specific and concern the general problem setting.

We are now ready to state the main result of this work:
\begin{theorem}[Main result]\label{thm:main}
	There exist absolute constants $c, C, C' > 0$ such that the following holds.
	
	For $\Rad \geq 1$, let $\set^-, \set^+\subset \Rad\ball[2][d]$ be $\mdist$-separated and let $\biasp \geq e\mdist$ be such that $\biasp \gtrsim\nobreak  \Rad\sqrt{\log(\biasp/\mdist)}$.
	Furthermore, let~$\set^+$ and~$\set^-$ have $(\Rad, \mdist, C'\biasp)$-mutual complexity $(N^+, N^-,\allowbreak \mwloc^+, \mwloc^-)$.
	We assume that $\NNlayer \colon \R^d\to \R^{n}$ and $\NNlayersec \colon \R^{n}\to \R^{\layersec{n}}$ are two (independent) random ReLU-layers with maximal biases $\biasp, \layersec{\biasp} \geq 0$, respectively, such that
	\begin{align}
		n &\gtrsim  \effdim{\cone{\set^--\set^-}\intersec\S^{d-1}} + \effdim{\cone{\set^+-\set^+}\intersec\S^{d-1}}, \\*
		n &\gtrsim \big(\tfrac{\biasp}{\mdist}\big)^{8} \cdot \Big( \biasp^{-2}\big(\mwloc^- + \mwloc^+\big)^2 + \log (2N^-N^+/\probsuccess) \Big)
		\label{eq:main:first_layer}
	\end{align}
	and
	\begin{align}
		\layersec{\biasp}&\gtrsim  \big(\tfrac{\biasp}{\mdist}\big)^{4} \cdot \big(\mwloc^-+\meanwidth{\set^+}+\biasp\big), \\ 
		\layersec{n} &\gtrsim \big(\tfrac{\layersec{\biasp}}{\mwloc^-+\meanwidth{\set^+}+\biasp}\big)\cdot \exp\Big(C \cdot \big(\mwloc^-+\meanwidth{\set^+}+\biasp\big)^2 \cdot \biasp^6\cdot \mdist^{-8} \cdot \log(\biasp/\mdist)\Big)\cdot\log(N^-/\probsuccess). \\
		\label{eq:main:second_layer}
	\end{align}
	Then, given the two-layer random NN $\NN \colon \R^d\to \R^{\layersec{n}}, \ \x \mapsto \NNlayersec(\NNlayer(\x))$, with probability at least $1-\probsuccess$, the sets $\NN(\set^-), \NN(\set^+) \subset \layersec{\biasp} \ball[2][\layersec{n}]$ are linearly separable with margin
	\begin{equation}\label{eq:main:margin}
		c\cdot\tfrac{(\mwloc^-+\meanwidth{\set^+}+\biasp)^2}{\layersec{\biasp}}\cdot 
		\exp\Big(-C \cdot \big(\mwloc^-+\meanwidth{\set^+}+\biasp\big)^2\cdot  \biasp^6\cdot \mdist^{-8}\cdot  \log(\biasp/\mdist)\Big).
	\end{equation}
\end{theorem}
Despite a strong resemblance to Theorem~\ref{thm:uniform_covering}, the above result entails several important improvements.
First, the exponential terms in \eqref{eq:main:second_layer} and \eqref{eq:main:margin} only depend on the localized mean width~$\mwloc^-$, but not the covering number~$N^-$.
Hence, the global size of $\set^-$ does not have any (negative) impact here.
The situation is different for $\set^+$, whose complexity is still captured by $\meanwidth{\set^+}$.
In fact, the following adaption of \eqref{eq:intro:mean_width_localized} clarifies the role of~$N^+$:
\begin{equation}\label{eq:main:mean_width_localized}
	\meanwidth{\set^+} \lesssim  \mwloc^+ + \Rad \sqrt{\log N^+}.
\end{equation}
The aforementioned asymmetry in Theorem~\ref{thm:main} becomes especially useful when the set~$\set^+$ is relatively ``small'' compared to $\set^-$.
A prototypical example in this regard is a low-complexity set ($= \set^+$), say a small Euclidean ball, which is surrounded by a hypersphere ($= \set^-$); see Figure~\ref{fig:main:hypersphere} for an illustration.
\begin{figure}[t]
	\centering
	\includegraphics[width=0.35\linewidth]{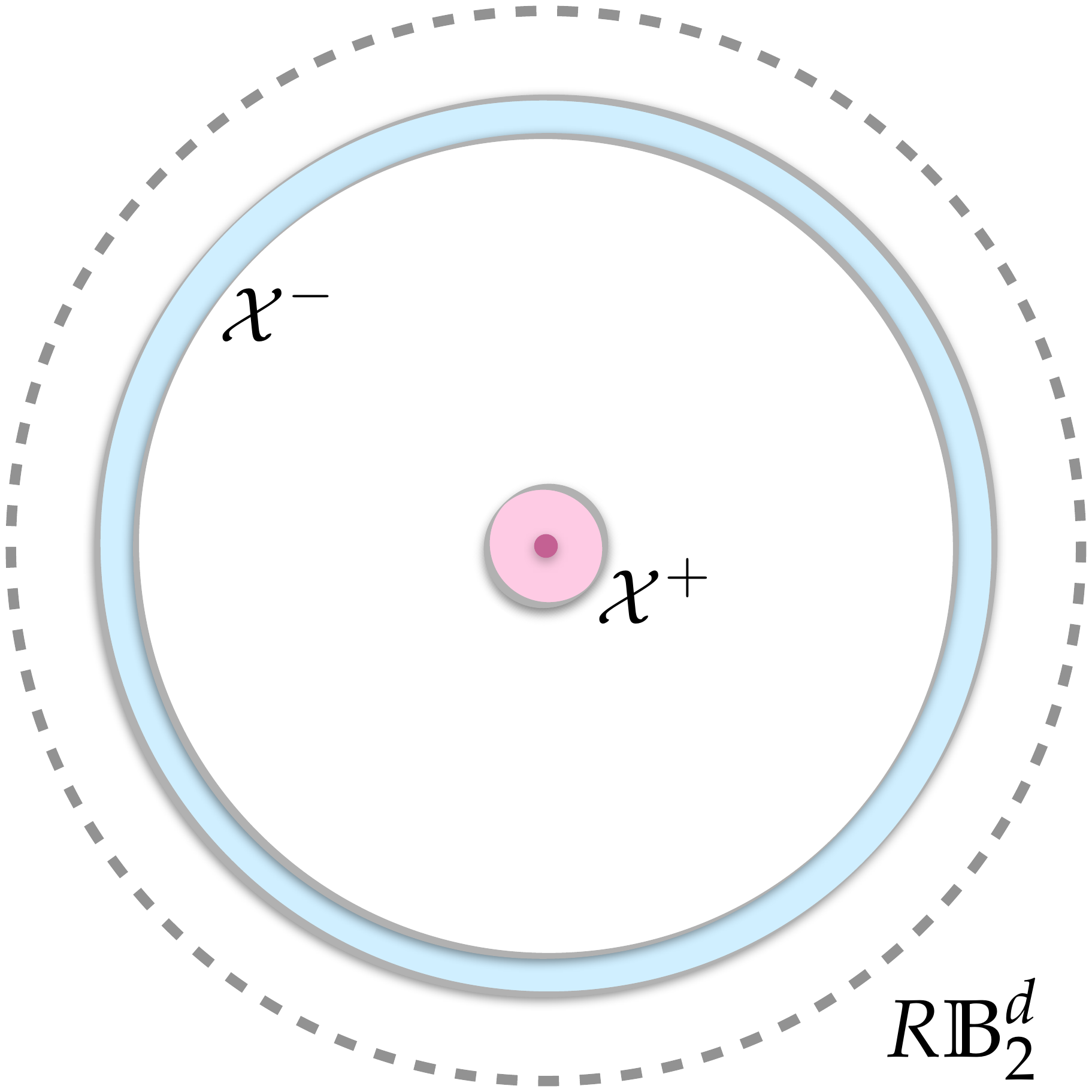}
	\caption{\textbf{An example of ``asymmetric complexity'' in the context of Theorem~\ref{thm:main}.} The set $\set^-$ corresponds to a thin hypersphere (say $\S^{d-1}$) around the origin, surrounding a small concentric ball $\set^+=\rad\ball[2][d]$. If $\rad\lesssim 1 / \sqrt{d}$, then $\mwloc^-$ and $\meanwidth{\set^+}$ are of constant order. Crucially, the covering number $N^-$, which scales exponentially in $d$, has no detrimental effect on the condition \eqref{eq:main:second_layer}.}
	\label{fig:main:hypersphere}
\end{figure}

Another distinctive feature of Theorem~\ref{thm:main} is the usage of \emph{mutual} complexity.
To understand its merits over the uniform covering considered in Theorem~\ref{thm:uniform_covering}, it is worth revisiting the scenario of Figure~\ref{fig:main:mutual_covering}: while the largest portion of the two classes is away from the (\mbox{$\mdist$-separated}) decision boundary, only a few \enquote{outliers} are close to it. 
Thus, a uniform covering would preset a very small radius (at the order $\asympfaster{\biasp^{-1}\mdist^2}$), which is appropriate for the outlier part but inefficient for the remaining bulk; this would lead to unnecessarily large covering numbers and thereby to poor complexity bounds (cf.~\eqref{eq:main:mean_width_localized}).
In contrast, our mutual covering strategy is flexible enough to handle such data configurations.
Therefore, Theorem~\ref{thm:main} indeed presents an \emph{instance-specific} solution to Problem~\ref{prob:rnn-sep}, including a variant of outlier robustness.

\begin{remark}[Possible extensions]\label{rmk:main:extend}
	For the sake of clarity, we have omitted some possible variations and generalizations of Theorem~\ref{thm:main}, which are however relatively straightforward to implement:
	\begin{rmklist}
		\item\label{rmk:main:extend:symmetry}
		\emph{Symmetry.} As discussed above, the asymmetric way of measuring complexity in Theorem~\ref{thm:main} can be advantageous in certain situations. On the other hand, it is obvious that the roles of~$\set^-$ and~$\set^+$ are interchangeable.
		Hence, Theorem~\ref{thm:main} could be \enquote{symmetrized} in this respect by a simple union bound argument.
		\item\label{rmk:main:extend:activation}
		\emph{Non-linear activation.} The considered random network design is tailored to the ReLU-activation (see Definition~\ref{def:intro:random-layer}). Nevertheless, our proof strategy is applicable to other functions as well, e.g., the thresholding activation.
		This might involve a slight adaption of Definition~\ref{def:mutual_covering}\ref{def:mutual_covering:radius} and lead to a different scaling of $\mdist$ and $\biasp$ in Theorem~\ref{thm:main}, but the qualitative statement remains valid.
		\item\label{rmk:main:extend:multiclass}
		\emph{Multiclass classification.} While we have focused on binary labels for the sake of simplicity, our main results can readily be extended to categorical data using a simple one-vs-rest strategy. Assume we are given data from $K$ different classes, say $\set^1, \set^2, \dots, \set^K \subset \R^d$. Then, for any $l\in [K]$, Theorem~\ref{thm:main} implies that a sufficiently large random NN separates $\set^+ := \set^l$ and $\set^- := \bigunion_{k \in [K]\setminus \{l\}} \set^k$ with high probability.
			Taking the union bound over these $K$ events, we conclude that with high probability a single, large random NN $\NN$ makes each individual set $\NN(\set^1), \dots, \NN(\set^K)$ linearly separable from the remaining ones. Analogously to the binary case, this separation property allows us to train a standard one-vs-rest SVM classifier on the transformed data sets.
	\end{rmklist}
\end{remark}

\section{Separation by Random Hyperplanes}
\label{sec:separation}

The goal of this section is to prove 
Theorem~\ref{thm:hyperplane_sep:general}, which is
our main result on the separation of two sets by a random hyperplane. Before outlining the main steps of our proof, let us define the relevant notions of separability. 

\begin{definition}\label{def:separable_refined}
	Let $\set^+, \set^- \subset \R^d$.
	\begin{deflist}
		\item\label{def:separable_refined:hypp}
		Let $\vec{v}\in \R^d\backslash\{\vec{0}\}, \sepbias\in \R$ and $t \geq 0$. A hyperplane $\hypp[\sepbias]{\vec{v}}$ \emph{$t$-separates $\set^-$ from $\set^+$} if
		\begin{align}
			\sp{\vec{v}}{\x^-}+\sepbias &\leq - \hyppsep \qquad \text{for all $\x^-\in \set^-$,}\\*
			\sp{\vec{v}}{\x^+}+\sepbias &> +\hyppsep \qquad \text{for all $\x^+\in \set^+$.}
		\end{align}
		If $t = 0$, we simply say that $\hypp[\sepbias]{\vec{v}}$ \emph{separates $\set^-$ from $\set^+$}.
		\item\label{def:separable_refined:narrpw}
		Let $\narrow\in [0,1]$ and $\narrowmdist>0$. We say that $\set^+$ and $\set^-$ are \emph{$(\narrow,\narrowmdist)$-linearly separable} if $\set^+$ and $\set^-$ are $\narrowmdist$-separated (see Problem~\ref{prob:rnn-sep}) and there exists $\sepdir\in \S^{d-1}$ such that 
		\begin{equation}\label{eq:separable_refined:eps_narrow}
			\sp{\sepdir}{\x^+-\x^-}\geq (1-\narrow)\lnorm{\x^+-\x^-} \qquad \text{for all $\x^+\in \set^+$ and $\x^-\in \set^-$.}
		\end{equation}
	\end{deflist}
\end{definition}

Recall from \eqref{eq:rnn-sep:hyper-sep} that $\set^+$ and $\set^-$ are called \emph{linearly separable with margin $t$} if they are \mbox{$t$-separated} by a hyperplane $\hypp[\sepbias]{\vec{v}}$ with $\lnorm{\vec{v}}=1$.
In comparison, $(\narrow,\narrowmdist)$-linearly separability is a strictly stronger condition (see also Proposition~\ref{prop:separable_relations}\ref{prop:separable_relations:eps_narrow_linear} below).
Intuitively, it captures how much a separating hyperplane can be perturbed, such that it still separates the sets $\set^+$ and $\nobreak\set^-$; geometrically, the parameter $\narrow$ controls the narrowness of $\cone{\set^+-\nobreak\set^-}$.

\paragraph{Proof sketch for Theorem~\ref{thm:hyperplane_sep:general}.} By a rescaling argument, we can assume that $\Rad=1$. For a $k\in \N$ specified below, we represent the standard Gaussian vector $\g\in \R^d$ by $\g=\G^T\normgauss'$, where $\G\in \R^{k\times d}$ is a standard Gaussian matrix,  $\normgauss'\in \S^{k-1}$ is uniformly distributed, and $\G, \normgauss'$ are independent. We then observe that, for any $\rho\geq 0$, the hyperplane $\hypp[\sepbias]{\g}$ \mbox{$\rho$-separates} $\set^-$ from $\set^+$ if and only if the hyperplane $\hypp[\sepbias]{\sqrt{k}\normgauss'}$ $\rho$-separates $\tfrac{1}{\sqrt{k}}\G\set^-$ from~$\tfrac{1}{\sqrt{k}}\G\set^+$. Therefore, one can prove Theorem~\ref{thm:hyperplane_sep:general} by first showing that for $k$ large enough, the sets $\tfrac{1}{\sqrt{k}}\G\set^-$ and $\tfrac{1}{\sqrt{k}}\G\set^+$ are again linearly separable with constant probability and second, showing that conditioned on this event the hyperplane  $\hypp[\sepbias]{\sqrt{k}\normgauss'}$ $\rho$-separates $\tfrac{1}{\sqrt{k}}\G\set^-$ and $\tfrac{1}{\sqrt{k}}\G\set^+$ with probability $p$, where $\rho$ and $p$ are specified in Theorem~\ref{thm:hyperplane_sep:general}. Specifically, the main technical steps are:
\begin{enumerate}    
	\item\label{separation:proof_sketch:step3} 
	to show that if 
	\begin{equation}
		k\gtrsim \narrowmdist^{-2}(1-\narrow)^{-2}(\effdim{\set^+-\set^-}+1),
	\end{equation}
	then the linear transformation $\tfrac{1}{\sqrt{k}}\G$ maps the $(\narrow, \narrowmdist)$-linearly separable sets $\set^-$ and $\set^+$ to $(\tfrac{1+\narrow}{2},\tfrac{\narrowmdist}{2})$-linearly separable sets $\tfrac{1}{\sqrt{k}}\G\set^-$ and $\tfrac{1}{\sqrt{k}}\G\set^+$
	with probability at least~$\tfrac{1}{2}$.
	\item\label{separation:proof_sketch:step4} 
	to derive a general separation result for two $(\narrow, \narrowmdist)$-linearly separable sets by a random hyperplane $\hypp[\sepbias]{\normgauss}$, where $\normgauss$ is uniformly distributed on Euclidean sphere (see Theorem~\ref{thm:hyperplane_sep:uniform} and Corollary~\ref{coro:hyperplane_sep:uniform}). 
\end{enumerate}

Let us now give the proof in full detail. We start with a simple proposition that relates our notions of separability.
\begin{proposition}\label{prop:separable_relations}
	Let $\set^+, \set^- \subset \R^d$. The following relationships hold:
	\begin{thmproperties}
		\item\label{prop:separable_relations:linear_mdist}
		If $\set^+$ and $\set^-$ are linearly separable with margin $\marg$, then they are $2\marg$-separated.
		\item\label{prop:separable_relations:hypp_linear}
		If a hyperplane $\hypp[\sepbias]{\sepdir}$ $t$-separates $\set^-$ from $\set^+$, then $\set^+$ and $\set^-$ are linearly separable with margin $t / \lnorm{\sepdir}$.
		\item\label{prop:separable_relations:eps_narrow_linear}
		If $\set^+$ and $\set^-$ are $(\narrow,\narrowmdist)$-linearly separable, then they are linearly separable with margin $\tfrac{(1-\narrow)\narrowmdist}{2}$.
		\item\label{prop:separable_relations:linear_eps_narrow}
		If $\set^+$ and $\set^-$ are linearly separable with margin $\marg$ and $\diam(\set^+- \set^-) \leq \Rad$, then they are $(\tfrac{\Rad-2\marg}{\Rad},2\marg)$-linearly separable.
	\end{thmproperties}
\end{proposition}
\renewenvironment{proof}{\par\noindent{\bf Proof\ }}{}
\begin{proof}
	To show \ref{prop:separable_relations:linear_mdist}, observe that by assumption there exist $\sepdir\in \S^{d-1}$ and $\sepbias\in \R$ such that 
	\begin{align}
		\sp{\sepdir}{\x^-}+\sepbias &\leq - \marg \qquad \text{for all $\x^-\in \set^-$}, \\
		\sp{\sepdir}{\x^+}+\sepbias & \geq +\marg \qquad \text{for all $\x^+\in \set^+$.}
	\end{align}
	It follows that $\sp{\sepdir}{\x^+-\x^-}\geq 2\marg$ for all $\x^+\in \set^+, \x^-\in \set^-$. By the Cauchy-Schwarz inequality, $\sp{\sepdir}{\x^+-\x^-}\leq \lnorm{\x^+-\x^-}$, which shows the claim. For \ref{prop:separable_relations:hypp_linear}, it suffices to note that if $\hypp[\sepbias]{\sepdir}$ $t$-separates $\set^-$ from $\set^+$, then $\hypp[\tfrac{\sepbias}{\lnorm{\sepdir}}]{\tfrac{\sepdir}{\lnorm{\sepdir}}}$ $\tfrac{t}{\lnorm{\sepdir}}$-separates $\set^-$ from~$\set^+$.
	Let us next show 
	\ref{prop:separable_relations:eps_narrow_linear}. If $\set^+$ and $\set^-$ are $(\narrow,\narrowmdist)$-linearly separable, then there exists $\sepdir\in\S^{d-1}$ such that  
	\begin{equation}
		\sp{\sepdir}{\x^+-\x^-}\geq (1-\narrow)\narrowmdist \qquad \text{for all $\x^+\in \set^+$ and $\x^-\in \set^-$.}
	\end{equation}
	Set $\sepbias=-\tfrac{(1-\narrow)\narrowmdist}{2}+\inf_{\x^-\in \set^-}\sp{\sepdir}{-\x^-}$. Fix $\vec{x}'\in \set^+$. Since 
	\begin{equation}
		\inf_{\x^-\in \set^-}\sp{\sepdir}{-\x^-} 
		= -\sp{\sepdir}{\vec{x}'} + \inf_{\x^-\in \set^-}\sp{\sepdir}{\vec{x}'-\x^-} \geq 
		-\sp{\sepdir}{\vec{x}'} + (1-\narrow)\narrowmdist, 
	\end{equation}
	we see that $\sepbias\in\R$. Further, for any $\x^-\in\set^-$, 
	\begin{align}
		\sp{\sepdir}{\x^-}+\sepbias &= \sp{\sepdir}{\x^-} -\tfrac{(1-\narrow)\narrowmdist}{2}+\inf_{\x^-\in \set^-}\sp{\sepdir}{-\x^-}\\
		&= \sp{\sepdir}{\x^-} -\sup_{\x^-\in \set^-}\sp{\sepdir}{\x^-} -\tfrac{(1-\narrow)\narrowmdist}{2} \leq -\tfrac{(1-\narrow)\narrowmdist}{2}
	\end{align}
	and for any
	$\x^+\in\set^+$, 
	\begin{align}
		\sp{\sepdir}{\x^+}+\sepbias &= \sp{\sepdir}{\x^+} -\tfrac{(1-\narrow)\narrowmdist}{2}+\inf_{\x^-\in \set^-}\sp{\sepdir}{-\x^-}\\
		&= \inf_{\x^-\in \set^-}\sp{\sepdir}{\x^+-\x^-} -\tfrac{(1-\narrow)\narrowmdist}{2} \geq \tfrac{(1-\narrow)\narrowmdist}{2}. 
	\end{align}
	Since $\lnorm{\sepdir}=1$, it follows that the hyperplane $\hypp[\sepbias]{\sepdir}$ linearly separates $\set^-$ and $\set^+$ with margin $\tfrac{(1-\narrow)\narrowmdist}{2}$. Finally, let us show \ref{prop:separable_relations:linear_eps_narrow}. 
	By \ref{prop:separable_relations:linear_mdist} we know that $\set^+$ and $\set^-$ are $2\marg$-separated. Let $\sepdir\in \S^{d-1}$ and $\sepbias\in \R$ be such that \begin{align}
		\sp{\sepdir}{\x^-}+\sepbias &\leq - \marg \qquad \text{for all $\x^-\in \set^-$}, \\
		\sp{\sepdir}{\x^+}+\sepbias & \geq +\marg \qquad \text{for all $\x^+\in \set^+$.}
	\end{align}
	It follows that $\sp{\sepdir}{\x^+-\x^-}\geq 2\marg$ for all $\x^+\in \set^+, \x^-\in \set^-$. Since $\diam(\set^+- \set^-) \leq \Rad$, we also have $2\marg\geq \tfrac{2\marg}{\Rad}\lnorm{\x^+-\x^-}$ for all $\x^+\in \set^+, \x^-\in \set^-$. Together this yields
	\begin{equation}
		\sp{\sepdir}{\x^+-\x^-}\geq \big(1-\tfrac{\Rad-2\marg}{\Rad}\big)\lnorm{\x^+-\x^-}\qquad \text{for all $\x^+\in \set^+$ and $\x^-\in \set^-$.} \tag*{\qed}
	\end{equation}
\end{proof}
\renewenvironment{proof}{\par\noindent{\bf Proof\ }}{\hfill\BlackBox\\[2mm]}

The next result gives a lower bound for the probability that a random hyperplane $\hypp[\sepbias]{\normgauss}$ separates two 
$(\narrow,\narrowmdist)$-linearly separable sets $\set^+, \set^-\subset \Rad \ball[2][d]$, where $\sepbias\in [-\biasp,\biasp]$ is uniformly distributed for 
$\biasp\geq \Rad$ and $\normgauss$ is uniformly distributed on $\S^{d-1}$. As detailed in our above proof sketch, this result (more precisely, Corollary~\ref{coro:hyperplane_sep:uniform}) forms a crucial ingredient of our proof of Theorem~\ref{thm:hyperplane_sep:general}.

\begin{theorem}\label{thm:hyperplane_sep:uniform}
	There exist absolute constants $c, C>0$ such that the following holds. 
	
	For $\narrow\in [0,1]$ and $\narrowmdist>0$, consider $(\narrow,\narrowmdist)$-linearly separable sets $\set^+, \set^-\subset \Rad \ball[2][d]$.
	Let $\normgauss\in \S^{d-1}$ and $\sepbias\in [-\biasp,\biasp]$ be both uniformly distributed. 
	If $\biasp\geq \Rad$, then
	with probability at least 
	\begin{equation}
		c\tfrac{\narrowmdist}{\biasp}(1-\narrow)(\sqrt{\narrow}+\tfrac{1}{\sqrt{d}})\exp({-C}\narrow d\log(2(1-\narrow)^{-1})),
	\end{equation}
	the hyperplane $\hypp[\sepbias]{\normgauss}$
	$\hyppsep$-separates $\set^-$ from $\set^+$ with 
	$\hyppsep=c\narrowmdist(1-\narrow)(\sqrt{\narrow}+\tfrac{1}{\sqrt{d}})$. 
\end{theorem}

\begin{corollary}\label{coro:hyperplane_sep:uniform}
	There exist absolute constants $c, C>0$ such that the following holds. 
	
	For $\narrow\in [0,1]$ and $\narrowmdist>0$, consider $(\narrow,\narrowmdist)$-linearly separable sets $\set^+, \set^-\subset \Rad \ball[2][d]$. Let $\nu>0$. 
	Let $\normgauss\in \S^{d-1}$ and $\sepbias\in [-\biasp,\biasp]$ be both uniformly distributed. 
	If $\biasp\geq \nu\Rad$, then
	with probability at least 
	\begin{equation}
		c\tfrac{\nu\narrowmdist}{\biasp}(1-\narrow)(\sqrt{\narrow}+\tfrac{1}{\sqrt{d}})\exp({-C}\narrow d\log(2(1-\narrow)^{-1})),
	\end{equation}
	the hyperplane $\hypp[\sepbias]{\nu\normgauss}$
	$\hyppsep$-separates $\set^-$ from $\set^+$ with 
	$\hyppsep=c\narrowmdist(1-\narrow)(\sqrt{\narrow}+\tfrac{1}{\sqrt{d}})\nu$.
\end{corollary}
\begin{proof}
	For $\hyppsep\geq 0$ the hyperplane $\hypp[\sepbias]{\nu\normgauss}$ $\hyppsep\nu$-separates $\set^-$ from $\set^+$ if and only if the hyperplane $\hypp[\tfrac{1}{\nu}\sepbias]{\normgauss}$ $\hyppsep$-separates $\set^-$ from $\set^+$. The random variable $\sepbias':=\tfrac{1}{\nu}\sepbias$ is uniformly distributed on $[-\biasp',\biasp']$ for $\biasp'=\tfrac{\biasp}{\nu}$. The result follows from Theorem~\ref{thm:hyperplane_sep:uniform} for $\hyppsep=c\narrowmdist(1-\nobreak\narrow)(\sqrt{\narrow}+\tfrac{1}{\sqrt{d}})$.
\end{proof}

For the proof of  Theorem~\ref{thm:hyperplane_sep:uniform}, we need the following standard result (e.g., see~\citealp[Sec.~7.2]{blm13}), which precisely describes the surface measure of a spherical cap.

\begin{lemma}\label{lem:uniform_distr}
	Let $\normgauss\in \S^{d-1}$ be uniformly distributed. Let $\mdist\in (0,1]$ and $d\geq 2\mdist^{-2}$. For any $\sepdir\in\S^{d-1}$, we have that
	\begin{equation}
		\tfrac{1}{6\mdist\sqrt{d}}(1-\mdist^2)^{\tfrac{d-1}{2}}\leq \prob\left(\sp{\normgauss}{\sepdir}\geq \mdist\right)\leq \tfrac{1}{2\mdist\sqrt{d}}(1-\mdist^2)^{\tfrac{d-1}{2}}.    
	\end{equation}
	If additionally $\mdist\leq \tfrac{1}{\sqrt{2}}$, then 
	\begin{equation}
		\prob\left(\sp{\normgauss}{\sepdir}\geq \mdist\right)\geq \tfrac{1}{2}\exp(-2\mdist^2d).
	\end{equation}
\end{lemma}

\begin{proof}[Theorem~\ref{thm:hyperplane_sep:uniform}]
	For $\theta\geq 0$ define the event\footnote{Formally, all events should be understood in the ordinary sense of probability theory, i.e., measurable subsets of some appropriate sample space. Note that the underlying probability space is not explicitly mentioned here. Our analysis does not require any treatment of measure theoretic issues, and we simply assume that the probability space is rich enough to model all random quantities and processes that we are interested in.}
	\begin{equation}
		\event_{\normgauss}(\theta)\coloneqq\Big\{\inf_{\x^+\in \set^+,\, \x^-\in \set^-} \sp{\normgauss}{\x^+-\x^-}\geq \theta\Big\}. 
	\end{equation}
	For any $s\geq 0$, we have that
	\begin{align}
		&\prob(\hypp[\sepbias]{\normgauss}\; s\text{-separates } \set^- \text{ from } \set^+)\\*
		&\geq \prob(\{\hypp[\sepbias]{\normgauss}\; s\text{-separates } \set^- \text{ from } \set^+\} \intersec \event_{\normgauss}(\theta))\\*
		&= p_1(s,\theta) \cdot p_2(\theta), 
	\end{align}
	where 
	\begin{equation}
		p_1(s,\theta) \coloneqq \prob( \hypp[\sepbias]{\normgauss}\; s\text{-separates } \set^- \text{ from } \set^+ \suchthat \event_{\normgauss}(\theta)) \quad \text{and} \quad 
		p_2(\theta) \coloneqq \prob(\event_{\normgauss}(\theta)). 
	\end{equation}

	Next, we bound both factors $p_1(s,\theta)$ and $p_2(\theta)$ from below.
	
	\paragraph{Lower bound for $p_1(s,\theta)$.}
	Let us show that if $\biasp\geq \Rad$, then for $s\leq \tfrac{\theta}{2}$, 
	\begin{equation}\label{eq:hyperplane_sep:uniform:p1_both_cases}
		\prob( \hypp[\sepbias]{\normgauss}\; s\text{-separates } \set^- \text{ from } \set^+ \suchthat \event_{\normgauss}(\theta))\geq \tfrac{\theta-2s}{2\biasp}.   
	\end{equation}
	If $\biasp\geq \Rad$, then $a\coloneqq\sup_{\x^-\in \set^-}\sp{\normgauss}{\x^-}\in [-\biasp,\biasp]$ and $b\coloneqq\inf_{\x^+\in \set^+}\sp{\normgauss}{\x^+}\in [-\biasp,\biasp]$.
	Further, on the event $\event_{\normgauss}(\theta)$ 
	it holds $b-a\geq \theta$. Let $s\leq \tfrac{\theta}{2}$. 
	If $-\sepbias\in [a+s, b-s]$, then 
	\begin{align}
		\sp{\normgauss}{\x^-}+\sepbias&\leq -s \qquad \text{for all } \x^-\in \set^-,\\
		\sp{\normgauss}{\x^+}+\sepbias &\geq +s \qquad \text{for all } \x^+\in \set^+. 
	\end{align}
	Since $\prob_\sepbias(-\sepbias \in[a+s, b-s])\geq\tfrac{\theta-2s}{2\biasp}$, this shows 
	\begin{equation}\label{eq:hyperplane_sep:p1_case1}
		\prob_{\sepbias}(\hypp[\sepbias]{\normgauss}\; s\text{-separates } \set^- \text{ from } \set^+ \suchthat \event_{\normgauss}(\theta))\geq \tfrac{\theta-2s}{2\biasp}.  
	\end{equation}
	For $s=\tfrac{\theta}{4}$ we obtain 
	\begin{equation}\label{eq:hyperplane_sep:uniform:prob}
		\prob(\hypp[\sepbias]{\normgauss}\; \tfrac{\theta}{4}\text{-separates } \set^- \text{ from } \set^+)\geq \tfrac{\theta}{4\biasp}\cdot p_2(\theta).
	\end{equation}
	\paragraph{Lower bound for $p_2(\theta)$.}
	Set $\set\coloneqq\cone{\set^+-\set^-}\intersec \S^{d-1}$ and for $\mdist\in [0,1]$ define
	\begin{equation}
		\alpha_\narrow(\mdist)\coloneqq\max\{\mdist-\sqrt{2\narrow}, 1-\narrow-\sqrt{2}\sqrt{1-\mdist}\}.
	\end{equation}
	Since $\set^+$ and $\set^-$ are 
	$(\narrow,\narrowmdist)$-linearly separable there exists $\sepdir\in \S^{d-1}$ such that 
	\begin{equation}\label{eq:thm:hyperplane_sep:uniform:cone}
		\sp{\sepdir}{\vgen}\geq 1-\narrow \qquad \text{for all $\vgen \in \set$.}
	\end{equation}
	Let us show that for any $\mdist\in [0,1]$ with $\alpha_\narrow(\mdist)\geq 0$, 
	\begin{equation}\label{eq:hyperplane_sep:uniform:decoupling}
		\{\sp{\normgauss}{\sepdir}\geq \mdist\}\subset\event_{\normgauss}(\alpha_\narrow(\mdist)\narrowmdist). 
	\end{equation}
	First observe that for any $\mdist\in [0,1]$,  
	\begin{equation}\label{eq:hyperplane_sep:uniform:normalized_Gaussian_normalized_differences}
		\{\sp{\normgauss}{\sepdir}\geq \mdist\}\subset\Big\{\inf_{\vgen\in \set}\sp{\normgauss}{\vgen}\geq \alpha_\narrow(\mdist)\Big\}.  \end{equation}
	Indeed, by \eqref{eq:thm:hyperplane_sep:uniform:cone}, $\lnorm{\vgen-\sepdir}\leq \sqrt{2\narrow}$ for every $\vgen\in \set$. Since $\normgauss\in \S^{d-1}$ it follows 
	\begin{equation}\label{eq:hyperplane_sep:uniform:alpha1}
		\sp{\normgauss}{\vgen}\geq \sp{\normgauss}{\sepdir}-\lnorm{\vgen-\sepdir}\geq \sp{\normgauss}{\sepdir}-\sqrt{2\narrow}\geq \mdist-\sqrt{2\narrow} \quad \text{for all } \vgen\in \set,    
	\end{equation}
	if $\sp{\normgauss}{\sepdir}\geq \mdist$. Moreover, $\sp{\normgauss}{\sepdir}\geq \mdist \iff \lnorm{\normgauss-\sepdir}\leq \sqrt{2}\sqrt{1-\mdist}$. Therefore, 
	if $\sp{\normgauss}{\sepdir}\geq \mdist$ then for every $\vgen\in \set$,
	\begin{equation}\label{eq:hyperplane_sep:uniform:alpha2}
		\sp{\normgauss}{\vgen}\geq \sp{\sepdir}{\vgen}-\lnorm{\normgauss-\sepdir}
		\geq 1-\narrow-\sqrt{2}\sqrt{1-\mdist}.
	\end{equation}
	Inequalities \eqref{eq:hyperplane_sep:uniform:alpha1} and \eqref{eq:hyperplane_sep:uniform:alpha2} imply \eqref{eq:hyperplane_sep:uniform:normalized_Gaussian_normalized_differences}. 
	If $\inf_{\vgen\in \set}\sp{\normgauss}{\vgen}\geq \alpha_\narrow(\mdist)$ and $\alpha_\narrow(\mdist)\geq 0$, the following holds for every $\x^+\in \set^+, \x^-\in \set^-$:
	\begin{equation}
		\sp{\normgauss}{\x^+-\x^-}\geq \alpha_\narrow(\mdist)\lnorm{\x^+-\x^-}\geq \alpha_\narrow(\mdist)\narrowmdist, 
	\end{equation}
	where for the second inequality we used that $\alpha_\narrow(\mdist)\geq 0$ and that 
	$\set^+$ and $\set^-$ are $\narrowmdist$-separated.
	In combination with \eqref{eq:hyperplane_sep:uniform:normalized_Gaussian_normalized_differences} this shows \eqref{eq:hyperplane_sep:uniform:decoupling}.
	From \eqref{eq:hyperplane_sep:uniform:decoupling} it follows that for every $\mdist\in [0,1]$ with $\alpha_\narrow(\mdist)\geq 0$ and every 
	$\theta\leq \alpha_\narrow(\mdist)\narrowmdist$, 
	\begin{equation}\label{eq:hyperplane_sep:uniform:p2}
		p_2(\theta)\geq \prob(\sp{\normgauss}{\sepdir}\geq \mdist).
	\end{equation}
	In order to bound the probability on the right hand side from below, we distinguish the cases $\narrow\leq \tfrac{1}{32}$ and $\narrow>\tfrac{1}{32}$.
	
	\paragraph{Case $\narrow\leq \tfrac{1}{32}$.}
	If $\mdist\in [0,1]$ satisfies $\mdist\geq \sqrt{8\narrow}$, then $\mdist-\sqrt{2\narrow}\geq \tfrac{\mdist}{2}$, which implies $\alpha_\narrow(\mdist)\geq \tfrac{\mdist}{2}\geq 0$. Therefore, by \eqref{eq:hyperplane_sep:uniform:p2} the following holds: For all $\mdist\in [\sqrt{8\narrow}, 1]$, 
	\begin{equation}
		p_2(\tfrac{\mdist}{2}\narrowmdist)\geq \prob(\sp{\normgauss}{\sepdir}\geq \mdist). 
	\end{equation}
	By Lemma~\ref{lem:uniform_distr} we obtain that for all $\mdist\in (\sqrt{8\narrow}, \tfrac{1}{\sqrt{2}}]$ with $d\geq 2\mdist^{-2}$, 
	\begin{equation}
		p_2(\tfrac{\mdist}{2}\narrowmdist)\geq \tfrac{1}{2}\exp(-2\mdist^2 d). 
	\end{equation}
	Applying \eqref{eq:hyperplane_sep:uniform:prob} for $\theta=\tfrac{\mdist}{2}\narrowmdist$, we obtain that
	for all $\mdist\in (\sqrt{8\narrow}, \tfrac{1}{\sqrt{2}}]$ with $d\geq 2\mdist^{-2}$,
	\begin{equation}
		\prob(\hypp[\sepbias]{\normgauss}\; \tfrac{\mdist\narrowmdist}{8}\text{-separates } \set^- \text{ from } \set^+)\geq \tfrac{\mdist\narrowmdist}{8\biasp}\cdot\tfrac{1}{2}\exp(-2\mdist^2 d).  
	\end{equation}
	The choice $\mdist=\sqrt{8\narrow}+\sqrt{\tfrac{2}{d}}$, which satisfies $\mdist\in (\sqrt{8\narrow}, \tfrac{1}{\sqrt{2}}]$ and 
	$d\geq 2\mdist^{-2}$, yields
	\begin{equation}
		\prob(\hypp[\sepbias]{\normgauss}\; \tfrac{\narrowmdist}{8}(\sqrt{8\narrow}+\sqrt{\tfrac{2}{d}})\text{-separates } \set^- \text{ from } \set^+)\geq c\tfrac{\narrowmdist}{\biasp}(\sqrt{\narrow}+\tfrac{1}{\sqrt{d}})\cdot\exp({-C}\narrow d)
	\end{equation}
	for absolute constants $c, C>0$. 
	The result in the case $\varepsilon\leq \tfrac{1}{32}$ follows by observing that 
	$1-\narrow\sim 1$.
	
	\paragraph{Case $\narrow> \tfrac{1}{32}$.}
	Set $\mdist' = 1-\tfrac{1}{4}(1-\narrow)^2$. Then $\mdist' \in [0,1]$ and $\alpha_\narrow(\mdist')\geq 1-\narrow-\sqrt{2}\sqrt{1-\mdist'}=(1-\tfrac{1}{\sqrt{2}})(1-\narrow)\geq \tfrac{1}{4}(1-\narrow)\geq 0$.
	Therefore, by \eqref{eq:hyperplane_sep:uniform:p2} 
	in combination with Lemma~\ref{lem:uniform_distr} we obtain
	\begin{equation}
		p_2(\tfrac{1}{4}(1-\narrow)\narrowmdist)\geq \tfrac{1}{6\mdist'\sqrt{d}}(1-\mdist'^2)^{\tfrac{d-1}{2}}. 
	\end{equation}
	Observe that 
	\begin{align}
		\tfrac{1}{6\mdist'\sqrt{d}}(1-\mdist'^2)^{\tfrac{d-1}{2}}
		&\geq \tfrac{1}{6\sqrt{d}}\exp(-d\log((1-\mdist'^2)^{-1}))\\ 
		&\geq \tfrac{1}{6\sqrt{d}}\exp(-d\log((1-\mdist')^{-1}))\\ 
		&=\tfrac{1}{6\sqrt{d}}\exp(-2d\log(2(1-\narrow)^{-1}))\\ 
		&\geq \exp(-3d\log(2(1-\narrow)^{-1})). \label{eq:hyperplane_sep:uniform:lowerboundd} 
	\end{align}
	Applying \eqref{eq:hyperplane_sep:uniform:prob} for $\theta=\tfrac{1}{4}(1-\narrow)\narrowmdist$, we obtain
	\begin{equation}
		\prob(\hypp[\sepbias]{\normgauss}\; \tfrac{(1-\narrow)\narrowmdist}{16}\text{-separates } \set^- \text{ from } \set^+)\geq \tfrac{(1-\narrow)\narrowmdist}{16\biasp}\cdot \exp(-3d\log(2(1-\narrow)^{-1})).
	\end{equation}
	The result in the case $\varepsilon> \tfrac{1}{32}$ follows by observing that 
	$\sqrt{\narrow}+\tfrac{1}{\sqrt{d}}\sim 1$ and $\narrow\sim 1$. 
	This completes the proof.
\end{proof}

\subsection{Proof of Theorem~\ref{thm:hyperplane_sep:general}}
\label{sec:proof:hyperplane_sep}

For the proof of Theorem~\ref{thm:hyperplane_sep:general}, we need one final ingredient, namely Lemma~\ref{lem:dimred} below. It gives a sufficient condition under which a set contained in a spherical cone is again contained in a spherical cone after a linear transformation.
Using this lemma, we will show that a Gaussian matrix with enough rows maps $(\narrow, \narrowmdist)$-linear separable sets to 
$(\tfrac{1+\narrow}{2}, \tfrac{\narrowmdist}{2})$-linearly separable sets with constant probability (see step \ref{separation:proof_sketch:step3} of the proof sketch at the beginning of Section~\ref{sec:separation}). 

\begin{lemma}\label{lem:dimred}
	Let $\sepdir\in \S^{d-1}$, $t\in [0,1]$, and $\set\subset \Rad \ball[2][d]$ satisfy
	\begin{equation}
		\sp{\sepdir}{\x}\geq t\lnorm{\x}\qquad \text{ for all } \x\in \set.
	\end{equation}
	For $\dimredprecision\in (0,\tfrac{1}{2}]$, $\alpha, \beta\in [0,\Rad]$ and $\dimredmatrix\in \R^{k\times d}$ assume that the following holds:
	\begin{thmproperties}
		\item $(1-\dimredprecision)\lnorm{\sepdir}^2\leq \lnorm{\dimredmatrix\sepdir}^2\leq(1+\dimredprecision)\lnorm{\sepdir}^2$, 
		\item $\lnorm{\sepdir-\x}-\alpha\leq  \lnorm{\dimredmatrix\sepdir -\dimredmatrix\x}\leq \lnorm{\sepdir-\x}+ \alpha$ for all $\x\in \set$,
		\item $\lnorm{\x}-\beta\leq \lnorm{\dimredmatrix\x}\leq \lnorm{\x}+\beta$ for all $\x\in \set$.
	\end{thmproperties}
	Then, we have that $\dimredmatrix\set\subset2\Rad\ball[2][k]$ and \begin{equation}\label{eq:dimred}
		\sp{\tfrac{\dimredmatrix\sepdir}{\lnorm{\dimredmatrix\sepdir}}}{\dimredmatrix\x}\geq
		\tfrac{t}{\sqrt{1+\dimredprecision}}\lnorm{\dimredmatrix\x}- \tfrac{\dimredprecision}{\sqrt{2}}-\sqrt{2}(\tfrac{3}{2}\Rad+t)\beta-\tfrac{3}{\sqrt{2}}(1+\Rad)\alpha
		\qquad \text{ for all } \x\in \set.
	\end{equation}
\end{lemma}
\begin{proof} For any $\x\in \set$,
	\begin{equation}
		\big|\lnorm{\dimredmatrix\x}^2-\lnorm{\x}^2\big|= \big|\lnorm{\dimredmatrix\x}-\lnorm{\x}\big|\cdot \big|\lnorm{\dimredmatrix\x}+\lnorm{\x}\big|\leq \beta(2\lnorm{\x}+\beta)\leq 3\Rad\beta.    
	\end{equation}
	Analogously,
	\begin{equation}
		\big|\lnorm{\dimredmatrix\sepdir-\dimredmatrix\x}^2-\lnorm{\sepdir-\x}^2\big|\leq 3(1+\Rad)\alpha.
	\end{equation}
	Therefore,
	\begin{align}
		\sp{\dimredmatrix\sepdir}{\dimredmatrix\x}&=\tfrac{1}{2}\lnorm{\dimredmatrix\sepdir}^2+\tfrac{1}{2}\lnorm{\dimredmatrix\x}^2-\tfrac{1}{2}\lnorm{\dimredmatrix(\sepdir-\x)}^2\\
		&\geq \tfrac{1}{2}(1-\dimredprecision)\lnorm{\sepdir}^2+\tfrac{1}{2}\lnorm{\x}^2-\tfrac{3}{2}\Rad\beta-\tfrac{1}{2}\lnorm{\sepdir-\x}^2-\tfrac{3}{2}(1+\Rad)\alpha\\
		&=\sp{\sepdir}{\x} - \tfrac{\dimredprecision}{2}\lnorm{\sepdir}^2-\tfrac{3}{2}\Rad\beta-\tfrac{3}{2}(1+\Rad)\alpha\\
		&\geq t\lnorm{\x} - \tfrac{\dimredprecision}{2}-\tfrac{3}{2}\Rad\beta-\tfrac{3}{2}(1+\Rad)\alpha\\
		&\geq t\lnorm{\dimredmatrix\x}  - \tfrac{\dimredprecision}{2}-(\tfrac{3}{2}\Rad+t)\beta-\tfrac{3}{2}(1+\Rad)\alpha. 
	\end{align}  
	Using $\tfrac{1}{\sqrt{2}}\leq\lnorm{\dimredmatrix\sepdir}\leq \sqrt{1+\dimredprecision}$, we obtain \eqref{eq:dimred}. 
\end{proof}

We are now ready to prove our main result on the separation of two sets by a random hyperplane:

\begin{proof}[Theorem~\ref{thm:hyperplane_sep:general}]
	First note that it suffices to prove the result for $\Rad=1$. Indeed, the general result then follows by a rescaling argument. 
	
	Let $k\in \N$. Let $\normgauss'\in \S^{k-1}$ be uniformly distributed, $\G\in\R^{k\times d}$ a standard Gaussian matrix and $\tau\in [-\biasp,\biasp]$ be uniformly distributed. Let all random variables be independent. 
	We define the random vector $\g:=\G^T\normgauss'\in \R^d$ and observe that it is standard Gaussian. Indeed, one may write $\normgauss' = \vec{Q}\vec{e}_1$ where $\vec{Q} \in \R^{k\times k}$ is a uniform random orthogonal matrix and $\vec{e}_1$ the first unit vector in $\R^k$. Due to the rotational invariance of standard Gaussian matrices, we have that $\G^T \vec{Q} \distributed \G^T$, and therefore, $\g =\G^T\normgauss' = \G^T \vec{Q}\vec{e}_1 \distributed \G^T \vec{e}_1 \distributed \Normdistr{\vnull}{\I{d}}$.
	
	Set $\dimredmatrix=\tfrac{1}{\sqrt{k}}\G$. 
	For $\narrow'\in [0,1]$, and $\narrowmdist'\geq 0$, we define the event 
	\begin{equation}
		\event_{\narrow',\narrowmdist'}:=\{\dimredmatrix\set^+,  \dimredmatrix\set^-\subset 5\ball[2][k] \text{ are } (\narrow',\narrowmdist')\text{-linearly separable}\}. 
	\end{equation}
	Using that $\g=\dimredmatrix^T\sqrt{k}\normgauss'$, we obtain for any $\rho\geq 0$,
	\begin{align}\label{eq:thm:hyperplane_sep:general:main}
		&\prob(\hypp[\sepbias]{\g}\; \rho\text{-separates } \set^- \text{ from } \set^+)\\
		&= \prob(\hypp[\sepbias]{\sqrt{k}\normgauss'}\; \rho\text{-separates } \dimredmatrix\set^- \text{ from } \dimredmatrix\set^+)\\
		&\geq \prob(\{\hypp[\sepbias]{\sqrt{k}\normgauss'}\; \rho\text{-separates } \dimredmatrix\set^- \text{ from } \dimredmatrix\set^+\}\intersec \event_{\narrow',\narrowmdist'})\\
		&=\prob(\hypp[\sepbias]{\sqrt{k}\normgauss'}\; \rho\text{-separates } \dimredmatrix\set^- \text{ from } \dimredmatrix\set^+\suchthat \event_{\narrow',\narrowmdist'})\cdot \prob(\event_{\narrow',\narrowmdist'}). 
	\end{align} 
	Since $\set^+$ and $\set^-$ are $(\narrow,\narrowmdist)$-linearly separable, 
	there exists a vector $\sepdir\in \S^{d-1}$ such that  
	\begin{equation}
		\sp{\sepdir}{\x^+-\x^-}\geq (1-\narrow)\lnorm{\x^+-\x^-}\qquad \text{ for all } \x^+\in \set^+, \x^-\in \set^-.
	\end{equation}
	For $\dimredprecision\in (0,\tfrac{1}{2}]$, $\alpha\in [0,2], \beta\in[0,\tfrac{\narrowmdist}{2}]$ define the event $\eventalt_{\kappa,\alpha,\beta}$ where:
	\begin{enumerate}
		\item $(1-\dimredprecision)\lnorm{\sepdir}^2\leq \lnorm{\dimredmatrix\sepdir}^2\leq(1+\dimredprecision)\lnorm{\sepdir}^2$, 
		\item For all $\x^+\in \set^+, \x^-\in \set^-$, 
		\begin{equation}
			\lnorm{\sepdir-(\x^+-\x^-)} -\alpha\leq\lnorm{\dimredmatrix\sepdir -\dimredmatrix(\x^+-\x^-)}\leq \lnorm{\sepdir-(\x^+-\x^-)} +\alpha, 
		\end{equation}
		\item For all $\x^+\in \set^+, \x^-\in \set^-$,
		\begin{equation}
			\lnorm{\x^+-\x^-}-\beta\leq \lnorm{\dimredmatrix(\x^+-\x^-)}\leq \lnorm{\x^+-\x^-}+\beta,
		\end{equation}
		\item There exists $\x^-\in \set^-$ such that $\lnorm{\dimredmatrix\x^-}^2\leq (1+\dimredprecision)\lnorm{\x^-}^2$,
		\item There exists $\x^+\in \set^+$ such that $\lnorm{\dimredmatrix\x^+}^2\leq (1+\dimredprecision)\lnorm{\x^+}^2$.
	\end{enumerate}
	On the event $\eventalt_{\kappa,\alpha,\beta}$ we clearly have $\dimredmatrix\set^-,\dimredmatrix\set^+\subset5\ball[2][k]$ and 
	$\lnorm{\dimredmatrix\x^+-\dimredmatrix\x^-}\geq \tfrac{\narrowmdist}{2}$ for all $\x^+\in\nobreak \set^+$, $\x^-\in\nobreak \set^-$.
	Further, 
	by applying Lemma~\ref{lem:dimred} for $\set=\set^+-\set^-\subset 2 \ball[2][d]$ and $t=1-\narrow$, we obtain that on the event $\eventalt_{\kappa,\alpha,\beta}$,  
	\begin{equation}
		\sp{\tfrac{\dimredmatrix\sepdir}{\lnorm{\dimredmatrix\sepdir}}}{\dimredmatrix(\x^+-\x^-)}\geq \sqrt{\tfrac{2}{3}}(1-\narrow)\lnorm{\dimredmatrix(\x^+-\x^-)}- \tfrac{\dimredprecision}{\sqrt{2}}-\sqrt{2}(3+(1-\narrow))\beta-\tfrac{3}{\sqrt{2}}(1+2)\alpha
	\end{equation}
	for all $\x^+\in \set^+, \x^-\in \set^-$. 
	For $\kappa\lesssim \narrowmdist(1-\narrow)$ and $\alpha,\beta\lesssim \narrowmdist(1-\narrow)$, we obtain 
	\begin{equation}
		\sp{\tfrac{\dimredmatrix\sepdir}{\lnorm{\dimredmatrix\sepdir}}}{\dimredmatrix(\x^+-\x^-)}\geq \tfrac{1}{2}(1-\narrow)\lnorm{\dimredmatrix(\x^+-\x^-)}\qquad \text{for all } \x^+\in \set^+, \x^-\in \set^-. 
	\end{equation}
	Hence,
	if $\kappa\lesssim \narrowmdist(1-\narrow)$ and $\alpha,\beta\lesssim \narrowmdist(1-\narrow)$, then 
	$\eventalt_{\kappa,\alpha,\beta}\subset \event_{\narrow',\narrowmdist'}
	$
	for $\narrow' = \tfrac{1+\narrow}{2}$ and $\narrowmdist'=\tfrac{\narrowmdist}{2}$.
	For this choice of $\narrow'$ and $\narrowmdist'$, 
	Corollary~\ref{coro:hyperplane_sep:uniform} implies that if  
	$\biasp\geq 5\sqrt{k}$,
	then
	\begin{align}
		&\prob_{\normgauss',\tau}(\hypp[\sepbias]{\sqrt{k}\normgauss'}\; c\narrowmdist(1-\narrow)\sqrt{k}\text{-separates } \dimredmatrix\set^- \text{ from } \dimredmatrix\set^+\suchthat \event_{\narrow',\narrowmdist'})\\
		&\geq 
		c\tfrac{\narrowmdist(1-\narrow)}{\biasp}\sqrt{k}\exp({-C}k\log(4(1-\narrow)^{-1})). 
	\end{align}
	Therefore, applying \eqref{eq:thm:hyperplane_sep:general:main} with $\rho=c\narrowmdist(1-\narrow)\sqrt{k}$ and $\narrow' = \tfrac{1+\narrow}{2}$,  $\narrowmdist'=\tfrac{\narrowmdist}{2}$, we obtain that if $\biasp\geq 5\sqrt{k}$, then 
	\begin{align}\label{eq:thm:hyperplane_sep:general:main2}
		&\prob(\hypp[\sepbias]{\g}\; c\narrowmdist(1-\narrow)\sqrt{k}\text{-separates } \set^- \text{ from } \set^+)\\
		&\geq c\tfrac{\narrowmdist(1-\narrow)}{\biasp}\sqrt{k}\exp({-C}k\log(4(1-\narrow)^{-1}))\cdot \prob(\event_{\narrow',\narrowmdist'})\\
		&\geq c\tfrac{\narrowmdist(1-\narrow)}{\biasp}\sqrt{k}\exp({-C}k\log(4(1-\narrow)^{-1}))\cdot \prob(\eventalt_{\kappa,\alpha,\beta}),
	\end{align} 
	where the second inequality holds for $\kappa\lesssim \narrowmdist(1-\narrow)$ and $\alpha,\beta\lesssim \narrowmdist(1-\narrow)$. 
	
	Let $\mathcal{T}\subset \R^d$ be a set. 
	By matrix deviation inequality for Gaussian matrices (e.g., see~\citealp[Sec.~9.1]{ver18}), if
	\begin{equation}
		k\gtrsim \theta^{-2}(\effdim{\mathcal{T}}+\log(2/\probsuccess)\radius^2(\mathcal{T})),
	\end{equation}
	then with probability at least $1-\probsuccess$, 
	\begin{equation}
		\sup_{\x\in\mathcal{T}}\big|\lnorm{\dimredmatrix\x}-\lnorm{\x}\big|\leq \theta. 
	\end{equation}
	Hence, a union bound implies that if 
	\begin{equation}
		k\gtrsim \kappa^{-2}\log(2/\probsuccess),\quad
		k\gtrsim (\alpha^{-2}+\beta^{-2})(\effdim{\set^+-\set^-}+\log(2/\probsuccess)),
	\end{equation}
	then $\eventalt_{\kappa,\alpha,\beta}$ occurs with probability at least $1-\probsuccess$. In particular, if 
	\begin{align}
		k&\gtrsim \narrowmdist^{-2}(1-\narrow)^{-2}(\effdim{\set^+-\set^-}+1), 
	\end{align}
	then $\eventalt_{\kappa,\alpha,\beta}$ with $\kappa\sim \narrowmdist(1-\narrow)$ and $\alpha,\beta\sim \narrowmdist(1-\narrow)$ occurs with probability at least $\tfrac{1}{2}$. 
	Combining this result with \eqref{eq:thm:hyperplane_sep:general:main2}, we obtain that if 
	\begin{align}
		k&\gtrsim \narrowmdist^{-2}(1-\narrow)^{-2}(\effdim{\set^+-\set^-}+1), \quad \biasp\geq 5\sqrt{k}, 
	\end{align}
	then 
	\begin{equation}
		\prob(\hypp[\sepbias]{\g}\; c\narrowmdist(1-\narrow)\sqrt{k}\text{-separates } \set^- \text{ from } \set^+)
		\geq c\tfrac{\narrowmdist(1-\narrow)}{2\biasp}\sqrt{k}\exp({-C}k\log(4(1-\narrow)^{-1})). 
	\end{equation}
	Let $\marg\coloneqq \narrowmdist(1-\narrow)$. 
	Setting $k=\marg^{-2}\hyppsep^2$ completes the proof.
\end{proof}

\subsection{Separation of Two Points}

The following result concerns the separation of two arbitrary points by a random hyperplane, and does not follow directly from Theorem~\ref{thm:hyperplane_sep:general}.
It can be shown in a more elementary way, leading to a stronger statement.
\begin{theorem}\label{thm:prob_single_hyperplane_separates_two_points}
	There exist absolute constants $c,C>0$ such that the following holds.
	
	Let $\x^-, \x^+ \in \Rad\ball[2][d]$.
	Let $\g \in \R^d$ denote a standard Gaussian random vector and let $\sepbias\in\nobreak [-\biasp,\biasp]$ be uniformly distributed.
	If $\biasp\geq C\Rad$, then
	with probability at least $c\lnorm{\x^+ -\nobreak \x^-} / \biasp$, the hyperplane $\hypp[\sepbias]{\g}$ $\lnorm{\x^+ -\nobreak \x^-}$-separates $\x^-$ from $\x^+$.
\end{theorem}
\begin{proof}
	Since $\tfrac{\x^+-\x^-}{\lnorm{\x^+-\x^-}}\in \S^{d-1}$, the random variable $\sp{\g}{\tfrac{\x^+-\x^-}{\lnorm{\x^+-\x^-}}}$ is standard Gaussian. Therefore, 
	\begin{equation}
		\prob(\sp{\g}{\x^+-\x^-}\geq 4\lnorm{\x^+-\x^-})\geq c
	\end{equation}
	for an absolute constant $c>0$. Further, we have the inequalities 
	\begin{equation}
		\prob(\sp{\g}{\x^+}\leq \biasp)\geq 1-\exp(-\biasp^2/2\lnorm{\x^+}^2)\geq 1-\exp(-\biasp^2/2\Rad^2)
	\end{equation}
	and 
	\begin{equation}
		\prob(\sp{\g}{\x^-}\geq -\biasp)\geq 1-\exp(-\biasp^2/2\Rad^2).
	\end{equation}
	Define the event 
	\begin{equation}
		\event\coloneqq \{\sp{\g}{\x^+-\x^-}\geq 4\lnorm{\x^+-\x^-}\}\intersec \{\sp{\g}{\x^+}\leq \biasp\}\intersec \{\sp{\g}{\x^-}\geq -\biasp\}.
	\end{equation}
	By the above, $\prob(\event)\geq c-2\exp(-\biasp^2/2\Rad^2)$. Therefore, 
	if $\biasp\geq C\Rad$ for $C>0$ an absolute constant that is chosen large enough, then 
	$\prob(\event)\geq\tfrac{c}{2}$. 
	Let us show that
	\begin{equation}
		\prob_\sepbias(\hypp[\sepbias]{\g}\; \lnorm{\x^+ - \x^-}\text{-separates } \x^- \text{ from } \x^+\; | \; \event)\geq \tfrac{\lnorm{\x^+-\x^-}}{\biasp}.
	\end{equation}
	Indeed, on the event $\event$ it holds $\sp{\g}{\x^+}, \sp{\g}{\x^-}\in [-\biasp, \biasp]$ and $\sp{\g}{\x^+-\x^-}\geq 4\lnorm{\x^+-\x^-}$. In particular, the interval $\mathcal{I}\coloneqq [\sp{\g}{\x^-}+\lnorm{\x^+-\x^-}, \sp{\g}{\x^+}-\lnorm{\x^+-\x^-}]$ belongs to $[-\biasp, \biasp]$ with $|\mathcal{I}|\geq 2\lnorm{\x^+-\x^-}$. If $-\sepbias\in \mathcal{I}$, then $\hypp[\sepbias]{\g}$ $\lnorm{\x^+ - \x^-}$-separates $\x^-$ from~$\x^+$. Therefore, 
	\begin{align}
		\prob_\sepbias(\hypp[\sepbias]{\g}\; \lnorm{\x^+ - \x^-}\text{-separates } \x^- \text{ from } \x^+\; | \; \event)\geq
		\prob_\sepbias(-\sepbias\in \mathcal{I}\; | \; \event) = \tfrac{|\mathcal{I}|}{2\biasp}\geq \tfrac{\lnorm{\x^+-\x^-}}{\biasp}. 
	\end{align}
	The result now follows from
	\begin{align}
		&\prob(\hypp[\sepbias]{\g}\; \lnorm{\x^+ - \x^-}\text{-separates } \x^- \text{ from } \x^+)\\
		&\geq \prob(\{\hypp[\sepbias]{\g}\; \lnorm{\x^+ - \x^-}\text{-separates } \x^- \text{ from } \x^+\}\intersec \event)\\
		&=
		\prob(\hypp[\sepbias]{\g}\; \lnorm{\x^+ - \x^-}\text{-separates } \x^- \text{ from } \x^+\; | \; \event)\cdot \prob(\event)\\
		&\geq \tfrac{\lnorm{\x^+-\x^-}}{\biasp}\cdot \tfrac{c}{2}.
	\end{align}
\end{proof}

\section{Distance Preservation}
\label{sec:distance}


The following theorem describes how the Euclidean geometry of two sets $\set^-, \set^+$ is transformed by applying a random ReLU-layer. It shows that with high probability, Euclidean distances are approximately preserved provided that both the layer is wide and the bias parameter $\biasp$ is large enough. 

\begin{theorem}\label{thm:distance_preservation_ReLU}
	There exist absolute constants $C, C', c >0$ such that the following holds. 
	
	Let $\set^-, \set^+\subset \Rad \ball[2][d]$ and let $\NNlayer \colon \R^d\to \R^n$ be a random ReLU-layer with maximal bias $\biasp \geq 0$.
	If $\biasp\geq C \Rad \sqrt{\log(\biasp^2/\varepsilon)}$ for $0 < \varepsilon \leq \biasp^2/ e$, and
	\begin{equation}\label{eq:distance_preservation_ReLU:cond_n}  
		n\geq C' \varepsilon^{-2}\biasp^2\big(\effdim{\set^+}+\effdim{\set^-}+ \probsuccessexp ^2\biasp^2\big),
	\end{equation}
	then with probability at least $1-2\exp(-c \probsuccessexp^2)$, the following three events occur: 
	\begin{thmproperties}
		\item For all $\x^+\in \set^+, \x^-\in \set^-$, we have
		\begin{equation}\label{eq:distance_preservation_ReLU:distance}
			\abs[\Big]{\lnorm{\NNlayer(\x^+)-\NNlayer(\x^-)}^2 - \lnorm{\x^+ - \x^-}^2\Big(1  - \sqrt{\tfrac{2}{\pi}} \tfrac{\lnorm{\x^+ -\x^-}}{3\biasp}  \Big)} \leq \varepsilon.
		\end{equation}
		\item For all $\x\in \set^-\union \set^+$, we have
		\begin{equation}\label{eq:distance_preservation_ReLU:norm}
			\abs[\Big]{\lnorm{\NNlayer(\x)}^2 - \big(\lnorm{\x}^2 + \tfrac{\biasp^2}{3}\big)}
			\leq \varepsilon.
		\end{equation}	
		\item
		For all $\x^+\in \set^+, \x^-\in \set^-$, we have
		\begin{equation}\label{eq:distance_preservation_ReLU:ip}
			\abs[\Big]{\sp{\NNlayer(\x^+)}{\NNlayer(\x^-)} - \Big(\sp{\x^+}{\x^-} + 
				\tfrac{\biasp^2}{3}  + \sqrt{\tfrac{2}{\pi}}\tfrac{1}{6}\tfrac{\lnorm{\x^+ -\x^-}^3}{\biasp}\Big) } \leq \varepsilon.
		\end{equation}
	\end{thmproperties}
\end{theorem}

The following two results are straightforward corollaries of \eqref{eq:distance_preservation_ReLU:norm} and \eqref{eq:distance_preservation_ReLU:distance} in Theorem~\ref{thm:distance_preservation_ReLU}, respectively.

\begin{corollary}\label{coro:distance_preservation_ReLU:norm}
	There exist absolute constants $C, C'>0$ such that the following holds. 
	
	Let $\set\subset \Rad \ball[2][d]$ and let $\NNlayer \colon \R^d\to \R^n$ be a random ReLU-layer with maximal bias $\biasp \geq 0$.
	For any $\probsuccess\in (0,1)$, if $\biasp\geq C \Rad$ and
	\begin{equation}\label{eq:distance_preservation_ReLU:norm:cond}  
		n\geq C' \big(\biasp^{-2}\effdim{\set}+ \log(e/\probsuccess)\big),
	\end{equation}
	then $\NNlayer(\set)\subset \biasp \ball[2][n]$ with probability at least $1-\probsuccess$.
\end{corollary}

\begin{corollary}\label{coro:distance_preservation_ReLU:mdist}
	There exist absolute constants $C, C' >0$ such that the following holds. 
	
	Let $\set^-, \set^+\subset \Rad \ball[2][d]$ be $\mdist$-separated sets and let $\NNlayer \colon \R^d\to \R^n$ be a random ReLU-layer with maximal bias $\biasp \geq 0$.
	For any $\probsuccess\in (0,1)$, if $\biasp\geq C \Rad \sqrt{\log(\biasp/\mdist)}$, $\biasp/\mdist\geq e$, and
	\begin{equation}\label{eq:distance_preservation_ReLU:mdist:cond}  
		n\geq C' \mdist^{-4}\biasp^2\big(\effdim{\set^+}+\effdim{\set^-}+ \log(e/\probsuccess)\biasp^2\big),
	\end{equation}
	then $\NNlayer(\set^-)$ and $\NNlayer(\set^+)$ are $\tfrac{\mdist}{2}$-separated 
	with probability at least $1-\probsuccess$.
\end{corollary}

Let us outline the main steps of the proof of Theorem~\ref{thm:distance_preservation_ReLU}. Note that it suffices to show \eqref{eq:distance_preservation_ReLU:ip}. Indeed, \eqref{eq:distance_preservation_ReLU:norm} trivially follows from \eqref{eq:distance_preservation_ReLU:ip}. Further, \eqref{eq:distance_preservation_ReLU:distance} follows from \eqref{eq:distance_preservation_ReLU:ip} and \eqref{eq:distance_preservation_ReLU:norm} by polarization. To show \eqref{eq:distance_preservation_ReLU:ip}, we proceed in two steps:
\begin{enumerate}
	\item 
	Compute the expected value of $\sp{\NNlayer(\x^+)}{\NNlayer(\x^-)}$ for two arbitrary points $\x^+, \x^-\in \Rad \ball[2][d]$ (see Proposition~\ref{prop:expectation_inner_product}). 
	\item 
	Show uniform concentration of $\sp{\NNlayer(\x^+)}{\NNlayer(\x^-)}$ around its expected value using a concentration result for empirical product processes due to \citet[Thm.~1.13]{men16}.
\end{enumerate}

\begin{proposition}\label{prop:expectation_inner_product}
	There exists an absolute constant $C>0$ such that the following holds. 
	
	Let $\NNlayer \colon \R^d\to \R^n$ be a random ReLU-layer with maximal bias $\biasp \geq 0$.
	If $\biasp\geq C \Rad \sqrt{\log(\biasp^2/\varepsilon)}$ and $\biasp^2/\varepsilon\geq e$, then for any
	$\x^+, \x^-\in \Rad \ball[2][d]$, we have that
	\begin{equation}
		\abs[\Big]{\mean[\sp{\NNlayer(\x^+)}{\NNlayer(\x^-)}] - \Big(\sp{\x^+}{\x^-} + 
			\tfrac{\biasp^2}{3}  + \sqrt{\tfrac{2}{\pi}}\tfrac{1}{6}\tfrac{\lnorm{\x^+-\x^-}^3}{\biasp}\Big)} \leq \varepsilon.
	\end{equation}
\end{proposition}

To prove Proposition~\ref{prop:expectation_inner_product}, we will make use of the following lemma:

\begin{lemma}\label{lem:inner_dither}  Let $\sepbias\in [-\biasp, \biasp]$ be uniformly distributed. For $a, b\in [-\biasp, \biasp]$,
	\begin{equation}
		\mean[\relu(a+\sepbias)\relu(b+\sepbias)]=  \tfrac{ab}{2} + \tfrac{(\min\{a, b\})^2 \max\{a,b\}}{4\biasp} - \tfrac{(\min\{a,b\})^3}{12\biasp} +
		(a+b)\tfrac{\biasp}{4} + \tfrac{\biasp^2}{6}.
	\end{equation}
\end{lemma}
\begin{proof} We may assume that $a\leq b$. Then 
	\begin{align}
		\mean[\relu(a+\sepbias)\relu(b+\sepbias)]
		&=\mean[(a+\sepbias) \probind{\sepbias\geq -a} (b+\sepbias) \probind{\sepbias\geq -b}]\\  
		&=\mean[(a+\sepbias)(b+\sepbias)\probind{\sepbias\geq -a}]\\
		&=\tfrac{1}{2\biasp}\int_{-a}^{\biasp} (ab + (a+b)s + s^2)ds\\
		&=\tfrac{1}{2\biasp}\big(ab(\biasp+a) + (a+b)(\tfrac{\biasp^2}{2}-\tfrac{a^2}{2}) + \tfrac{\biasp^3}{3} + \tfrac{a^3}{3} \big)\\
		&= \tfrac{ab}{2} + \tfrac{a^2b}{4\biasp} - \tfrac{a^3}{12\biasp} +
		(a+b)\tfrac{\biasp}{4} + \tfrac{\biasp^2}{6}\\
		&= \tfrac{ab}{2} + \tfrac{(\min\{a, b\})^2 \max\{a,b\}}{4\biasp} - \tfrac{(\min\{a,b\})^3}{12\biasp} +
		(a+b)\tfrac{\biasp}{4} + \tfrac{\biasp^2}{6}.
	\end{align}
\end{proof}

\begin{proof}[Proposition~\ref{prop:expectation_inner_product}]
	Clearly, we have that
	\begin{equation}
		\mean[\sp{\NNlayer(\x^+)}{\NNlayer(\x^-)}]=2\cdot \mean[\relu(\sp{\g}{\x^+}+\sepbias)\relu(\sp{\g}{\x^-}+\sepbias)],
	\end{equation}
	where $\g$ denotes a standard Gaussian vector and $\sepbias\in [-\biasp, \biasp]$ is an independent and uniformly distributed random variable. 
	Let us begin by showing that 
	\begin{align}
		&\mean\Big[\tfrac{\sp{\g}{\x^+}\sp{\g}{\x^-}}{2}+\tfrac{(\min\{\sp{\g}{\x^+}, \sp{\g}{\x^-}\})^2 \max\{\sp{\g}{\x^+},\sp{\g}{\x^-}\}}{4\biasp} - \tfrac{(\min\{\sp{\g}{\x^+},\sp{\g}{\x^-}\})^3}{12\biasp} \\
		&\quad + \tfrac{\biasp(\sp{\g}{\x^+}+\sp{\g}{\x^-})}{4} + \tfrac{\biasp^2}{6}
		\Big]\\
		&=\tfrac{1}{2}\Big(\sp{\x^+}{\x^-} + 
		\tfrac{\biasp^2}{3}  + \sqrt{\tfrac{2}{\pi}}\tfrac{1}{6}\tfrac{\lnorm{\x^+-\x^-}^3}{\biasp}\Big). \label{eq:expectation_inner_product:Gaussian_exp_inner}
	\end{align}
	Since $\mean[\sp{\g}{\x^+}\sp{\g}{\x^-}]=\sp{\x^+}{\x^-}$ for any vectors $\x^+, \x^-\in \R^d$
	and we have $\mean[\sp{\g}{\x^+}]=\mean[\sp{\g}{\x^-}]=0$, this amounts to showing that 
	\begin{align}
		&\mean\Big[3(\min\{\sp{\g}{\x^+}, \sp{\g}{\x^-}\})^2 \max\{\sp{\g}{\x^+},\sp{\g}{\x^-}\} - (\min\{\sp{\g}{\x^+},\sp{\g}{\x^-}\})^3 \Big]\\
		&= \sqrt{\tfrac{2}{\pi}}\lnorm{\x^+-\x^-}^3. \label{eq:expectation_inner_product:maxmin_Gaussian}
	\end{align}
	Since $\g$ is symmetric, it follows that
	\begin{align}
		&\mean \big[(\min\{\sp{\g}{\x^+}, \sp{\g}{\x^-}\})^2 \max\{\sp{\g}{\x^+}, \sp{\g}{\x^-}\}\big] \\*
		&=  -\mean \big[(\max\{\sp{\g}{\x^+}, \sp{\g}{\x^-}\})^2 \min\{\sp{\g}{\x^+}, \sp{\g}{\x^-}\}\big]   
	\end{align}
	and 
	\begin{equation}
		\mean \big[(\min\{\sp{\g}{\x^+}, \sp{\g}{\x^-}\})^3\big] =  -\mean \big[(\max\{\sp{\g}{\x^+}, \sp{\g}{\x^-}\})^3 \big].   
	\end{equation}
	Therefore, 
	\begin{align}
		&\mean\Big[3(\min\{\sp{\g}{\x^+}, \sp{\g}{\x^-}\})^2 \max\{\sp{\g}{\x^+},\sp{\g}{\x^-}\} - (\min\{\sp{\g}{\x^+},\sp{\g}{\x^-}\})^3 \Big]\\
		&=\tfrac{1}{2}\cdot \Big(
		3\mean \big[(\min\{\sp{\g}{\x^+}, \sp{\g}{\x^-}\})^2 \max\{\sp{\g}{\x^+}, \sp{\g}{\x^-}\}\big]\\
		&\quad-3\mean \big[(\max\{\sp{\g}{\x^+}, \sp{\g}{\x^-}\})^2 \min\{\sp{\g}{\x^+}, \sp{\g}{\x^-}\}\big]\\
		&\quad -\mean \big[(\min\{\sp{\g}{\x^+}, \sp{\g}{\x^-}\})^3\big]+\mean \big[(\max\{\sp{\g}{\x^+}, \sp{\g}{\x^-}\})^3\big]
		\Big)\\
		&=\tfrac{1}{2}\cdot \mean \big[(\max\{\sp{\g}{\x^+}, \sp{\g}{\x^-}\} - \min\{\sp{\g}{\x^+}, \sp{\g}{\x^-}\})^3\big]\\
		&=\tfrac{1}{2}\cdot \mean \big[\abs{\sp{\g}{\x^+-\x^-}}^3\big].
	\end{align}
	Using that $\g$ is rotation invariant and $\mean[\abs{g}^3]=2\sqrt{\tfrac{2}{\pi}}$ for $g \distributed \Normdistr{0}{1}$, we arrive at 
	\eqref{eq:expectation_inner_product:maxmin_Gaussian}.
	Define the event 
	\begin{equation}
		\event=\Big\{\max\{\abs{\sp{\g}{\x^+}}, \abs{\sp{\g}{\x^-}}\}\leq \biasp \Big\}.    
	\end{equation}
	Since $\normsubg{\sp{\g}{\x^+}}, \normsubg{\sp{\g}{\x^-}}\lesssim \Rad$, we have
	\begin{equation}
		\prob(\event^C)
		\leq \prob(\abs{\sp{\g}{\x^+}}>\biasp) + \prob(\abs{\sp{\g}{\x^-}}>\biasp) 
		\leq 4\exp(-c\biasp^2/\Rad^2)
	\end{equation}
	for some absolute constant $c>0$.
	By using the Cauchy-Schwarz inequality twice, we obtain 
	\begin{align}
		&\mean[\abs{\relu(\sp{\g}{\x^+}+\sepbias)\relu(\sp{\g}{\x^-}+\sepbias)}\cdot \probind{\event^C}]\\
		&\leq \big(\mean\big[(\relu(\sp{\g}{\x^+}+\sepbias))^2(\relu(\sp{\g}{\x^-}+\sepbias))^2\big]\big)^{1/2}\cdot \prob(\event^C)^{1/2}\\
		&\leq \norm{\relu(\sp{\g}{\x^+}+\sepbias)}_{L^4}\cdot \norm{\relu(\sp{\g}{\x^-}+\sepbias)}_{L^4}\cdot 2\exp(-c\biasp^2/2\Rad^2)\\
		&\lesssim \biasp^2\exp(-c\biasp^2/2\Rad^2) \leq \varepsilon,
	\end{align}
	where the last inequality follows if 
	$\biasp\geq C \Rad \sqrt{\log(\biasp^2/\varepsilon)}$ for $C>0$ an absolute constant that is chosen large enough and $\biasp^2/\varepsilon\geq e$.
	Therefore, 
	\begin{align}
		&\abs[\Big]{\mean[\sp{\NNlayer(\x^+)}{\NNlayer(\x^-)}] - \Big(\sp{\x^+}{\x^-} + 
			\tfrac{\biasp^2}{3}  + \sqrt{\tfrac{2}{\pi}}\tfrac{1}{6}\tfrac{\lnorm{\x^+-\x^-}^3}{\biasp}\Big)
		}\\*
		&=\abs[\Big]{2 \mean[\relu(\sp{\g}{\x^+}+\sepbias)\relu(\sp{\g}{\x^-}+\sepbias)] - \Big(\sp{\x^+}{\x^-} + 
			\tfrac{\biasp^2}{3}  + \sqrt{\tfrac{2}{\pi}}\tfrac{1}{6}\tfrac{\lnorm{\x^+-\x^-}^3}{\biasp}\Big)
		}\\*
		&\leq \abs[\Big]{2 \mean[\relu(\sp{\g}{\x^+}+\sepbias)\relu(\sp{\g}{\x^-}+\sepbias) \cdot \probind{\event}] - \Big(\sp{\x^+}{\x^-} + 
			\tfrac{\biasp^2}{3}  + \sqrt{\tfrac{2}{\pi}}\tfrac{1}{6}\tfrac{\lnorm{\x^+-\x^-}^3}{\biasp}\Big)
		} \\* & \qquad+ 2\varepsilon. \label{eq:expectation_inner_product:two_summands}
	\end{align}
	Using Lemma~\ref{lem:inner_dither} and the independence of $\g$ and $\sepbias$ we obtain 
	\begin{align}
		&\mean[\relu(\sp{\g}{\x^+}+\sepbias)\relu(\sp{\g}{\x^-}+\sepbias)\cdot \probind{\event}]\\
		&=\mean_{\g}[\probind{\event} \mean_\sepbias[\relu(\sp{\g}{\x^+}+\sepbias)\relu(\sp{\g}{\x^-}+\sepbias)]]\\
		&=\mean_{\g}\big[\probind{\event}\big(\tfrac{\sp{\g}{\x^+}\sp{\g}{\x^-}}{2} + \tfrac{(\min\{\sp{\g}{\x^+}, \sp{\g}{\x^-}\})^2 \max\{\sp{\g}{\x^+},\sp{\g}{\x^-}\}}{4\biasp} - \tfrac{(\min\{\sp{\g}{\x^+},\sp{\g}{\x^-}\})^3}{12\biasp}\big)\big]\\*
		&\quad + \mean_{\g}\big[\probind{\event}\big(
		(\sp{\g}{\x^+}+\sp{\g}{\x^-})\tfrac{\biasp}{4} + \tfrac{\biasp^2}{6}\big)\big].
	\end{align}
	In combination with \eqref{eq:expectation_inner_product:Gaussian_exp_inner} and the Cauchy-Schwarz inequality this yields
	\begin{align}
		&\abs[\Big]{ \mean[\relu(\sp{\g}{\x^+}+\sepbias)\relu(\sp{\g}{\x^-}+\sepbias) \cdot \probind{\event}] - \tfrac{1}{2}\Big(\sp{\x^+}{\x^-} + 
			\tfrac{\biasp^2}{3}  + \sqrt{\tfrac{2}{\pi}}\tfrac{1}{6}\tfrac{\lnorm{\x^+-\x^-}^3}{\biasp}\Big)
		}\\
		&\leq \mean\big[\abs[\big]{\tfrac{\sp{\g}{\x^+}\sp{\g}{\x^-}}{2} + \tfrac{(\min\{\sp{\g}{\x^+}, \sp{\g}{\x^-}\})^2 \max\{\sp{\g}{\x^+},\sp{\g}{\x^-}\}}{4\biasp} - \tfrac{(\min\{\sp{\g}{\x^+},\sp{\g}{\x^-}\})^3}{12\biasp}\\*
			&\quad + 
			(\sp{\g}{\x^+}+\sp{\g}{\x^-})\tfrac{\biasp}{4} + \tfrac{\biasp^2}{6}} \cdot \probind{\event^C}\big]\\
		&\lesssim \Big(\norm{\sp{\g}{\x^+}}_{L^4} \cdot \norm{\sp{\g}{\x^-}}_{L^4} +\tfrac{1}{\biasp}\big(\norm{\sp{\g}{\x^+}}_{L^6}^3 + \norm{\sp{\g}{\x^-}}_{L^6}^3\big)\\ &\quad +\biasp\big(\norm{\sp{\g}{\x^+}}_{L^2}+\norm{\sp{\g}{\x^-}}_{L^2}\big) +\biasp^2 \Big) \cdot \prob(\event^C)^{1/2}\\*
		&\lesssim \biasp^2\cdot \exp(-c\biasp^2/2\Rad^2) \leq \varepsilon, 
	\end{align}
	where the last two inequalities follow by using $\normsubg{\sp{\g}{\x^+}}, \normsubg{\sp{\g}{\x^-}}\lesssim \Rad$ and $\biasp\geq C \Rad \sqrt{\log(\biasp^2/\varepsilon)}$ for $C>0$ an absolute constant that is chosen large enough and $\biasp^2/\varepsilon\geq e$.
	Together with \eqref{eq:expectation_inner_product:two_summands} this implies
	\begin{equation}
		\abs[\Big]{\mean[\sp{\NNlayer(\x^+)}{\NNlayer(\x^-)}] - \Big(\sp{\x^+}{\x^-} + 
			\tfrac{\biasp^2}{3}  + \sqrt{\tfrac{2}{\pi}}\tfrac{1}{6}\tfrac{\lnorm{\x^+-\x^-}^3}{\biasp}\Big)}\leq 4\varepsilon.
	\end{equation}	
	By rescaling $\varepsilon$, we obtain the result.
\end{proof}

\begin{proof}[Theorem~\ref{thm:distance_preservation_ReLU}]
	Let us start by showing \eqref{eq:distance_preservation_ReLU:ip}. 
	Let $\g\in \R^d$ be a standard Gaussian random vector and $\sepbias\in [-\biasp, \biasp]$ be uniformly distributed. 
	Since the ReLU is $1$-Lipschitz, it follows that  
	\begin{equation}
		\normsubg{\relu(\sp{\g}{\x^+}+\sepbias)-\relu(\sp{\g}{\x^-}+\sepbias)}\lesssim   \normsubg{\sp{\g}{\x^+-\x^-}}\lesssim \lnorm{\x^+-\x^-}  
	\end{equation}
	and 
	\begin{equation}
		\normsubg{\relu(\sp{\g}{\x}+\sepbias)}\lesssim  \lnorm{\x} + \biasp\leq 2\biasp
	\end{equation}
	for all $\x^+,\x^-\in \R^d$ and $\x\in \Rad\ball[2][d]$. 
	Hence, the stochastic processes 
	\begin{equation}
		\Big\{\relu(\sp{\g}{\x^+}+\sepbias)\Big\}_{\x^+\in \set^+}, \quad   \Big\{\relu(\sp{\g}{\x^-}+\sepbias)\Big\}_{\x^-\in \set^-}
	\end{equation}
	are sub-Gaussian with respect to the Euclidean metric and their radii in sub-Gaussian norm are bounded by $\biasp$. 
	By a concentration result for empirical product processes where each process is sub-Gaussian
	\citep[Thm.~1.13]{men16}, we have 
	\begin{align}
		&\tfrac{1}{n}\abs[\Big]{\sum_{i=1}^n
			\relu(\sp{\w_i}{\x^+}+\bias_i)\relu(\sp{\w_i}{\x^-}+\bias_i) \\* & \qquad\qquad\qquad\qquad - \mean[\relu(\sp{\w_i}{\x^+}+\bias_i)\relu(\sp{\w_i}{\x^-}+\bias_i)]} \\*
		&\leq C\cdot \bigg(\frac{(\meanwidth{\set^+}+u \biasp)(\meanwidth{\set^-}+u\biasp)}{n} + \frac{\biasp(\meanwidth{\set^+}+\meanwidth{\set^-})+u\biasp^2}{\sqrt{n}} \bigg) \label{eq:expectation_inner_product:concentration_productprocess}
	\end{align}
	uniformly for all $\x^+\in \set^+, \x^-\in \set^-$ with probability at least $1-2\exp(-cu^2)$. 
	The condition on $n$ given by \eqref{eq:distance_preservation_ReLU:cond_n} now implies that  
	the right hand side of \eqref{eq:expectation_inner_product:concentration_productprocess} is bounded by~$\varepsilon$. 
	In combination with Proposition~\ref{prop:expectation_inner_product} this shows \eqref{eq:distance_preservation_ReLU:ip} by using the triangle inequality.
	Analogously, we can show that~\eqref{eq:distance_preservation_ReLU:norm} holds.
	Finally, by polarization, \eqref{eq:distance_preservation_ReLU:ip} and \eqref{eq:distance_preservation_ReLU:norm} imply \eqref{eq:distance_preservation_ReLU:distance}.
\end{proof}

\section{Proof of the Main Result (Theorem~\ref{thm:main})}
\label{sec:proof:main}

The following lemma and especially its Corollary~\ref{coro:lin_sep} are crucial ingredients for the proof of Theorem~\ref{thm:main}.
In short, they make the following geometric statement precise: Let $\set^+, \set^- \in\nobreak \R^d$ be two sets and $\NNlayer \colon \R^d\to \R^n$ a deterministic ReLU-layer. If for every $\x^+ \in \set^+$ there exists at least one \enquote{neuron} that separates $\set^-$ from $\x^+$, then the transformed sets $\NNlayer(\set^-)$ and $\NNlayer(\set^+)$ are linearly separable.
\begin{lemma}\label{lem:lin_sep}
	Let $\set^+, \set^-\subset \R^d$.
	For $\W =[\w_1, \dots, \w_n]^\T \in \R^{n\times d}$ and $\Bias = (\bias_1, \dots, \bias_n) \in\nobreak \R^n$, define the associated (deterministic) ReLU-layer $\NNlayer \colon \R^d\to \R^n$ by 
	\begin{equation}
		\NNlayer(\x) \coloneqq \sqrt{\tfrac{2}{n}} \cdot \relu(\W \x + \Bias), \qquad \x \in \R^d.
	\end{equation}
	Set 
	\begin{equation}
		\setind \coloneqq \big\{i\in [n] \suchthat \text{$\sp{\w_i}{\x^-}+\bias_i\leq 0$ for all $\x^-\in \set^-$}\big\}
	\end{equation}
	and for $\x^+\in \set^+$, define $\setind_{\x^+}(0)\subset [n]$ to be the set of all indices $i\in [n]$ such that 
	$\hypp[\bias_i]{\w_i}$ separates $\set^-$ from~$\x^+$.
	Assume that $\min_{\x^+\in \set^+}\cardinality{\setind_{\x^+}(0)}\geq 1$. 
	Then $\cardinality{I}\geq 1$ and the hyperplane $\hypp[0]{\sepdir}$ given by the vector $\sepdir\in \S_+^{n-1}$ with 
	\begin{equation}
		\sepdircoor_i = \begin{cases}
			\tfrac{1}{\sqrt{\cardinality{{I}}}}, & i\in \setind,\\
			0, & \text{otherwise,}
		\end{cases}
	\end{equation}
	separates $\NNlayer(\set^-)$ from $\NNlayer(\set^+)$. More precisely, 
	\begin{align}
		\sp{\sepdir}{\NNlayer(\x^-)} &\leq 0 \qquad \text{for all $\x^-\in \set^-$,}\\
		\sp{\sepdir}{\NNlayer(\x^+)} &\geq \tfrac{1}{n}\sum_{i\in \setind_{\x^+}(0)}\abs{\sp{\w_i}{\x^+}+\bias_i}\qquad \text{for all $\x^+\in \set^+$.}
	\end{align}
\end{lemma}
\begin{proof}
	Since $\setind_{\x^+}(0)\subset \setind$ for every $\x^+\in \set^+$ and $\min_{\x^+\in \set^+}\cardinality{\setind_{\x^+}(0)}\geq 1$, it follows $\cardinality{I}\geq 1$.
	Further, for any $\x^-\in \set^-$, we have that
	\begin{align}
		\sp{\sepdir}{\NNlayer(\x^-)}=\sum_{i\in \setind}\tfrac{1}{\sqrt{\cardinality{I}}}\cdot \sqrt{\tfrac{2}{n}}\relu(\underbrace{\sp{\w_i}{\x^-}+\bias_i}_{\leq 0})= 
		0.    
	\end{align}
	On the other hand, for any $\x^+\in \set^+$, we have that
	\begin{align}
		\sp{\sepdir}{\NNlayer(\x^+)}
		\geq \sum_{i\in \setind_{\x^+}(0)}\tfrac{1}{\sqrt{\cardinality{I}}}\cdot \sqrt{\tfrac{2}{n}}\relu(\sp{\w_i}{\x^+}+\bias_i)
		\geq \tfrac{1}{n}\sum_{i\in \setind_{\x^+}(0)}\abs{\sp{\w_i}{\x^+}+\bias_i}.
	\end{align}
\end{proof}

\begin{corollary}\label{coro:lin_sep}
	Let $\set^+, \set^-\subset \R^d$ and $\hyppsep\geq 0$.
	For $\W =[\w_1, \dots, \w_n]^\T \in \R^{n\times d}$ and $\Bias = (\bias_1, \dots, \bias_n) \in\nobreak \R^n$, define the associated (deterministic) ReLU-layer $\NNlayer \colon \R^d\to \R^n$ by 
	\begin{equation}
		\NNlayer(\x) \coloneqq \sqrt{\tfrac{2}{n}} \cdot \relu(\W \x + \Bias), \qquad \x \in \R^d.
	\end{equation}
	For $\x^+\in \set^+$, define $\setind_{\x^+}(\hyppsep)\subset [n]$ to be the set of all indices $i\in [n]$ such that 
	$\hypp[\bias_i]{\w_i}$ $\hyppsep$-separates $\set^-$ from~$\x^+$.
	Assume that $\min_{\x^+\in \set^+}\cardinality{\setind_{\x^+}(\hyppsep)}\geq n'$ for some $n'\geq 1$.
	Then $\NNlayer(\set^-)$ and $\NNlayer(\set^+)$ are linearly separable with margin $\tfrac{\hyppsep n'}{2n}$.
\end{corollary}
\begin{proof}
	Set 
	\begin{equation}
		\setind \coloneqq \big\{i\in [n] \suchthat \text{$\sp{\w_i}{\x^-}+\bias_i\leq 0$ for all $\x^-\in \set^-$}\big\}.
	\end{equation}    
	By Lemma~\ref{lem:lin_sep}, we have  $\cardinality{I}\geq 1$ and the hyperplane $\hypp[0]{\sepdir}$ given by the vector $\sepdir\in \S_+^{n-1}$ with 
	\begin{equation}
		\sepdircoor_i = \begin{cases}
			\tfrac{1}{\sqrt{\cardinality{{I}}}}, & i\in \setind,\\
			0, & \text{otherwise,}
		\end{cases}
	\end{equation}
	satisfies  
	\begin{align}
		\sp{\sepdir}{\NNlayer(\x^-)} &\leq 0 \qquad \text{for all $\x^-\in \set^-$,}\\*
		\sp{\sepdir}{\NNlayer(\x^+)} &\geq \tfrac{1}{n}\sum_{i\in \setind_{\x^+}(0)}\abs{\sp{\w_i}{\x^+}+\bias_i}\qquad \text{for all $\x^+\in \set^+$,}
	\end{align}
	where $\setind_{\x^+}(0)$ is the set of all indices $i\in [n]$ such that 
	$\hypp[\bias_i]{\w_i}$ separates $\set^-$ from~$\x^+$. Clearly, $\setind_{\x^+}(\hyppsep)\subset \setind_{\x^+}(0)$, which implies for any $\x^+\in \set^+$ that
	\begin{equation}
		\sp{\sepdir}{\NNlayer(\x^+)} \geq \tfrac{1}{n}\sum_{i\in \setind_{\x^+}(t)}\abs{\sp{\w_i}{\x^+}+\bias_i}> \tfrac{\hyppsep n'}{n}.  
	\end{equation}
	It follows that the hyperplane $\hypp[-\tfrac{\hyppsep n'}{2n}]{\sepdir}$ separates $\NNlayer(\set^-)$ and $\NNlayer(\set^+)$ with margin $\tfrac{\hyppsep n'}{2n}$.
\end{proof}

In line with our proof sketch in Section~\ref{sec:intro:proof}, the following two results describe the effect of the first random ReLU-layer $\NNlayer$ in the setup of Theorem~\ref{thm:main}.
\begin{theorem}\label{thm:lin_sep_random_ReLU/Thres_NN_finite_1st_layer}
	There exists an absolute constant $c>0$ such that the following holds.
	
	Let $\set^-, \set^+\subset \Rad\ball[2][d]$ be $\mdist$-separated sets with $N^- \coloneqq \cardinality{\set^-}$,  $N^+ \coloneqq \cardinality{\set^+}$.
	Let $\NNlayer \colon \R^d\to \R^{n}$ be a random ReLU-layer with maximal bias $\biasp\geq 0$.
	Suppose that $\biasp\gtrsim  \Rad$ and 
	\begin{equation}
		n\gtrsim \mdist^{-1}\biasp\cdot \log (2 N^-N^+/\probsuccess). \label{eq:lin_sep_random_ReLU_NN_finite_1st_layer}
	\end{equation}
	Then with probability at least $1-\probsuccess$, the following event occurs: for every $\x^-\in \set^-$ there exists a vector $\sepdir_{\x^-}\in \S_+^{n-1}$ such that the hyperplane $\hypp[0]{\sepdir_{\x^-}}$ linearly separates $\NNlayer(\x^-)$ from $\NNlayer(\set^+)$.  
	The vector $\sepdir_{\x^-}$ is given by $\sepdir_{\x^-}=\lnorm{\sepdir_{\x^-}'}^{-1}\sepdir_{\x^-}'$ for 
	\begin{equation}\label{eq:lin_sep_random_ReLU_NN_finite_1st_layer:direction_minus}
		(\sepdir_{\x^-}')_i = \begin{cases}
			1, & (\NNlayer(\x^-))_i=0,\\
			0, & \text{otherwise,}
		\end{cases}
	\end{equation}
	and satisfies 
	\begin{align}
		\sp{\sepdir_{\x^-}}{\NNlayer(\x^-)} &\leq 0,\\
		\sp{\sepdir_{\x^-}}{\NNlayer(\x^+)} &\geq c \lnorm{\x^+-\x^-}^2 \cdot \biasp^{-1}\qquad \text{for all $\x^+\in \set^+$.}
	\end{align}
\end{theorem}
\begin{proof}
	Let $\W =[\w_1, \dots, \w_{n}]^\T \in \R^{n\times d}$ and $\Bias = (\bias_1, \dots, \bias_{n}) \in\nobreak \R^{n}$ be the weight matrix and bias vector of $\NNlayer$, respectively. 
	For $\x^-\in \set^-$, $\x^+\in \set^+$ define $\setind_{\x^-,\x^+}\subset [n]$ to be the set of all indices $i\in [n]$ where  
	$\hypp[\bias_i]{\w_{i}}$ separates $\x^-$ from $\x^+$ and define the events 
	\begin{equation}
		\eventalt_{\x^-, \x^+}^i \coloneqq \big\{
		\hypp[\bias_i]{\w_{i}} \; \lnorm{\x^+-\x^-}\text{-separates } \x^- \text{ from } \x^+\}. 
	\end{equation}
	For $n'(\x^-,\x^+)\in \{1,\ldots, n\}$ a number that is specified later, set
	\begin{align}
		&\eventalt_{\x^-, \x^+, n'(\x^-,\x^+)} \coloneqq \Big\{\sum_{i=1}^{n} \probind{\eventalt_{\x^-, \x^+}^i}\geq n'(\x^-,\x^+)\Big\}, \\ &\eventalt_{\x^-} \coloneqq \bigintersec_{\x^+\in \set^+}\eventalt_{\x^-, \x^+, n'(\x^-,\x^+)}, \quad \eventalt \coloneqq \bigintersec_{\x^-\in \set^-} \eventalt_{\x^-}.  
	\end{align}
	On the event $\eventalt$, 
	the following holds for every $\x^-\in \set^-$: For all $\x^+\in \set^+$ there exist at least $n'(\x^-,\x^+)\geq 1$ hyperplanes  $\hypp[\bias_i]{\w_{i}}$ which $\lnorm{\x^+-\x^-}$-separate
	$\x^-$
	from $\x^+$. 
	By Lemma~\ref{lem:lin_sep}, this implies that 
	the following holds on the event $\eventalt$:
	For every $\x^-\in \set^-$ there exists $\sepdir_{\x^-}\in \S_+^{n-1}$ such that the hyperplane $\hypp[0]{\sepdir_{\x^-}}$ linearly separates $\NNlayer(\x^-)$ from~$\NNlayer(\set^+)$. More precisely, 
	\begin{align}
		\sp{\sepdir_{\x^-}}{\NNlayer(\x^-)} &\leq 0,\\*
		\sp{\sepdir_{\x^-}}{\NNlayer(\x^+)} &\geq \tfrac{1}{n}\sum_{i\in \setind_{\x^-,\x^+}}\abs{\sp{\w_i}{\x^+}+\bias_i}\geq \tfrac{n'(\x^-, \x^+)}{n}\lnorm{\x^+-\x^-}\qquad \text{for all $\x^+\in \set^+$.}
	\end{align}
	Further, Lemma~\ref{lem:lin_sep} shows that 
	$\sepdir_{\x^-}=\lnorm{\sepdir_{\x^-}'}^{-1}\sepdir_{\x^-}'$ for $\sepdir_{\x^-}'$ with 
	\begin{equation}
		(\sepdir_{\x^-}')_i = \begin{cases}
			1, & \sp{\w_i}{\x^-}+\bias_i\leq 0,\\
			0, & \text{otherwise.}
		\end{cases}
	\end{equation}
	Since $\sp{\w_i}{\x^-}+\bias_i\leq 0$ is equivalent to $(\NNlayer(\x^-))_i=0$, this shows that $\sepdir_{\x^-}$ is given as described in \eqref{eq:lin_sep_random_ReLU_NN_finite_1st_layer:direction_minus}.
	By the union bound, we obtain
	\begin{equation}
		\prob(\eventalt^C)\leq \sum_{\x^-\in \set^-,\, \x^+\in \set^+} \prob(\eventalt_{\x^-, \x^+, n'(\x^-,\x^+)}^C).
	\end{equation} 
	Let $i\in [n]$. Theorem~\ref{thm:prob_single_hyperplane_separates_two_points} implies that if $\biasp\gtrsim \Rad$, then $\prob(\eventalt_{\x^-, \x^+}^i)\geq c_1 \biasp^{-1} \lnorm{\x^+-\x^-}$ for some absolute constant $c_1>0$.
	Therefore, the Chernoff bound implies that 
	\begin{equation}
		\prob\Big(\sum_{i=1}^{n} \probind{\eventalt_{\x^-, \x^+}^i}\geq \tfrac{c_1}{2} \biasp^{-1}\lnorm{\x^+-\x^-} \cdot n \Big) \geq 1-\exp(-c' \biasp^{-1}\lnorm{\x^+-\x^-} \cdot n).   
	\end{equation}
	Setting $n'(\x^-,\x^+) = \floor{\tfrac{c_1}{2} \biasp^{-1}\lnorm{\x^+-\x^-} \cdot n}$, we obtain 
	\begin{equation}
		\prob(\eventalt_{\x^-, \x^+, n'(\x^-,\x^+)}^C)\leq \exp(-c' \biasp^{-1}\lnorm{\x^+-\x^-} \cdot n) \leq \exp(-c' \biasp^{-1}\mdist n).
	\end{equation}
	Hence, 
	\begin{equation}
		\prob(\eventalt^C)\leq 
		\sum_{\x^-\in \set^-,\, \x^+\in \set^+}\exp(-c' \biasp^{-1}\mdist n)\leq \probsuccess,
	\end{equation}
	where the last inequality follows from 
	\begin{equation}
		n\gtrsim \mdist^{-1}\biasp\cdot \big(\log N^- + \log N^+  + \log(\probsuccess^{-1}) \big).    
	\end{equation}
	Finally, observe that \eqref{eq:lin_sep_random_ReLU_NN_finite_1st_layer} implies that $n'(\x^-,\x^+) = \floor{\tfrac{c_1}{2} \biasp^{-1}\lnorm{\x^+-\x^-} \cdot n}\geq 1$ for all $\x^-\in \set^-$, $\x^+\in \set^+$.
\end{proof}

\begin{theorem}\label{thm:lin_sep_random_ReLU_NN_continuous_1st_layer}
	There exist absolute constants $c,c'>0$ such that the following holds.
	
	Let $\set^-, \set^+\subset \Rad\ball[2][d]$ be $\mdist$-separated sets. Let $\biasp>0$ satisfy $\biasp\gtrsim \Rad\sqrt{\log(\biasp/\mdist)}$ and $\biasp/\mdist\geq\nobreak e$.
	Let $\setcen^+ = \{ \cent_1^+, \dots, \cent_{N^+}^+ \} \subset \Rad\ball[2][d]$ and $\setcen^- = \{ \cent_1^-, \dots, \cent_{N^-}^- \} \subset \Rad\ball[2][d]$
	be $\mdist$-separated and form a $\biasp/c'$-mutual covering for $\set^+$ and $\set^-$ with components $\set_1^+, \dots, \set_{N^+}^+ \subset \set^+$ and $\set_1^-, \dots, \set_{N^-}^- \subset \set^-$. 
	
	Let $\NNlayer \colon \R^d\to \R^{n}$ be a random ReLU-layer with maximal bias $\biasp\geq 0$, such that
	\begin{align}
		n&\gtrsim \biasp^{-2}\cdot \big(\effdim{\set^-}+\effdim{\set^+}\big)+(\tfrac{\biasp}{\mdist})^8\cdot \log(2N^-N^+/\probsuccess), \\
		n&\gtrsim \biasp^6\cdot\Big(\max_{l\in[N^-]}\big\{\distance^{-8}(\cent^-_l, \setcen^+)\cdot \effdim{\set^-_l}\big\}+\max_{j\in[N^+]}\big\{\distance^{-8}(\cent^+_j, \setcen^-)\cdot \effdim{\set^+_j}\big\}\Big).
		\label{eq:lin_sep_random_ReLU_NN_continuous_1st_layer:cond}
	\end{align}
	Then with probability at least $1-\probsuccess$, the following two events occur:
	\begin{enumerate}
		\item $\NNlayer(\set^-),\NNlayer(\set^+)\subset \biasp \ball[2][n]$; 
		\item For every $l\in [N^-]$ there exists a vector $\sepdir_{\cent^-_l}\in\S_+^{n-1}$ such that 
		\begin{align}
			\sp{\sepdir_{\cent^-_l}}{\NNlayer(\x^-)}-4c' \biasp^{-1}\distance^2(\cent^-_l, \setcen^+) &\leq -2c' \biasp^{-1}
			\distance^2(\cent^-_l, \setcen^+)
			\qquad \text{for all $\x^-\in \set^-_l$},\\
			\sp{\sepdir_{\cent^-_l}}{\NNlayer(\x^+)} -4c' \biasp^{-1}
			\distance^2(\cent^-_l, \setcen^+)
			&\geq 2c \biasp^{-1}
			\lnorm{\cent^-_l-\x^+}^2\qquad \text{for all $\x^+\in \set^+$.}\label{eq:lin_sep_random_ReLU_NN_continuous_1st_layer:sep_from_Cplus}
		\end{align}
		Further, $\lnorm{\cent^-_l-\x^+}\geq \tfrac{1}{2}\distance(\cent^-_l, \setcen^+)$ for every $\x^+\in \set^+$, which implies that the hyperplane 
		\begin{equation}
			\hypp[-4c' \biasp^{-1} \distance^2(\cent^-_l, \setcen^+)]{\sepdir_{\cent^-_l}}
		\end{equation}
		linearly separates $\NNlayer(\set^-_l)$ from $\NNlayer(\set^+)$ with margin $\min\{2c',\tfrac{c}{2}\} \biasp^{-1}
		\distance^2(\cent^-_l, \setcen^+)$. 
	\end{enumerate}
\end{theorem}
\begin{proof}
	Since $\setcen^+ = \{ \cent_1^+, \dots, \cent_{N^+}^+ \}$ and $\setcen^- = \{ \cent_1^-, \dots, \cent_{N^-}^- \}$
	form a $\biasp/c'$-mutual covering for $\set^+$ and $\set^-$, there exist 
	$\rad_1^+, \dots, \rad_{N^+}^+ \geq 0$ and $\rad_1^-, \dots, \rad_{N^-}^- \geq 0$ such that
	\begin{enumerate}
		\item
		the sets $\set_j^+ \coloneqq \set^+ \intersec\ball[2][d](\cent_j^+,\rad_j^+)$ for $j \in [N^+]$, and $\set_l^- \coloneqq \set^- \intersec\ball[2][d](\cent_l^-,\rad_l^-)$ for $l \in [N^-]$, cover $\set^+$ and $\set^-$, respectively;
		\item
		$\rad_j^+ \leq c'\biasp^{-1} \distsq{\cent_j^+}{\setcen^-}$ for $j \in [N^+]$, and $\rad_l^- \leq c'\biasp^{-1}\distsq{\cent_l^-}{\setcen^+}$ for $l \in [N^-]$.
	\end{enumerate}
	By Corollary~\ref{coro:distance_preservation_ReLU:norm}, if 
	\begin{equation}
		\biasp\gtrsim \Rad, \quad n\gtrsim \biasp^{-2}\big(\effdim{\set^-}+\effdim{\set^+}\big)+\log(e/\probsuccess),
	\end{equation}
	then $\NNlayer(\set^-),\NNlayer(\set^+)\subset \biasp \ball[2][n]$ with probability at least $1-\probsuccess$.
	Define $\event$ to be the event where
	for every $l\in[N^-]$ there exists a vector $\sepdir_{\cent^-_l}\in \S_+^{n-1}$ such that
	\begin{align}
		\sp{\sepdir_{\cent^-_l}}{\NNlayer(\cent^-_l)} &\leq 0,\\
		\sp{\sepdir_{\cent^-_l}}{\NNlayer(\cent^+_j)} &\geq c \biasp^{-1}\lnorm{\cent^+_j-\cent^-_l}^2 \qquad \text{for all $j\in[N^+]$.}
	\end{align}
	Applying Theorem~\ref{thm:lin_sep_random_ReLU/Thres_NN_finite_1st_layer} to $\setcen^+$ and $\setcen^-$, the condition \eqref{eq:lin_sep_random_ReLU_NN_continuous_1st_layer:cond} implies $\prob(\event)\geq 1-\probsuccess$. Define $\eventalt$ to be the event where the following holds:
	\begin{thmproperties}
		\item
		For all $l\in [N^-]$: 
		\begin{equation}
			\sup_{\x^-\in \set_l^-}\abs[\Big]{\lnorm{\NNlayer(\x^-)-\NNlayer(\cent_l^-)}^2 - \lnorm{\x^--\cent_l^-}^2\Big(1  - \sqrt{\tfrac{2}{\pi}} \tfrac{\lnorm{\x^--\cent_l^-}}{3\biasp}  \Big)}
			\leq \big(c' \biasp^{-1}\distance^2(\cent_l^-,\setcen^+)\big)^2,
		\end{equation}
		\item 
		For all $j\in [N^+]$: 
		\begin{equation}
			\sup_{\x^+\in \set_j^+}\abs[\Big]{\lnorm{\NNlayer(\x^+)-\NNlayer(\cent_j^+)}^2 - \lnorm{\x^+-\cent_j^+}^2\Big(1  - \sqrt{\tfrac{2}{\pi}} \tfrac{\lnorm{\x^+-\cent_j^+}}{3\biasp}  \Big)}
			\leq \big(c' \biasp^{-1}\distance^2(\cent_j^+,\setcen^-)\big)^2.
		\end{equation}	
	\end{thmproperties}
	By Theorem~\ref{thm:distance_preservation_ReLU} and a union bound, the condition \eqref{eq:lin_sep_random_ReLU_NN_continuous_1st_layer:cond} implies $\prob(\eventalt)\geq 1-\probsuccess$. 
	
	Let us show that on the event  $\event\intersec\eventalt$ the second event from Theorem~\ref{thm:lin_sep_random_ReLU_NN_continuous_1st_layer} holds. Let $l\in [N^-]$. For any $\x^-\in \set_l^-$, we have that
	\begin{align}
		\sp{\sepdir_{\cent_l^-}}{\NNlayer(\x^-)} 
		&= \sp{\sepdir_{\cent_l^-}}{\NNlayer(\cent_l^-)} + \sp{\sepdir_{\cent_l^-}}{\NNlayer(\x^-) - \NNlayer(\cent_l^-)}\\
		&\leq \lnorm{\NNlayer(\x^-) - \NNlayer(\cent_l^-)}\\
		&\leq \lnorm{\x^- - \cent_l^-} + c' \biasp^{-1}\distance^2(\cent_l^-,\setcen^+)\\
		&\leq \rad_l^- + c' \biasp^{-1}\distance^2(\cent_l^-,\setcen^+)\\
		&\leq 2c' \biasp^{-1}\distance^2(\cent_l^-,\setcen^+).\label{eq:lin_sep_random_ReLU_NN_continuous_1st_layer:xminus}
	\end{align}
	Let $\x^+\in \set^+$. Then there exists $j\in [N^+]$ such that $\x^+\in \set_j^+$. It holds
	\begin{align}
		\sp{\sepdir_{\cent_l^-}}{\NNlayer(\x^+)} 
		&= \sp{\sepdir_{\cent_l^-}}{\NNlayer(\cent_j^+)} + \sp{\sepdir_{\cent_l^-}}{\NNlayer(\x^+) - \NNlayer(\cent_j^+)}\\
		&\geq c \lnorm{\cent_j^+-\cent_l^-}^2 \biasp^{-1} - \lnorm{\NNlayer(\x^+) - \NNlayer(\cent_j^+)}\\
		&\geq c \lnorm{\cent_j^+-\cent_l^-}^2 \biasp^{-1} - \lnorm{\x^+ - \cent_j^+} - c' \biasp^{-1}\distance^2(\cent_j^+,\setcen^-)\\
		&\geq c \lnorm{\cent_j^+-\cent_l^-}^2 \biasp^{-1} - 2c' \biasp^{-1}\distance^2(\cent_j^+,\setcen^-)\\
		&\geq c \lnorm{\cent_j^+-\cent_l^-}^2 \biasp^{-1} - 2c'\biasp^{-1}\lnorm{\cent_j^+-\cent_l^-}^2\\
		&\geq \tfrac{c}{2} \lnorm{\cent_j^+-\cent_l^-}^2 \biasp^{-1}, 
	\end{align}
	where the last inequality follows if $c'\leq \tfrac{c}{4}$. 
	If $\biasp\gtrsim \Rad$ and $c'\leq 1$, then  
	\begin{align}
		\lnorm{\x^+ - \cent_l^-}&\leq \lnorm{\x^+-\cent_j^+}+\lnorm{\cent_j^+-\cent_l^-}\\*
		&\leq c'\biasp^{-1}\distance^2(\cent_j^+, \setcen^-)+\lnorm{\cent_j^+-\cent_l^-}\\*
		&\leq 2\lnorm{\cent_j^+-\cent_l^-}.
	\end{align}
	Therefore, for any $\x^+\in \set^+$, we obtain
	\begin{equation}
		\sp{\sepdir_{\cent_l^-}}{\NNlayer(\x^+)} 
		\geq \tfrac{c}{8} \lnorm{\x^+-\cent_l^-}^2 \biasp^{-1}.
		\label{eq:lin_sep_random_ReLU_NN_continuous_1st_layer:xplus}
	\end{equation}
	Subtracting $4c' \biasp^{-1}\distance^2(\cent_l^-,\setcen^+)$ in \eqref{eq:lin_sep_random_ReLU_NN_continuous_1st_layer:xminus} and \eqref{eq:lin_sep_random_ReLU_NN_continuous_1st_layer:xplus}, we obtain that for all $\x^-\in \set_l^-$ that
	\begin{equation}
		\sp{\sepdir_{\cent_l^-}}{\NNlayer(\x^-)} - 4c' \biasp^{-1}\distance^2(\cent_l^-,\setcen^+)\leq -2c' \biasp^{-1}\distance^2(\cent_l^-,\setcen^+)
	\end{equation}
	and for all $\x^+\in \set^+$ that
	\begin{align}
		\sp{\sepdir_{\cent_l^-}}{\NNlayer(\x^+)} - 4c' \biasp^{-1}\distance^2(\cent_l^-,\setcen^+)
		\geq \tfrac{c}{8} \lnorm{\x^+-\cent_l^-}^2 \biasp^{-1}- 4c' \biasp^{-1}\distance^2(\cent_l^-,\setcen^+).
	\end{align}    
	If $\biasp\gtrsim R$ and $c'\leq 1$, then for any $\x^+\in \set_j^+$, 
	\begin{align}
		\lnorm{\cent_j^+ - \cent_l^-}&\leq \lnorm{\cent_j^+-\x^+}+\lnorm{\x^+-\cent_l^-}\\*
		&\leq c'\biasp^{-1}\distance^2(\cent_j^+, \setcen^-)+\lnorm{\x^+-\cent_l^-}\\*
		&\leq \tfrac{1}{2}\lnorm{\cent_j^+-\cent_l^-}+\lnorm{\x^+-\cent_l^-},
	\end{align}
	which implies $\lnorm{\cent_j^+ - \cent_l^-}\leq 2\lnorm{\x^+-\cent_l^-}$. 
	In particular, $\lnorm{\x^+-\cent_l^-}\geq \tfrac{1}{2}\distance(\cent_l^-, \setcen^+)$ for all $\x^+\in \set^+$. Furthermore, 
	\begin{equation}
		4c' \biasp^{-1}\distance^2(\cent_l^-,\setcen^+)\leq 4c'\biasp^{-1}\lnorm{\cent_l^--\cent_j^+}^2\leq 16c'\biasp^{-1}\lnorm{\x^+-\cent_l^-}^2
	\end{equation}
	for any $\x^+\in \set_j^+$. 
	Hence, for every $\x^+\in \set^+$, we conclude that
	\begin{align}
		\sp{\sepdir_{\cent_l^-}}{\NNlayer(\x^+)} - 4c' \biasp^{-1}\distance^2(\cent_l^-,\setcen^+)
		&\geq \tfrac{c}{8} \lnorm{\x^+-\cent_l^-}^2 \biasp^{-1}- 16c'\biasp^{-1}\lnorm{\x^+-\cent_l^-}^2\\
		&\geq \tfrac{c}{16}\biasp^{-1} \lnorm{\x^+-\cent_l^-}^2, 
	\end{align}  
	where the last inequality follows if $c'\leq \tfrac{c}{256}$. 
\end{proof}

The final ingredient for the proof of Theorem~\ref{thm:main} is the following lemma. It provides a sufficient condition under which we have that $\meanwidth{\NNlayer(\set)}\lesssim \meanwidth{\set}$ with high probability for $\set \subset \R^d$ and $\NNlayer \colon \R^d\to \R^{n}$ a random ReLU-layer.
\begin{lemma}\label{lem:Gaussian_width_after_1_layer}
	Let $\set\subset \R^d$ and $\NNlayer \colon \R^d\to \R^{n}$ be a random ReLU-layer with standard Gaussian weight matrix $\W\in\R^{n\times d}$ and maximal bias $\biasp\geq 0$. Then, we have that $\meanwidth{\NNlayer(\set)}\leq \meanwidth{\sqrt{\tfrac{2}{n}}\W\set}$, and furthermore, the following holds:
	\begin{thmproperties}
		\item If $n\gtrsim \log(2/\probsuccess)$,
		then $\meanwidth{\tfrac{1}{\sqrt{n}}\W\set}\lesssim  \meanwidth{\set}+\sqrt{n}\diam(\set)$ with probability at least $1-\probsuccess$. 
		\item If 
		$n\gtrsim \effdim{\cone{\set-\set}\intersec\S^{d-1}} +\log(2/\probsuccess)$, 
		then with probability at least $1-\probsuccess$, 
		\begin{equation}
			\sup_{\x\neq\x'\in \set}\lnorm{\tfrac{1}{\sqrt{n}}\W(\tfrac{\x-\x'}{\lnorm{\x-\x'}})}\leq 2.
		\end{equation}
		On this event $\meanwidth{\tfrac{1}{\sqrt{n}}\W\set'}\leq 2\meanwidth{\set'}$ and therefore $\meanwidth{\NNlayer(\set')}\leq 2^{3/2}\meanwidth{\set'}$ for every $\set'\subset \set$. 
	\end{thmproperties}
\end{lemma}
\begin{proof} Let us write $\NNlayer(\x)=\relu(\vec{T}(\x))$ for 
	\begin{equation}
		\vec{T}(\x) \coloneqq \sqrt{\tfrac{2}{n}}\W \x + \sqrt{\tfrac{2}{n}}\Bias, 
	\end{equation}
	where $\W\in \R^{n\times d}$ is a standard Gaussian random matrix and $\Bias$ is uniformly distributed on $[-\biasp, \biasp]^n$. Since the $\relu$ is $1$-Lipschitz, the Gaussian version of Talagrand's contraction principle (see, e.g., \citealp[Ex.~7.2.13]{ver18}) implies that $\meanwidth{\NNlayer(\set)}\leq \meanwidth{\vec{T}(\set)}$. Let $\g\in \R^n$ denote a standard Gaussian vector. Since $\mean_{\g}\sp{\g}{\x}=0$ for every vector $\x$, it follows that
	\begin{align}
		\meanwidth{\vec{T}(\set)}&=\mean_{\g}\Big[\sup_{\x\in \set}\sp{\g}{\sqrt{\tfrac{2}{n}}\W \x + \sqrt{\tfrac{2}{n}}\Bias}\Big]=\mean_{\g}\Big[\sup_{\x\in \set}\sp{\g}{\sqrt{\tfrac{2}{n}}\W \x}\Big].
	\end{align}
	Therefore, $\meanwidth{\NNlayer(\set)}\leq \meanwidth{\sqrt{\tfrac{2}{n}}\W\set}$. Since $\meanwidth{\mathcal{S}}\leq \tfrac{\sqrt{n}}{2}\diam(\mathcal{S})$ for any $\mathcal{S}\subset \R^n$, it follows $\meanwidth{\tfrac{1}{\sqrt{n}}\W\set}\leq \tfrac{1}{2} \diam(\W\set)$. By Gaussian projection (e.g., see~\citealp[Sec.~7.7]{ver18}), there exists an absolute constant $C>0$ such that
	\begin{equation}
		\diam(\W\set)\leq C \cdot (\meanwidth{\set}+\sqrt{n}\diam(\set))
	\end{equation}
	with probability at least $1-2\exp(-n)$. Define 
	\begin{equation}
		\norm{\tfrac{1}{\sqrt{n}}\W}_{\set}\coloneqq\sup_{\x\neq\x'\in \set}\lnorm{\tfrac{1}{\sqrt{n}}\W(\tfrac{\x-\x'}{\lnorm{\x-\x'}})}. 
	\end{equation}
	Let $\set'\subset \set$. Then $\lnorm{\tfrac{1}{\sqrt{n}}\W\x-\tfrac{1}{\sqrt{n}}\W\x'}\leq \norm{\tfrac{1}{\sqrt{n}}\W}_{\set}\lnorm{\x-\x'}$ for all $\x,\x'\in \set'$, which implies $\meanwidth{\tfrac{1}{\sqrt{n}}\W\set'}\leq \norm{\tfrac{1}{\sqrt{n}}\W}_{\set}\meanwidth{\set'}$ by the Sudakov-Fernique inequality. 
	By a Gaussian deviation inequality (e.g., see~\citealp[Sec.~9.1]{ver18}), if 
	\begin{equation}
		n\gtrsim \effdim{\cone{\set-\set}\intersec\S^{d-1}} +\log(2/\probsuccess),
	\end{equation}
	then $\norm{\tfrac{1}{\sqrt{n}}\W}_{\set}\leq 2$ with probability at least $1-\probsuccess$. 
\end{proof}

We are now ready to prove the main result of this work:

\begin{proof}[Theorem~\ref{thm:main}]
	Let $\W =[\w_1, \dots, \w_n]^\T \in \R^{n\times d}$ and $\Bias = (\bias_1, \dots, \bias_n) \in\nobreak [-\biasp, \biasp]^n$ be the Gaussian weight matrix and bias vector of the random ReLU-layer $\NNlayer$, and let $\layersec{\W} =[\layersec{\w}_1, \dots, \layersec{\w}_{\layersec{n}}]^\T \in\nobreak \R^{\layersec{n}\times n}$ and $\layersec{\Bias} = (\layersec{\bias}_1, \dots, \layersec{\bias}_{\layersec{n}}) \in\nobreak [-\layersec{\biasp}, \layersec{\biasp}]^{\layersec{n}}$ be the Gaussian weight matrix and bias vector of the random ReLU-layer $\layersec{\NNlayer}$.
	Since $\set^+$ and $\set^-$ have $(\Rad, \mdist, C'\biasp)$-mutual complexity $(N^+, N^-, \mwloc^+, \mwloc^-)$, there exists a $C'\biasp$-mutual covering $\setcen^+=\{\cent_1^+, \ldots, \cent_{N^+}^+\}\subset\nobreak\R^d$ and $\setcen^-=\{\cent_1^-, \ldots, \cent_{N^-}^-\}\subset\R^d$ for $\set^+$ and $\set^-$ such that 
	\begin{romanlist}
		\item
		$\displaystyle\max_{j \in [N^+]} \meanwidth{\set_j^+} \leq \mwloc^+$ and $\displaystyle\max_{l \in [N^-]} \meanwidth{\set_l^-} \leq \mwloc^-$;
		\item
		$\setcen^+, \setcen^-\subset \Rad\ball[2][d]$ are $\mdist$-separated.
	\end{romanlist}
	Here, $\set_1^+, \ldots, \set_{N^+}^+\subset \set^+$ and $\set_1^-, \ldots, \set_{N^-}^-\subset \set^-$ are the components of the covering. 
	Let $C'\coloneqq\tfrac{1}{c'}$, where $c'>0$ is the absolute constant from Theorem~\ref{thm:lin_sep_random_ReLU_NN_continuous_1st_layer}.
	According to Theorem~\ref{thm:lin_sep_random_ReLU_NN_continuous_1st_layer}, the condition \eqref{eq:main:first_layer} implies that with probability at least $1-\probsuccess$, the following event $\event$ occurs:  
	\begin{enumerate}
		\item $\NNlayer(\set^-),\NNlayer(\set^+)\subset \biasp \ball[2][n]$; 
		\item For every $l\in [N^-]$, there exists a vector $\sepdir_{\cent^-_l}\in\S_+^{n-1}$ such that 
		the hyperplane
		\begin{equation}
			\hypp[-4c' \biasp^{-1} \distance^2(\cent^-_l, \setcen^+)]{\sepdir_{\cent^-_l}}
		\end{equation}
		linearly separates $\NNlayer(\set^-_l)$ from $\NNlayer(\set^+)$ with margin $\min\{2c',\tfrac{c}{2}\} \biasp^{-1}
		\distance^2(\cent^-_l, \setcen^+)$. 
	\end{enumerate} 
	Here, $c>0$ is the absolute constant from Theorem~\ref{thm:lin_sep_random_ReLU_NN_continuous_1st_layer}.	
	By Proposition~\ref{prop:separable_relations}~\ref{prop:separable_relations:linear_eps_narrow}, and on the event $\event$, the sets $\NNlayer(\set^-_l)$ and $\NNlayer(\set^+)$ are contained in $\biasp \ball[2][n]$ for every $l\in [N^-]$ and they are $(\narrow_l, \narrowmdist_l)$-linearly separable with 
	\begin{equation}
		\narrow_l = 1- \min\{2c',\tfrac{c}{2}\} \biasp^{-2}
		\distance^2(\cent^-_l, \setcen^+), \quad \narrowmdist_l = \min\{4c',c\} \biasp^{-1}
		\distance^2(\cent^-_l, \setcen^+). 
	\end{equation}
	For $l\in [N^-]$, $\hyppsep_l\gtrsim \meanwidth{\NNlayer(\set_l^-)-\NNlayer(\set^+)}+\biasp$, and $i\in [\layersec{n}]$, we define the event
	\begin{equation}
		\eventalt^i_l(\hyppsep_l)\coloneqq \{\hypp[\layersec{\bias}_i]{\layersec{\w}_{i}} \; \hyppsep_l\text{-separates } \NNlayer(\set^+) \text{ from } \NNlayer(\set^-_l)\}.
	\end{equation}
	Set $\marg_{l}=\narrowmdist_{l}(1-\narrow_{l})$. By Theorem~\ref{thm:hyperplane_sep:general}, if $\layersec{\biasp}\gtrsim \biasp \hyppsep_l\marg_l^{-1}$, then $\prob(\eventalt^i_l(\hyppsep_l)\suchthat \event)\geq p_l$ for 
	\begin{equation}
		p_l= \tfrac{\hyppsep_l}{\layersec{\biasp}}\exp(-C\hyppsep_l^2\marg_l^{-2}\log(4(1-\narrow_l)^{-1})). 
	\end{equation}
	Define $\event'$ to be the event where
	\begin{equation}
		\sup_{\x\neq\x'\in \set^-}\lnorm{\tfrac{1}{\sqrt{n}}\W(\tfrac{\x-\x'}{\lnorm{\x-\x'}})}\leq 2, \quad \sup_{\x\neq\x'\in \set^+}\lnorm{\tfrac{1}{\sqrt{n}}\W(\tfrac{\x-\x'}{\lnorm{\x-\x'}})}\leq 2. 
	\end{equation}
	By Lemma~\ref{lem:Gaussian_width_after_1_layer} and the union bound, condition \eqref{eq:main:first_layer} implies $\prob(\event')\geq 1-\probsuccess$. 
	Further, Lemma~\ref{lem:Gaussian_width_after_1_layer} shows that on the event $\event'$, we have that
	\begin{equation}
		\meanwidth{\NNlayer(\set_l^-)-\NNlayer(\set^+)}=\meanwidth{\NNlayer(\set_l^-)}+\meanwidth{\NNlayer(\set^+)}\leq 2^{3/2}(\meanwidth{\set_l^-}+\meanwidth{\set^+})  
	\end{equation}
	for every $l\in [N^-]$. Using that $\marg_l\gtrsim \biasp^{-3}\mdist^4$ and $1-\narrow_l\gtrsim \biasp^{-2}\mdist^2$ for every $l\in [N^-]$, we obtain that for every $i\in [\layersec{n}]$, $l\in [N^-]$, and $\hyppsep\asymp \mwloc^- +\meanwidth{\set^+}+\biasp$, if $\layersec{\biasp}\gtrsim \biasp^4\mdist^{-4}\hyppsep$, then
	$\prob(\eventalt^i_l(\hyppsep)\suchthat \event\intersec \event')\geq p(\hyppsep)$ for 
	\begin{equation}
		p(\hyppsep)= \tfrac{\hyppsep}{\layersec{\biasp}}\exp(-C\hyppsep^2\biasp^6\mdist^{-8}\log(\biasp/\mdist)). 
	\end{equation}
	Define the events 
	\begin{equation}
		\eventalt_{l}(\hyppsep) \coloneqq \Big\{\sum_{i=1}^{\layersec{n}} \probind{\eventalt_{l}^i(\hyppsep)}\geq \tfrac{p(\hyppsep)}{2}\layersec{n}\Big\},  \quad  \eventalt_\hyppsep \coloneqq \bigintersec_{l\in [N^-]} \eventalt_{l}(\hyppsep).  
	\end{equation}
	By Chernoff's inequality, there exists an absolute constant $c>0$ such that for all $l\in [N^-]$, it holds that
	\begin{equation}
		\prob(\eventalt_{l}(\hyppsep)\suchthat\event\intersec \event')\geq 1-\exp(-c \cdot p(\hyppsep)\layersec{n}).    
	\end{equation}
	On the event $\eventalt_\hyppsep$, for every $l\in [N^-]$ at least $\tfrac{p(\hyppsep)}{2}\layersec{n}$ out of the $\layersec{n}$ hyperplanes $\hypp[\layersec{\bias}_i]{\layersec{\w}_{i}}$ \mbox{$\hyppsep$-separate} $\NNlayer(\set^+)$ from  $\NNlayer(\set_l^-)$. Using that $\NNlayer(\set^-)=\bigcup_{l\in [N^-]}\NNlayer(\set_l^-)$, Corollary~\ref{coro:lin_sep} implies that $\NN(\set^+)$ and $\NN(\set^-)$ are linearly separable with margin 
	\begin{equation}
		\tfrac{\hyppsep p(\hyppsep)}{4}\asymp
		\tfrac{(\mwloc^-+\meanwidth{\set^+}+\biasp)^2}{\layersec{\biasp}}\cdot 
		\exp(-C(\mwloc^-+\meanwidth{\set^+}+\biasp)^2 \biasp^6\mdist^{-8} \log(\biasp/\mdist)).
	\end{equation}
	Define $\eventalt'$ to be the event where $\NN(\set^-), \NN(\set^+)\subset \layersec{\biasp}\ball[2][\layersec{n}]$. 
	On the event $\eventalt_\hyppsep\intersec \eventalt'$, the conclusion of Theorem~\ref{thm:main} holds. 
	Now, we observe that
	\begin{align}
		\prob(\eventalt_\hyppsep\intersec\eventalt')&\geq \prob(\eventalt_\hyppsep\intersec\eventalt'\intersec\event\intersec\event')\\
		&=\prob(\eventalt_\hyppsep\intersec\eventalt'\suchthat \event\intersec \event')\cdot \prob(\event\intersec\event')\\
		&\geq (1 - \prob(\eventalt_\hyppsep^C\suchthat \event\intersec \event') -\prob((\eventalt')^C\suchthat \event\intersec \event') )\cdot(1-\prob(\event^C)-\prob((\event')^C))\\
		&\geq (1 - \prob(\eventalt_\hyppsep^C\suchthat \event\intersec \event') -\prob((\eventalt')^C\suchthat \event\intersec \event') )\cdot(1-2\probsuccess). 
	\end{align}
	The union bound implies   
	\begin{align}
		\prob(\eventalt_\hyppsep^C\suchthat\event\intersec \event')\leq \sum_{l\in [N^-]}\prob((\eventalt_{l}(\hyppsep))^C\suchthat \event\intersec \event')\leq \exp(\log(N^-)-cp(\hyppsep)\layersec{n})\leq \probsuccess,
	\end{align}
	where the last inequality follows from  
	\begin{equation}
		\layersec{n}\gtrsim (p(\hyppsep))^{-1}\log(N^-/\probsuccess).
	\end{equation}
	On the event $\event\intersec \event'$ it holds $\NNlayer(\set^-), \NNlayer(\set^+)\subset\biasp \ball[2][n]$ and $\meanwidth{\NNlayer(\set^-)}\leq 2^{3/2}\meanwidth{\set^-}, \meanwidth{\NNlayer(\set^+)}\leq 2^{3/2}\meanwidth{\set^+}$. Consequently, by Corollary~\ref{coro:distance_preservation_ReLU:norm}, if $\layersec{\biasp}\gtrsim \biasp$ and 
	\begin{equation}\label{eq:thm:main:B'}
		\layersec{n}\gtrsim (\layersec{\biasp})^{-2}\big(\effdim{\set^+}+\effdim{\set^-}\big)+\log(e/\eta),
	\end{equation}
	then $\prob((\eventalt')^C\suchthat \event\intersec \event') \leq \probsuccess$. By Lemma~\ref{lem:Gaussian_width_union},
	\begin{equation}
		\meanwidth{\set^-}=\meanwidth{\bigcup_{l\in [N^-]}\set_l^-}\leq \mwloc^-+C\Rad\sqrt{\log N^-}
	\end{equation}
	for $C>0$ an absolute constant. Therefore, 
	condition \eqref{eq:main:second_layer} implies $\layersec{\biasp}\gtrsim \biasp$ and \eqref{eq:thm:main:B'}. In total, we have $\prob(\eventalt_\hyppsep\intersec\eventalt')\geq (1-2\probsuccess)^2\geq 1-4\probsuccess$. This completes the proof.
\end{proof}

\section{Proofs of Special-Case Results}
\label{sec:special_cases}

To apply our main result, Theorem~\ref{thm:main}, to various special cases, the following lemma will prove very useful. 
Although the inequalities stated therein are well-known (e.g., see \citealp[Lem.~10]{jc17} for the first inequality), we give a proof for the sake of completeness. 
\begin{lemma}\label{lem:Gaussian_width_union} 
	There exists an absolute constant $C>0$ such that the following holds.
	
	Let $\set_j\subset \Rad\ball[2][d]$ for $j\in [N]$.
	Then 
	\begin{equation}\label{eq:lem:Gaussian_width_union:1}
		\meanwidth{\bigcup_{j\in[N]}\set_j}\leq  \max_{j\in[N]}\meanwidth{\set_j}+ C\cdot \Rad\sqrt{\log N}. 
	\end{equation}
	If all sets $\set_j$ additionally satisfy $\diam(\set_j)\leq \rad$ for some $\rad > 0$, then
	\begin{equation}\label{eq:lem:Gaussian_width_union:2}
		\meanwidth{\bigcup_{j\in[N]}\set_j}\lesssim r\sqrt{d}+ \Rad \sqrt{\log N}.
	\end{equation}
	Furthermore, if all sets $\set_j$ are finite with $\diam(\set_j)\leq \rad_j$,
	then 
	\begin{equation}\label{eq:lem:Gaussian_width_union:3}
		\meanwidth{\bigcup_{j\in[N]}\set_j}\lesssim \max_{j\in[N]}(r_j\sqrt{\log \cardinality{\set_j}})+ \Rad\sqrt{\log N}. 
	\end{equation}
\end{lemma}
\begin{proof}
	Let us start by showing \eqref{eq:lem:Gaussian_width_union:1}. 
	Let $\g\in \R^d$ denote a standard Gaussian random vector and for $j \in [N]$ pick any $\cent_j\in \set_j$. Then $\radius(\set_j-\cent_j)\leq 2\Rad$.
	Set $X_j\coloneqq \sup_{\x \in \set_j}\sp{\g}{\x-\cent_j}$.
	Then 
	\begin{align}
		\meanwidth{\bigcup_{j\in[N]}\set_j}&= \mean \max_{j\in [N]} (X_j + \sp{\g}{\cent_j}) \leq \max_{j\in [N]}\mean X_j+  \mean \max_{j\in [N]} (X_j -\mean X_j) +\mean \max_{j\in [N]} \sp{\g}{\cent_j}. 
	\end{align}
	Clearly, $\mean X_j=\meanwidth{\set_j}$. By Gaussian Lipschitz concentration (e.g., see~\citealp[Thm.~8.34]{fh13}), we conclude that $X_j -\mean X_j$ is a sub-Gaussian random variable with $\normsubg{X_j -\mean X_j}\lesssim \radius(\set_j-\cent_j)\leq 2\Rad$. Further, the random variables $\sp{\g}{\cent_j}$ are sub-Gaussian with $\normsubg{\sp{\g}{\cent_j}}\lesssim \lnorm{\cent_j}\leq \Rad$.  
	Inequality \eqref{eq:lem:Gaussian_width_union:1} now follows by applying the maximal inequality for sub-Gaussian random variables (e.g., see~\citealp[Thm.~2.5]{blm13}). Inequalities  \eqref{eq:lem:Gaussian_width_union:2} and \eqref{eq:lem:Gaussian_width_union:3} immediately follow from \eqref{eq:lem:Gaussian_width_union:1} by using the standard estimates  $\meanwidth{\ball[2][d]}\lesssim \nobreak \sqrt{d}$ and $\meanwidth{\set}\lesssim \sqrt{\log(\cardinality{\set})}$ for any finite $\set\subset \ball[2][d]$.
\end{proof}

\subsection{Proof of Theorem~\ref{thm:memorization}}

Define $\setcen^+\coloneqq\set^+$ and $\setcen^-\coloneqq\set^-$. Then $\setcen^+, \setcen^-\subset \ball[2][d]$ are $\mdist$-separated. We may write $\setcen^+=\{\cent_1^+, \ldots, \cent_{N^+}^+\}$ and $\setcen^-=\{\cent_1^-, \ldots, \cent_{N^-}^-\}$. Clearly, the sets   $\set_j^+\coloneqq\set^+\intersec \ball[2][d](\cent_j^+,0)=\{\cent_j^+\}$ for $j\in [N^+]$ and $\set_l^-\coloneqq\set^-\intersec \ball[2][d](\cent_l^-,0)=\{\cent_l^-\}$ for $l\in [N^-]$ cover $\set^+$ and $\set^-$, respectively.
Let $C'>0$ denote the absolute constant from Theorem~\ref{thm:main}. Then
\begin{equation}
	0\leq \tfrac{1}{C'\biasp}\distance^2(\cent_j^+, \setcen^-), \quad 0\leq \tfrac{1}{C'\biasp}\distance^2(\cent_l^-, \setcen^+)
\end{equation}
for all $j\in [N^+]$ and $l\in [N^-]$. Therefore, $\setcen^+$ and $\setcen^-$ form a $C'\biasp$-mutual covering for $\set^+$ and $\set^-$. Moreover, $\set^+$ and $\set^-$ have $(1, \mdist, C'\biasp)$-mutual complexity $(N^+, N^-, \mwloc^+, \mwloc^-)$ with $\mwloc^+=\mwloc^-=0$. Since $\meanwidth{\set^+}\lesssim \sqrt{\log N^+}$ and 
\begin{equation}
	\effdim{\cone{\set^--\set^-}\intersec\S^{d-1}} + \effdim{\cone{\set^+-\set^+}\intersec\S^{d-1}}\lesssim \log N^- + \log N^+,
\end{equation}
the result follows from Theorem~\ref{thm:main}. \qed

\subsection{Proof of Theorem~\ref{thm:eucl_balls}}

Let $\rad\leq \tfrac{1}{C'\biasp}\mdist^2$, where $C'>0$ denotes the absolute constant from Theorem~\ref{thm:main}.
Set $\setcen^+ \coloneqq \{ \cent_1^+, \dots, \cent_{N^+}^+ \}$ and $\setcen^- \coloneqq \{ \cent_1^-, \dots, \cent_{N^-}^- \}$. 
Then, $\setcen^+, \setcen^-\subset \ball[2][d]$ are $\mdist$-separated and 
the sets $\set_j^+ \coloneqq \set^+ \intersec\ball[2][d](\cent_j^+,\rad)$ for $j \in [N^+]$, and $\set_l^- \coloneqq \set^- \intersec\ball[2][d](\cent_l^-,\rad)$ for $l \in [N^-]$, cover~$\set^+$ and~$\set^-$, respectively. Furthermore, the $\mdist$-separability and the assumption $\rad \lesssim \mdist^2/\biasp$ imply that 
\begin{equation}
	\rad \leq \tfrac{1}{C'\biasp} \distsq{\cent_j^+}{\setcen^-}, \quad \rad \leq \tfrac{1}{C'\biasp}\distsq{\cent_l^-}{\setcen^+}
\end{equation}
for all $j\in [N^+], l\in [N^-]$. This shows that $\setcen^+$ and $\setcen^-$ form a $C'\biasp$-mutual covering for $\set^+$ and $\set^-$. Therefore, $\set^+$ and $\set^-$ have $(1, \mdist, C'\biasp)$-mutual complexity $(N^+, N^-, \mwloc^+, \mwloc^-)$ with $\mwloc^+=\max_{j\in [N^+]}\meanwidth{\set_j^+}$ and $\mwloc^-=\max_{l\in [N^-]}\meanwidth{\set_l^-}$. 
By Lemma~\ref{lem:Gaussian_width_union},  
\begin{equation}
	\meanwidth{\set^+}=\meanwidth{\bigunion_{j\in [N^+]}\ball[2][d](\cent^+_j,\rad)}\lesssim \rad\sqrt{d}+\sqrt{\log N^+},
\end{equation}
and for any $l\in [N^-]$, 
\begin{equation}
	\meanwidth{\set_l^-}=  \meanwidth{\set^- \intersec\ball[2][d](\cent_l^-,\rad)}\lesssim \rad\sqrt{d},  
\end{equation}
which yields $\mwloc^-\lesssim \rad\sqrt{d}$. Analogously, it follows that $\mwloc^+\lesssim \rad\sqrt{d}$. 
The result now follows from Theorem~\ref{thm:main} by observing that 
\begin{equation}
	\effdim{\cone{\set^--\set^-}\intersec\S^{d-1}} + \effdim{\cone{\set^+-\set^+}\intersec\S^{d-1}}\lesssim d.
\end{equation}
\qed

\subsection{Proof of Theorem~\ref{thm:uniform_covering}}

For an absolute constant $c>0$ that is specified later, let $\setcen^+ = \{ \cent_1^+, \dots, \cent_{N^+}^+ \} \subset \set^+$ and $\setcen^- = \{ \cent_1^-, \dots, \cent_{N^-}^- \} \subset \set^-$ 
be minimal $c\mdist^2/\biasp$-coverings of $\set^+$ and $\set^-$, respectively. Then $N^+=\covnumber{\set^+}{c\mdist^2/\biasp}$ and 
$N^-=\covnumber{\set^-}{c\mdist^2/\biasp}$. Since $\set^+$ and $\set^-$ are $\mdist$-separated, it follows that $\setcen^+$ and $\setcen^-$ are $\mdist$-separated as well. By definition of covering, the sets $\set_j^+ \coloneqq \set^+ \intersec\ball[2][d](\cent_j^+,c\mdist^2/\biasp)$ for $j \in [N^+]$, and $\set_l^- \coloneqq \set^- \intersec\ball[2][d](\cent_l^-,c\mdist^2/\biasp)$ for $l \in [N^-]$, cover $\set^+$ and $\set^-$, respectively. Further, since $\setcen^+$ and $\setcen^-$ are $\mdist$-separated, we have that
\begin{equation}
	c\mdist^2/\biasp \leq c\biasp^{-1} \distsq{\cent_j^+}{\setcen^-}, \quad c\mdist^2/\biasp \leq c\biasp^{-1}\distsq{\cent_l^-}{\setcen^+}
\end{equation}
for all $j\in [N^+], l\in [N^-]$. This shows that $\setcen^+$ and $\setcen^-$ form a $\tfrac{\biasp}{c}$-mutual covering for~$\set^+$ and~$\set^-$. Therefore, $\set^+$ and $\set^-$ have $(1, \mdist, \tfrac{\biasp}{c})$-mutual complexity $(N^+, N^-, \mwloc^+, \mwloc^-)$ with $N^+=\covnumber{\set^+}{c\mdist^2/\biasp}$, $N^-=\covnumber{\set^-}{c\mdist^2/\biasp}$, $\mwloc^+=\meanwidth{\set^+}$ and $\mwloc^-=\meanwidth{\set^-}$. Choosing $c=\nobreak\tfrac{1}{C'}$, where~$C'$ is the absolute constant from Theorem~\ref{thm:main}, the result follows from Theorem~\ref{thm:main}. \qed

	
	\acks{S.D.\ and M.G.\ acknowledge support by the DFG Priority Programme DFG-SPP 1798 Grant DI 2120/1-1. A.S.\ acknowledges support by the Fonds de la Recherche Scientifique -- FNRS under Grant n$^\circ$ T.0136.20 (Learn2Sense). L.J. is a FNRS Senior Research Associate.}
	
	
	

	\vskip 0.2in
	\bibliography{references.bib}
	\addcontentsline{toc}{section}{References}
	
\end{document}